%% file: arxiv-convex-booster.tex
\documentclass[12pt]{article}
\usepackage[margin=1in]{geometry}

\usepackage{makecell,colortbl}
\usepackage{microtype}
\usepackage{graphicx}
\usepackage{subfigure}
\usepackage{booktabs} 

\usepackage{times,rotating}
\usepackage{url}
\usepackage{booktabs, multicol, multirow}
\usepackage{caption}

\usepackage{wrapfig}

\usepackage[utf8]{inputenc} 
\usepackage[T1]{fontenc}    
\usepackage{amsfonts}       
\usepackage{amsmath}
\usepackage{amsthm}

\usepackage{caption}
\usepackage{boldline}
\usepackage{hhline}
\usepackage{enumitem}

\usepackage{nicefrac}       
\usepackage{microtype}      

\usepackage{upgreek,morenotations2,rotating}
\newcolumntype{?}{!{\vrule width 1pt}}

\usepackage{hyperref}
\usepackage[capitalize,noabbrev]{cleveref}

\title{\papertitle}
\author{Yishay Mansour$^\dagger$ \quad Richard Nock$^\dagger$ \quad Robert C. Williamson$^\ddagger$\\

  $^\dagger$Google Research, $^\ddagger$University of T{\"u}bingen\\

{\normalsize \texttt{$\{$mansour,richardnock$\}$@google.com,bob.williamson@uni-tuebingen.de}}
}
\begin{document}

\date{}

\maketitle

\begin{abstract}
  \input{content-arxiv/abstract}
\end{abstract}

\input{content-arxiv/introduction}

\input{content-arxiv/definitions}
\input{content-arxiv/surrogate-losses}

\input{content-arxiv/generalisation}
\input{content-arxiv/boosting}
\input{content-arxiv/experiments}
\input{content-arxiv/discussion}
\input{content-arxiv/acknowledgements}


\bibliographystyle{plain}
\bibliography{bibgen}

\newpage
\appendix
\onecolumn
\renewcommand\thesection{\Roman{section}}
\renewcommand\thesubsection{\thesection.\arabic{subsection}}
\renewcommand\thesubsubsection{\thesection.\thesubsection.\arabic{subsubsection}}

\renewcommand*{\thetheorem}{\Alph{theorem}}
\renewcommand*{\thelemma}{\Alph{lemma}}
\renewcommand*{\thecorollary}{\Alph{corollary}}

\renewcommand{\thetable}{A\arabic{table}}

\begin{center}
\Huge{Appendix}
\end{center}

To
differentiate with the numberings in the main file, the numbering of
Theorems, etc. is letter-based (A, B, ...).

\section*{Table of contents}

\noindent \textbf{What the papers say} \hrulefill Pg
\pageref{sec-wtps}\\

\noindent \textbf{Supplementary material on proofs} \hrulefill Pg
\pageref{sec-sup-pro}\\
\noindent $\hookrightarrow$ Proof of Lemma \ref{lem-PHI1} \hrulefill Pg \pageref{proof-lem-PHI1}\\
\noindent $\hookrightarrow$ Proof of Lemma \ref{lem-GEN1} \hrulefill Pg \pageref{proof-lem-GEN1}\\
\noindent $\hookrightarrow$ Proof of Lemma \ref{lem-Rot} \hrulefill Pg \pageref{proof-lem-Rot}\\
\noindent $\hookrightarrow$ A side negative result for \topdowngen~with \cls \hrulefill Pg \pageref{proof-sidenegative}\\
\noindent $\hookrightarrow$ Proof of Lemma \ref{solALPHA} \hrulefill Pg \pageref{proof-lem-solALPHA}\\
\noindent $\hookrightarrow$ Proof of Theorem \ref{th-boost-pbls} \hrulefill Pg \pageref{proof-th-boost-pbls}\\
\noindent $\hookrightarrow$ Proof of Lemma \ref{lem-split-gives-wla} \hrulefill Pg \pageref{proof-lem-split-gives-wla}\\
\noindent $\hookrightarrow$ Proof of Lemma \ref{lem-J} \hrulefill Pg \pageref{proof-lem-J}\\
\noindent $\hookrightarrow$ Proof of Lemma \ref{lem-BOO2} \hrulefill Pg \pageref{proof-lem-BOO2}\\
\noindent $\hookrightarrow$ Proof of Lemma \ref{lem-NoNoise} \hrulefill Pg \pageref{proof-lem-NoNoise}\\

\noindent \textbf{Supplementary material on experiments} \hrulefill Pg
\pageref{sec-sup-exp}\\

\newpage 

\input{content-arxiv/whatthepaperssay}
\input{content-arxiv/appendix}

\end{document}

%% file: content-arxiv/abstract.tex
A landmark negative result of Long and Servedio established a worst-case spectacular failure of a supervised learning trio (loss, algorithm, model) otherwise praised for its high precision machinery. Hundreds of papers followed up on the two suspected culprits: the loss (for being convex) and/or the algorithm (for fitting a classical boosting blueprint). Here, we call to the half-century+ founding theory of losses for class probability estimation (properness), an extension of Long and Servedio's results and a new general boosting algorithm to demonstrate that the real culprit in their specific context was in fact the (linear) model class. We advocate for a more general stanpoint on the problem as we argue that the source of the negative result lies in the dark side of a pervasive -- and otherwise prized -- aspect of ML: \textit{parameterisation}.

%% file: content-arxiv/introduction.tex
\section{Introduction}\label{sec-intro}

In a now very influential paper cumulating hundreds of citations on Google Scholar, Long and Servedio \cite{lsRC-conf,lsRC} made a series of observations on how simple symmetric label noise can "wipe out" the edge of a learner against the fair coin. The negative result is extreme in the sense that without noise, the learner fits a large margin, 100$\%$ accurate classifier but as soon as noise afflicts labels, \textit{regardless of its magnitude}, the learner ends up with a classifier only as good as the fair coin; importantly, the result also holds if we remove the algorithm from the equation and just focus on the loss' minimizer. The paper has been the source of considerable attention, especially stirring up research in the field of loss functions and robust boosting algorithms. It would not do justice to the many citing references to sample a few of them to fit in there, so we have summarised dozens of them, \textit{inclusive of the context of citation}, in the \supplement, Section \ref{sec-wtps}. Notwithstanding mentions in the original papers \cite{lsRC-conf,lsRC} of the existence of noise-tolerant boosting algorithms \cite{ksBI,lsAM} operating on different models\footnote{but whose boosting blueprint does not openly follow the "master routine" in \cite[Section 1.1]{lsRC}, perhaps explaining the path chosen for the citing history of \cite{lsRC-conf,lsRC}.}, almost all citing papers have converged to the high-level tagline that noise defeats convex loss boosters, usually omitting the reference to models in the trio (algorithm, loss, model) that Long and Servedio focused on.

Our paper starts with an apparent and striking paradox based on this synopsis. When they are symmetric, proper losses -- loss functions eliciting Bayes optimal prediction and overwhelmingly popular in ML (log-, square-, Matusita losses) \cite{rwCB,sEO} -- have a dual surrogate form which exactly fits to Long and Servedio's margin loss blueprint \cite{nmSL}. The paradox comes from the fact that on their data, such losses end up eliciting nothing better than a fair coin -- quite arguably far from even the noise-dependent optimal prediction !\\
\textit{Our first contribution} shows that the picture is even worse looking as Long and Servedio's results survive to dropping the "symmetry" constraint on the loss, thus extending their result to any proper loss not necessarily admitting a margin form (yet satisfying differentiability and lower-boundedness of the partial losses, which are weak constraints). So, \textit{where is the glitch} ?\\
\textit{Our second contribution} provides a clue where to look as we show that boosting is neither to blame: we introduce a simple and general "model-adaptive" boosting algorithm (\topdowngen), following the "boosting blueprint" \cite{lsRC} and able to boost a very general class of models generalizing, among others, decision trees, linear separators, alternating decision trees, nearest neighbor classifiers and labeled branching programs. Our main theoretical result is a general margin / edge boosting rate theorem for \topdowngen, which then specialises into specific rates for all classes mentioned; apart from linear separators \cite{sfblBT}, we are not aware of the existence of formal margin-based boosting results for any of the other classes. \topdowngen~
\textit{also} complies with the blueprint boosting algorithm of Long and Servedio's negative results. Hence, if it learns linear separators on Long and Servedio's data, \topdowngen~can spectacularly fail and early hit fair coin prediction; \textit{however}, if it boosts \textit{any other} class mentioned in the list above on Long and Servedio's data, it does learn \textit{Bayes optimal predictor regardless of the noise level}. Toy experiments involving symmetric and asymmetric proper losses confirm the theory: the weakest link in Long and Servedio's results happens to be the model class, not a property of the loss (convexity) nor of the algorithm (the boosting blueprint).

Which brings us to \textit{our third contribution}: Long and Servedio's results show a remarkable failure of a trio (algorithm, loss, model), but as much as our technical results show that it is overshot to- blame singularly the loss-(x) or the algorithm, so would it be to end up blaming the model. As much as class probability estimation (=supervised learning in the proper framework) can be seen as a motherboard / pipeline involving data, loss, algorithm and/or model to estimate a posterior from an observation, we believe that the real culprit appears in \textit{each} part as the dark side of an otherwise "sugar-coated" valued component of ML: \textit{parameterisation} -- parameterisation of a loss that results in it being convex, of an algorithm that results in it emulating a boosting blueprint, of a model that results in a specific architecture, etc. --. We discuss a broad agenda on such issues beyond algorithms, losses and models.

Importantly, the context of Long and Servedio's negative results implies having access to the whole domain for learning, so we shall not discuss the generalization abilities of our algorithm but rather ground its formal analysis in the boosting rates on training -- a standard approach in boosting.

The rest of this paper is as follows: Sections \ref{sec-defs-setting} and \ref{sec-surrogate} introduce definitions that lead to the apparent paradox mentioned; Section \ref{sec-gen} extends the results of \cite{lsRC} to asymmetric proper losses, and Section \ref{sec-boosting} introduces and details results about our boosting algorithm. Section \ref{sec-toy-exp} presents experiments on Long and Servedio's data, Section \ref{sec-discussion} provides the discussion mentioned and concludes.

%% file: content-arxiv/definitions.tex
\section{Definitions and setting}\label{sec-defs-setting}
\paragraph{Losses for class probability estimation} A \textit{loss for class probability estimation} (CPE), $\loss : \mathcal{Y} \times [0,1]
\rightarrow \mathbb{R}$, is expressed as
\begin{eqnarray}
\loss(y,u) & \defeq & \iver{y=1}\cdot \partialloss{1}(u) +
                     \iver{y=-1}\cdot \partialloss{-1}(u), \label{eqpartialloss}
\end{eqnarray}
where $\iver{.}$ is Iverson's bracket \cite{kTN}. Functions $\partialloss{1}, \partialloss{-1}$ are called \textit{partial} losses. A CPE loss is
\textit{symmetric} when $\partialloss{1}(u) = \partialloss{-1}(1-u),
\forall u \in [0,1]$ \cite{nnOT}, \textit{differentiable} when
its partial losses are differentiable and \textit{lower-bounded} when its partial losses are lowerbounded.

The pointwise conditional risk of local guess $u \in [0,1]$ with respect to a ground truth $v \in [0,1]$ is:  
\begin{eqnarray}
  \poirisk(u,v) & \defeq & v\cdot \partialloss{1}(u) + (1-v)\cdot \partialloss{-1}(u) \label{eqpoirisk}.
\end{eqnarray}
A loss is \textit{proper} iff for any ground truth $v \in [0,1]$, $\poirisk(v,v) = \inf_u \poirisk(u,v)$, and strictly proper iff $u=v$ is the sole minimiser \cite{rwID}. The (pointwise) \textit{Bayes} risk is $\poibayesrisk(v) \defeq \inf_u \poirisk(u,v)$. For proper losses, we thus have:
\begin{eqnarray}
 \poibayesrisk(v) & = & v\cdot \partialloss{1}(v) + (1-v)\cdot \partialloss{-1}(v).\label{defpoibayesrisk}
\end{eqnarray}
Proper losses have a long history in statistics and quantitative psychology that long predates their use in ML \cite{rwCB,samAP}.
Hereafter, unless otherwise stated, we assume the following about the loss at hand:
\begin{enumerate}
\item $|\poibayesrisk(0)|, |\poibayesrisk(1)|, |\partialloss{1}(1)|, |\partialloss{-1}(0)| \neq \infty$;
  \item the loss is strictly proper and differentiable (we call such losses \spd~for short).
  \end{enumerate}
  Conventional proper losses like the log-, square- or Matusita- are \spd~losses with $\poibayesrisk(0) = \poibayesrisk(1) = \partialloss{1}(1) = \partialloss{-1}(0) = 0$. Losses satisfying $\partialloss{1}(1) = \partialloss{-1}(0) = 0$ are called \textit{fair} in \cite{rwCB}. 

  \paragraph{Population loss} Usually in ML, we are given a training sample $\mathcal{S} \defeq \{(\ve{x}_i, y_i), i= 1, 2, ..., m\}$ where $\ve{x}_i$ is an observation from a domain $\mathcal{X}$ and $y_i \in \mathcal{Y} \defeq \{0,1\}$ a binary representation for classes in a two-classes problem (0 goes for the "negative class", 1 for the "positive class"). In the CPE setting, we wish to learn an estimated posterior $\estposterior : \mathcal{X} \rightarrow [0,1]$, and to do so, following some of \cite{lsRC}'s notations, we wish to learn $\estposterior$ by 
minimizing a population loss called a risk:
\begin{eqnarray}
  \popsur(\estposterior, \mathcal{S}) & \defeq & \expect_{i\sim [m]}\left[\loss(y_i, \estposterior(\ve{x}_i))\right] \label{eqpopsurrisk1}.
\end{eqnarray}
As already explained in the introduction, we assume training on the whole domain to fit in the framework of \cite{lsRC}'s negative results, so the question of the generalisation abilities of models does not arise. In such a case, Bayes rule can be computed from the training data.

%% file: content-arxiv/surrogate-losses.tex
\section{Surrogate losses and a proper paradox}\label{sec-surrogate}

\paragraph{Link and canonical losses} The \textit{inverse link} of a \spd~loss is:
\begin{eqnarray}
\estposterior(z) & \defeq & ({{-\poibayesrisk}'})^{-1}(z)\label{maxliketa}.
\end{eqnarray}
One can check that for any \spd~loss, $\mathrm{Im}(\estposterior) = [0,1]$, and it turns out that the inverse link provides a maximum likelihood estimator of the posterior CPE given a learned real-valued predictor $h:\mathcal{X} \rightarrow \mathbb{R}$ \cite[Section 5]{nnOT}. A substantial part of ML learns naturally real-valued models (from linear models to deep nets) so the link is important to "naturally" embed the prediction in a CPE loss. A loss using its own link for the embedding is called a \textit{canonical loss} \cite{rwCB}. One can use a different link, in which case the loss is called "composite" but technical conditions arise to keep the whole construction proper \cite{rwCB} so we restrict ourselves to the simplest case of proper canonical losses and call them proper for short. When used with real-valued prediction, \spd~losses have a remarkable analytical form -- called in general a \textit{surrogate loss} \cite{nnOT} (and references therein) -- from which directly arises the apparent paradox we mentioned in the introduction.

\paragraph{Surrogate losses} It comes from \textit{e.g.} \cite[Theorem 1]{nmSL} that any \spd~loss can be written for a real valued classifier $h:\mathcal{X} \rightarrow \mathbb{R}$ on example $(\ve{x},y)$ with binary-described class $y \in \mathcal{Y}$ as:
\begin{mdframed}[style=MyFrame,nobreak=true,align=center]
\begin{eqnarray}
  \loss(y, h(\ve{x})) & = & D_{{-\poibayesrisk}}\left(y \| {{-\poibayesrisk}'}^{-1}(h(\ve{x}))\right)\nonumber\\
                      & \defeq & {-\poibayesrisk}(y) + {(-\poibayesrisk)}^\star(h(\ve{x})) - y h(\ve{x})\nonumber\\
  & = & {-\poibayesrisk}(y) + \underbrace{\philoss(-h(\ve{x})) - y h(\ve{x})}_{\mbox{model dependent term}},\label{deflossgen}
\end{eqnarray}
with
\begin{eqnarray}
\philoss(z) & \defeq & (-\poibayesrisk)^\star(-z).\label{defUpphi}
\end{eqnarray}
\end{mdframed}
We single out function $\philoss$ to follow notations from \cite{lsRC} (we add $\loss$ in index to remind it depends on the loss). Here, $D_{-\poibayesrisk}$ is a Bregman divergence with generator ${-\poibayesrisk}$ (definition using the convex conjugate \textit{e.g.} in \cite{anMO}). Note that \eqref{deflossgen} does not fit to the classical margin loss definition in ML (as in, \textit{e.g.}, \cite{lsRC}), \textit{however}, when the loss is in addition symmetric -- which happens to be the case for most ML losses like log-, square-, Matusita, etc. --, the formula simplifies further to a margin loss formulation. Indeed, we remark ${\poibayesrisk}(u) = {\poibayesrisk}(1-u)$ and it comes ${(-\poibayesrisk)}^\star(-z) = {(-\poibayesrisk)}^\star(z) - z$. Using a "dual" real-valued class $y^* \in \mathcal{Y}^* \defeq \{-1,1\}$ (1 still goes to the positive class), we can rewrite the loss as
\begin{mdframed}[style=MyFrame,nobreak=true,align=center]
\begin{eqnarray}
\loss(y^*, h(\ve{x})) & = & {-\poibayesrisk}\left(\frac{1+y^*}{2}\right) + \underbrace{\philoss(y^* h(\ve{x}))}_{\mbox{model dependent term}}. \label{deflosssym}
\end{eqnarray}
\end{mdframed}
At this stage, it is important to insist that the "loss", in the CPE framework, is $\loss$ \eqref{eqpartialloss}. Eq. \eqref{deflosssym} is a \textit{reparameterisation} of it. The distinction is not superficial: the Bayes risk $\poibayesrisk$ in \eqref{defpoibayesrisk} is concave in its probability argument, while $\philoss$ in \eqref{defUpphi} is convex in its real-valued argument. Popular choices for $\loss$, like log-, square-, Matusita, yield as popular forms for $\philoss$, respectively logistic, square and Matusita. They are often called losses as well since they quantify a discrepancy, but equally often they are called \textit{surrogates} (or surrogate losses) for the simple reason that when properly scaled, they yield upperbounds of the "historic loss" of ML, the 0/1 loss \cite{klpvOT}, which with our notations equates $\iver{\mathrm{sign}(h(\ve{x})) \neq y^*}$.

For learning, we can focus only in the model dependent term in \eqref{deflossgen}, \eqref{deflosssym} and thus define the population (surrogate) risk as:
\begin{eqnarray}
  \popsur(h, \mathcal{S}) & \defeq & \expect_{i\sim [m]}\left[\loss(y_i^*, h(\ve{x}_i))\right] + \expect_{i\sim [m]}\left[{\poibayesrisk}\left(\frac{1+y_i^*}{2}\right)\right]\nonumber\\
                          & = &\left\{ \begin{array}{cl}
                                                                     \expect_{i\sim [m]}\left[ \philoss(-h(\ve{x}_i)) - y_i h(\ve{x}_i)\right] & \mbox{(general form)}\\
                                                                     \expect_{i\sim [m]}\left[ \philoss(y_i^* h(\ve{x}_i))\right] & \mbox{(for symmetric losses)}
                                                                    \end{array}\right. \label{eqpopsurrisk2}.
\end{eqnarray}

\paragraph{A Bayes born paradox} There is one technical argument that needs to be shown to relate the surrogate form in \eqref{eqpopsurrisk2} to \cite{lsRC}'s results: we need to show that the corresponding surrogates of any symmetric \spd~loss fits to their blueprint margin loss.
\begin{lemma}\label{lem-PHI1}
For any \spd~loss, $\philoss$ is $C^1$, convex, decreasing, has $\philoss'(0)<0$ and $\lim_{z\rightarrow +\infty} \philoss(z) = \poibayesrisk(0)$. 
\end{lemma}
Proof in \supplement, Section \ref{proof-lem-PHI1}. Hence, if we offset the constant $\poibayesrisk(0)$ or just assume it is 0, any \textit{symmetric} \eqref{deflosssym} \spd~loss fits to \cite[Definition 1]{lsRC}.

We now explain \cite[Section 4]{lsRC}'s data. The domain $\mathcal{X} = \mathbb{R}^2$ and we have a sample
\begin{mdframed}[style=MyFrame,nobreak=true,align=center]
  \begin{eqnarray}
    \mathcal{S}_{\mbox{\tiny{clean}}} & \defeq & \left\{\left(\left[
                      \begin{array}{c}
                        1\\
                         0
                        \end{array}
  \right],1\right), \left(\left[
                      \begin{array}{c}
                        \gamma\\
                         -\gamma
                        \end{array}
  \right],1\right), \left(\left[
                      \begin{array}{c}
                        \gamma\\
                         -\gamma
                        \end{array}
  \right],1\right), \left(\left[
                      \begin{array}{c}
                        \gamma\\
                         K\gamma
                        \end{array}
  \right],1\right)\right\}. \label{defSclean}
  \end{eqnarray}
\end{mdframed}
In \cite{lsRC}, $K=5$ and $\gamma>0$ is a margin parameter. Since all labels are positive, we easily get Bayes prediction, $\posterior(\ve{x}) = 1 = \pr[\Y =1 | \X = \ve{x}]$. In the setting of \cite{lsRC}, it is a simple matter to check that the optimal real-valued linear separator (\cls) $h$ minimizing $\popsur(h, \mathcal{S}_{\mbox{\tiny{clean}}})$ makes zero mistakes on predicting labels for $\mathcal{S}_{\mbox{\tiny{clean}}}$. One would expect this to happen since the loss $\loss$ at the core is proper, yet this seems to all go sideways \textit{as soon as} label noise enters the picture. We replace $\mathcal{S}_{\mbox{\tiny{clean}}}$ by a noisy (multiset or bag) version $\mathcal{S}_{\mbox{\tiny{noisy}}}$,
\begin{mdframed}[style=MyFrame,nobreak=true,align=center]
  \begin{eqnarray}
    \mathcal{S}_{\mbox{\tiny{noisy}}} & \defeq & N\mbox{ copies of } \mathcal{S}_{\mbox{\tiny{clean}}} \cup 1\mbox{ copy of } \mathcal{S}_{\mbox{\tiny{clean}}} \mbox{with labels flipped}.
  \end{eqnarray}
\end{mdframed}
This mimics a symmetric label noise level $\etanoise = 1/(N+1)$, with $N>1$ \cite{lsRC}. The paradox mentioned above comes from the following two observations: (i) Bayes posterior prediction with noise becomes $\posterior(\ve{x}) = 1 - \etanoise > 1/2$, which still makes no error on $\mathcal{S}_{\mbox{\tiny{clean}}}$, and (ii) \cite{lsRC} show that regardless of this noise level, for any margin loss $\philoss$ complying with Lemma \ref{lem-PHI1}, the optimal model \textit{is as bad as the fair coin on $\mathcal{S}_{\mbox{\tiny{clean}}}$}. Since symmetric \spd~losses \eqref{eqpopsurrisk2} fit to Lemma \ref{lem-PHI1}, \cite{lsRC}'s optimal model should have the same properties as Bayes' predictor, yet this clearly does not happen. The picture looks even gloomier as algorithms enter the stage: despite its acclaimed performances \cite{fhtAL}, boosting can perform so badly that after a \textit{single} iteration its "strong" model hits the fair coin prediction. Not only do we hit a paradox from the standpoint of the optimal model, we also observe a stark failure of a powerful algorithmic machinery. There are clearly "things that break" in the trio (algorithm, loss, model) in the context of \cite{lsRC}. Post-\cite{lsRC} work clearly shed light on the algorithm and loss as the culprits (Section \ref{sec-wtps}).

In the context of properness, we have shown that the "margin form" parameterisation of the loss used by \cite{lsRC} is in fact not mandatory as asymmetric losses do not comply with it. Because asymmetry alleviates ties between partial losses, one could legitimately hope that it could address the paradox. We now show that it is not the case as \cite{lsRC}'s results mentioned above still stand without symmetry.

%% file: content-arxiv/generalisation.tex
\section{Long and Servedio's results hold without symmetry}\label{sec-gen}

We reuse some of \cite{lsRC}'s notations and first denote \bphiideal~the algorithm returning the optimal linear separator (\cls) $h$ minimizing \eqref{eqpopsurrisk2}.
\begin{lemma}\label{lem-GEN1}
For any $N>1$, there exists $\gamma> 0, K>0$ such that when trained on $\mathcal{S}_{\mbox{\tiny{noisy}}}$, \bphiideal's classifier has at most $50\%$ accuracy on $\mathcal{S}_{\mbox{\tiny{clean}}}$.
\end{lemma}
Proof in \supplement, Section \ref{proof-lem-GEN1}. The proof displays an interesting phenomenon for asymmetric losses, which is not observed on \cite{lsRC}'s results. If the noise $\etanoise$ is large enough and the asymmetry such that $\philoss'(0) < \etanoise - 1$, then the optimal classifier can do more than $50\%$ mistakes on $\mathcal{S}_{\mbox{\tiny{clean}}}$ -- thus perform worse than the unbiased coin. This cannot happen with symmetric losses since in this case $\philoss'(0) = -1/2$ and we constrain $\etanoise < 1/2$. What this shows is that asymmetry, while accomodating non-trivial different misclassification costs depending on the class, can lead to non-trivial pitfalls over noisy data.

Similarly to \cite{lsRC}, we denote \bphiearly~the booster of \eqref{eqpopsurrisk2} which proceeds by following the boosting blueprint as described in \cite{lsRC}; we assume that the weak learner chooses the weak classifier offering the largest absolute edge \eqref{defWLA3}, returning nil if all possible edges are zero (and then the booster stops). We let $\mathcal{S}_{\mbox{\tiny{clean}},\theta}, \mathcal{S}_{\mbox{\tiny{noisy}},\theta}$ denote $\mathcal{S}_{\mbox{\tiny{clean}}}, \mathcal{S}_{\mbox{\tiny{noisy}}}$ with observations rotated by an angle $\theta$. 
\begin{lemma}\label{lem-Rot}
For any $N>1, T\geq 1$, there exists $\gamma> 0, K>0, \theta \in [0,2\pi]$ such that when trained on $\mathcal{S}_{\mbox{\tiny{noisy}},\theta}$, within at most $T$ boosting iterations \bphiearly~hits a classifier at most $50\%$ accurate on $\mathcal{S}_{\mbox{\tiny{clean}},\theta}$.
\end{lemma}
Proof in \supplement, Section \ref{proof-lem-Rot}. 

%% file: content-arxiv/boosting.tex
\section{The boosting blueprint does provide a fix}\label{sec-boosting}

We investigate a new boosting algorithm learning model architectures that generalise those of decision trees and linear separators, among other model classes. We call such models partition-linear models (\pbls). The algorithm boosts any \spd~loss using the blueprint boosting algorithm of \cite{lsRC}. To our knowledge, it is the first boosting algorithm which can provably boost asymmetric proper losses, which is non trivial as it involves two different forms of the corresponding surrogate that are not compliant with the classical margin representation \cite{lsRC}. A simple way to define a \pbls~$H_t$ from a sequence of triples $(\alpha_j,h_j,\mathcal{X}_j)_{j\in [t]}$ (where $\alpha_j \in \mathbb{R}, h_j \in \mathbb{R}^{\mathcal{X}}, \mathcal{X}_j \subseteq \mathcal{X}$) is, for $t\geq 1$:
\begin{eqnarray}
  H_{t}(\ve{x}) & = & \left\{
              \begin{array}{ccl}
                H_{t-1}(\ve{x}) + \alpha_t h_t(\ve{x}) & \mbox{ if } & \ve{x} \in \mathcal{X}_t\\
                H_{t-1}(\ve{x}) & \multicolumn{2}{l}{\mbox{ otherwise}}
                \end{array}
                                  \right. , \label{eqHt}\\
  & = & \sum_{t=1}^T \iver{\ve{x} \in \mathcal{X}_t} \cdot \alpha_t h_t(\ve{x}),
\end{eqnarray}
and we add $H_0(\ve{x}) \defeq 0, \forall \ve{x} \in \mathcal{X}$\footnote{We can equivalently consider that $h_t = 0$ in $\mathcal{X}\backslash \mathcal{X}_t$. We opt for \eqref{eqHt} since it makes a clear distinction for $\mathcal{X}_t$. Notice that this setting generalizes boosting with weak hypotheses that abstain \cite{ssIB}.}. We also define the weight function
\begin{eqnarray}
w((\ve{x}, y), H) & \defeq & y - y^* \cdot ({-\poibayesrisk'})^{-1}(H(\ve{x})),\label{defWEIGHTS}
\end{eqnarray}
which is in $[0,1]$. Notice we use both (real and binary) class encodings in the weight function, recalling the relationship $y^* \defeq 2y - 1 \in \{-1,1\}$. Algorithm \topdowngen~presents the boosting approach to learning \pbls. Note that \topdowngen~picks the subset of $\mathcal{S}$ on which to evaluate this hypothesis, which corresponds \textit{e.g.} in decision trees induction to the choice of a leaf to split via a call to the weak learner \cite{kmOT}.
\begin{algorithm}[t]
\caption{\topdowngen $(\mathcal{S}, \loss, \weak, T)$}\label{splitMF}
\begin{algorithmic}
  \STATE  \textbf{Input:} Dataset $\mathcal{S} = \{(\ve{x}_i,y_i)\}_{i=1}^m$, \spd~loss $\ell$, weak learner \weak, iteration number $T\geq 1$;
  \STATE  \textbf{Output:} \pbls~$H_T$;
  \STATE  Step 1 : $\forall i \in [m], w_{i,1} \defeq w((\ve{x}_i, y_i), H_0)$ // weight initialisation
  \STATE  Step 2 : \textbf{for} $t=1, 2, ..., T$
  \STATE \hspace{1cm} Step 2.1 : pick $\mathcal{X}_{t} \subseteq \mathcal{X}$; 
  \STATE \hspace{1cm} Step 2.2 : $h_{t} \leftarrow \weak (\bm{w}^*_t, \mathcal{S}_{t})$; 
  \STATE \hspace{2.5cm} // weak learner call: $\mathcal{S}_{t} \defeq \{(\ve{x}_i, y_i) \in \mathcal{S}: \ve{x}_i \in \mathcal{X}_t\}$; $\bm{w}^*_t \defeq \bm{w}_t \mbox{ restricted to } \mathcal{S}_{t}$;
  \STATE \hspace{1cm} Step 2.3 : compute $\alpha_{t}$ as the solution to:
    \begin{eqnarray}
\sum_{i\in [m]_{t}} w((\ve{x}_i, y_i), H_{t}) \cdot y_i^* h_{t}(\ve{x}_i)  & = & 0; \label{defAlpha}
    \end{eqnarray}
      \STATE \hspace{2.5cm} // $[m]_{t}$ = indices of $\mathcal{S}$ in $\mathcal{S}_t$; $\alpha_t$ appears in $H_t$, see \eqref{eqHt}
      \STATE \hspace{1cm} Step 2.4 : $\forall i \in [m]_t, w_{t+1,i} \defeq w((\ve{x}_i, y_i), H_{t})$
        \STATE \hspace{2.5cm} // weight update
        \STATE  \textbf{return} $H_T(\ve{x}) \defeq \sum_{t=1}^T \mathrm{1}_{\ve{x} \in \mathcal{X}_t} \cdot \alpha_t h_t(\ve{x})$;
\end{algorithmic}
\end{algorithm}

\textbf{Solutions to \eqref{defAlpha} are finite} We assume without loss of generality that $\pm h_t$ does not achieve 100$\%$ accuracy over $\mathcal{S}_t$ (Step 2.2; otherwise there would be no need for boosting, at least in $\mathcal{X}_t$) and that $\max_{\mathcal{S}_t} |h_t| \ll \infty$, "$\ll \infty$" denoting finiteness. 
\begin{lemma}\label{solALPHA}
The solution to \eqref{defAlpha} satisfies $|\alpha_t|\ll \infty$.
\end{lemma}
Proof in \supplement, Section \ref{proof-lem-solALPHA}.

\paragraph{The boosting abilities of \topdowngen} Given a real valued classifier $H$ and an example $(\ve{x}, y^*)$, we define the (unnormalized) edge or margin of $H$ on the example as $y^* H(\ve{x})$ \cite{nnOT,sfblBT}, a quantity that integrates both the accuracy of classification (its sign) and a confidence (its absolute value). Formal guarantees on edges / margins are not frequent in boosting \cite{nnARj,sfblBT}. We now provide one such general guarantee for \topdowngen. While requirements on the weak hypotheses follow the weak learning assumption of boosting, the constraints on the loss itself are minimal: they essentially require it to be \spd~with partial losses meeting a lower-boundedness condition and a condition on derivatives.

\begin{definition}\label{defUT-COMP}
  Let $\{u_t\}_{t\in \mathbb{N}}$ be a sequence of strictly positive reals. We say that the choice of $\mathcal{X}_t$ in Step 2.1 of \topdowngen~is "$u_t$ compliant" iff, letting $J(\mathcal{W},t) \defeq \mathrm{Card}(\mathcal{W}) \cdot (\expect_{i\sim \mathcal{W}}[w_{t,i}])^2$ where $\mathcal{W} \subseteq [m]$, Step 2.1 guarantees:
\begin{eqnarray}
  J([m]_t,t) & \geq & u_t \cdot J([m],t), \forall t = 1, 2, ... \label{condUT},
\end{eqnarray}
\end{definition}
Notice that the sum of terms $\sum_{t=1}^T u_t$ is strictly increasing and thus invertible. Let $U:\mathbb{N}_* \rightarrow \mathbb{R}_+$ such that $U(T) \defeq \sum_{t=1}^T u_t$. The role of $J$ is fundamental in our results and can guide the choice of $\mathcal{X}_t$ in Step 2.1: in short, the larger $u_t$, the better the rates. Lemma \ref{lem-J} gives a concrete and intuitive simplification of $J$ in the case of decision trees. In the most general case, it is good to keep in mind the intuition of boosting that the weight of an example is larger as the outcome of the current classifier gets \textit{worse}. Hence, \eqref{condUT} encourages focus on $\mathcal{X}_t$ with a large number of examples ($\mathrm{Card}([m]_t)$) \textit{and} with large weights ($\expect_{i\sim [m]_t}[w_{t,i}]$) -- hence with subpar current classification.
\begin{theorem}\label{th-boost-pbls}
Suppose the following assumptions are satisfied on the loss and weak learner:
  \begin{itemize}
  \item [\textbf{LOSS}] the loss is strictly proper differentiable; its partial losses are such that $\exists \kappa > 0, C \in \mathbb{R}$,
    \begin{eqnarray}
     \partialloss{-1}(0), \partialloss{1}(1) & \geq & C,\\
\inf \{\partialloss{-1}'-\partialloss{1}' \} & \geq & \kappa.
    \end{eqnarray}
    \item [\textbf{WLA}] There exists a constant $\gammawla >0$ such that at each iteration $t\in [T]$, the weak hypothesis $h_t$ returned by $\weak$ satisfies\footnote{The quantity in the absolute value is sometimes called the (normalized) edge of $h_t$; it takes values in $[-1,1]$.}
  \begin{eqnarray}
\left|\sum_{i\in [m]_t} \frac{w_{t,i}}{\sum_{j\in [m]_t} w_{t,j}} \cdot y^*_i \cdot \frac{h_{t}(\ve{x}_i)}{\max_{j\in [m]_t}|h_t(\ve{x}_j)|}\right| & \geq & \gammawla .\label{defWLA3}
  \end{eqnarray}
\end{itemize}
Then for any $\theta \geq 0, \varepsilon > 0$, letting $\underline{w}(\theta) \defeq \min\{1 - ({-\poibayesrisk'})^{-1}(\theta), ({-\poibayesrisk'})^{-1}(-\theta)\}$, if \topdowngen~is run for at least
  \begin{eqnarray}
T & \geq & U^{-1}\left(\frac{2\left(\popsur(H_{0}, \mathcal{S}) - C\right)}{\kappa \cdot \varepsilon^2 \underline{w}(\theta)^2 \gammawla^2}\right)\label{eqWITHPHIX}
  \end{eqnarray}
iterations, then we are guaranteed
\begin{eqnarray}
\pr_{i\sim [m]}[y^*_i H_T(\ve{x}_i) \leq \theta] & < & \varepsilon.\label{boundEdges}
\end{eqnarray}
Here, $U$ is built from a sequence of $u_t$ such that the choice of $\mathcal{X}_t$ in Step 2.1 is $u_t$ compliant.
  \end{theorem}
  Proof in \supplement, Section \ref{proof-th-boost-pbls}. The proof of the Theorem involves as intermediate step the proof that the surrogate $\popsur(H_T, \mathcal{S})$ is also boosted, which is of independent interest given \cite{lsRC}'s framework and the potential asymmetry of the loss (Theorem \ref{th-boost-1} in \supplement). Figure \ref{f-all-functions} depicts some key functions used in \topdowngen~and Theorem \ref{th-boost-pbls}.
\begin{figure}[t]
\begin{center}
\includegraphics[trim=30bp 280bp 500bp 20bp,clip,width=0.45\linewidth]{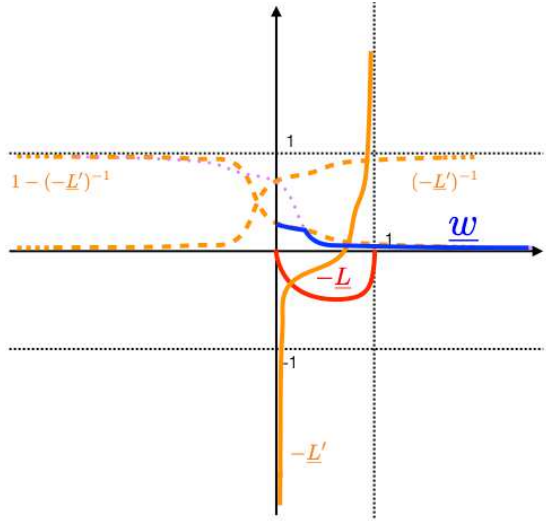} 
\end{center}
\caption{Schematic depiction of key functions used for weights \eqref{defWEIGHTS} and Theorem \ref{th-boost-pbls} on an example of loss.}
  \label{f-all-functions}
\end{figure}
  
  \begin{remark}
    The \textbf{LOSS} requirements are weak. It can be shown that strict properness implies $\inf \{\partialloss{-1}'-\partialloss{1}' \} > 0$ \cite[Theorem 1]{rwCB}; since the domain of the partial losses is closed, we are merely naming the strictly positive infimum with condition $\inf \{\partialloss{-1}'-\partialloss{1}' \} \geq \kappa > 0$. The "extremal" value condition for partial losses ($\partialloss{-1}(0), \partialloss{1}(1) \geq C$) is also weak as if it did not hold, partial losses would not be lower-bounded on each's respective best possible prediction, which would make little sense. Usually, $C=0$ (the best predictions occur no loss) and such losses, that include many popular choices like square-, log-, Matusita, are called "fair" \cite{rwCB}.
\end{remark}

We now give five possible instantiations of \topdowngen, each with separate discussion about $u_t$ compliance and boosting rates. We start by the two most important ones (linear separators and decision trees), providing additional details for decision trees on how \topdowngen~emulates well known algorithms.

\paragraph{Application of \topdowngen~$\#1$: linear separators (\cls)} This is a trivial use of \topdowngen.

\noindent $\triangleright$ \underline{\textit{$u_t$ compliance and the weak learner}}: $\mathcal{X}_t = \mathcal{X}, \forall t$ so we trivially have $u_t = 1 (\forall t)$ compliance and the weak learner returns an index of a feature to leverage.

\noindent $\triangleright$ \underline{\textit{Boosting rate}}: we have the guarantee that $\pr_{i\sim [m]}[y^*_i H_T(\ve{x}_i) \leq \theta] <  \varepsilon$ if
      \begin{eqnarray}
T & \geq & \underbrace{\frac{2\left(\popsur(H_{0}, \mathcal{S}) - C\right)}{\kappa \cdot \varepsilon^2 \underline{w}(\theta)^2 \gammawla^2}}_{\defeq \bcls} = \tilde{O} \left(\frac{1}{\varepsilon^2 \gammawla^2}\right),\label{boost-rate-LS}
      \end{eqnarray}
      a dependence (the tilda removes dependences in other factors) that fits to the general optimal lower-bound in $\gammawla$ \cite{aghmBS} but is suboptimal in $\varepsilon$, albeit not far from the lowerbound of $O(1/\varepsilon)$ \cite{tAP}. Both the algorithm and its analysis generalises a previous one for linear separators and symmetric losses \cite{nnOT}.

      \noindent $\triangleright$ \underline{\textit{Effect of Long and Servedio's data}}: Since \topdowngen~falls in the negative result's boosting blueprint of \cite[Section 2.5]{lsRC}, it does face the negative result of \cite{lsRC}. In fact, we can show a more impeding result directly in the setting of Lemma \ref{lem-GEN1}, \textit{i.e.} without the rotation trick of Lemma \ref{lem-Rot}, as with the square loss (which allows to compute quantities in closed form), \topdowngen~hits a classifier as bad as the fair coin on $\mathcal{S}_{\mbox{\tiny{clean}}}$ in at most 2 iterations. This is shown and discussed in \supplement, Section \ref{proof-sidenegative}.

\paragraph{Application of \topdowngen~$\#2$: decision trees (\cdt)} This is a slightly more involved use of \topdowngen, from the "location" of the weak learner to the perhaps surprising observation that in this case, \topdowngen~emulates and generalizes well known top-down induction schemes.

\noindent $\triangleright$ \underline{\textit{$u_t$ compliance and the weak learner}}: we investigate $u_t$ compliance from the general case where $[m]_t \in \mathcal{P}([m])$, where $\mathcal{P}([m])$ is a partition of $[m]$ in $N_t$ subsets. Jensen's inequality brings
        \begin{eqnarray*}
          \sum_{\mathcal{W} \in \mathcal{P}([m])} J(\mathcal{W},t) &= & m\cdot \sum_{\mathcal{W} \in \mathcal{P}([m])} \frac{J(\mathcal{W},t)}{m} \\
          & = & m \cdot \expect_{\mathcal{W} \sim \mathcal{P}([m])} \left(\expect_{i\sim \mathcal{W}}[w_{t,i}]\right)^2 \\
          & \geq & m \cdot \left(\expect_{\mathcal{W} \sim \mathcal{P}([m])} \left[\expect_{i\sim \mathcal{W}}[w_{t,i}]\right]\right)^2\nonumber\\
          & & = m \cdot \left(\expect_{i\sim [m]}[w_{t,i}]\right)^2 =  J([m],t),
        \end{eqnarray*}
        therefore there exists $\mathcal{W}^* \in \mathcal{P}([m])$ such that $J(\mathcal{W}^*,t) \geq (1/N_t) \cdot J([m],t)$ and picking any such "heavy" subset of indices $[m]_t = \mathcal{W}^*$ guarantees $u_t$ compliance for $u_t = 1/N_t$. Applied to a decision tree with $t$ leaves, we see that we can guarantee $u_t \geq 1/t$, implying $\mathcal{X}_t$ is the domain of a leaf of the current tree and the weak learner is used to find splits. 

        \noindent \textit{Regarding the weak learner}, \topdowngen~iteratively replaces a leaf in the current tree by a decision stump. There are two strategies for that: the first consists in asking the weak learner for one complete split, just like in \cite{kmOT}, but \topdowngen~would then fit a single correction (leveraging coefficient $\alpha_.$) for both leaves and this would be suboptimal. To correct every single leaf prediction separately, we let the weak learner return a split and a corresponding real-valued prediction for \textit{half the split}, \textit{e.g.} for "split$\_$predicate = true". Quite remarkably, we show that if this meets the \textbf{WLA}, then so does \textit{the other half} (for "split$\_$predicate = false"). In other words, we get two \textbf{WLA} compliant weak hypotheses for the price of a single query to the weak learner, and both turn out to define the split sought. This is formalized in the following Lemma, which assumes wlog that the split variable is $x_i$, continuous.
\begin{lemma}\label{lem-split-gives-wla}
Suppose the weak learner returns $\mathrm{1}_{x_i \geq a} \cdot h_t$ ($h_t \in \mathbb{R}_*$ constant) that meets the \textbf{WLA} for the half split. Then the "companion" hypothesis $h'_t(\ve{x}) \defeq \mathrm{1}_{x_i < a} \cdot (-h_t)$ satisfies the \textbf{WLA}.
\end{lemma}
Proof in \supplement, Section \ref{proof-lem-split-gives-wla}. The choice of the leaf to split is simple: denote $\leafset(H)$ the set of leaves of \cdt~$H$ and $\leaf$ a general leaf. Since the leaves of a \cdt~induce a partition of the tree, we denote $J(\leaf)$ the expression of $J(\mathcal{W},t)$ for $\mathcal{W} = \{i:\ve{x}_i \mbox{ reaches }\leaf\}$, omitting index $t$ for readability. Let us analyze what $\mathcal{W}^*$ would satisfy in this case.
      \begin{lemma}\label{lem-J}
        We have
        \begin{eqnarray}
          J(\leaf) & \propto & p_\leaf \cdot (\underbrace{p^+_\leaf (1-p^+_\leaf)}_{=\bayessqrisk(p^+_\leaf)})^2,
        \end{eqnarray}
        where $p_\leaf \defeq m_\leaf / m, p^+_\leaf \defeq m^+_\leaf / m_\leaf$, $m_\leaf \defeq \mathrm{Card}(\{i:\ve{x}_i \mbox{ reaches }\leaf\}), m^+_\leaf \defeq \mathrm{Card}(\{i:\ve{x}_i \mbox{ reaches }\leaf \wedge y_i = 1\})$ and $\bayessqrisk(u) = u(1-u)$ is Bayes risk for the square loss.
      \end{lemma}
      Proof in \supplement, Section \ref{proof-lem-J}. Hence, the leaf to split in Step 2.1 has a good compromise between its "weight" ($p_\leaf$) and its local error (since $2 p^+_\leaf (1-p^+_\leaf) \geq \min\{p^+_\leaf,1-p^+_\leaf\}$). In traditional "tree-based" boosting papers (from \cite{kmOT} to \cite{nwLO}), one usually picks the heaviest leaf ($=\arg\max_\leaf p_\leaf$) but it may well be a leaf with zero error -- thus preventing boosting through splitting. Inversely, focusing only on large error to pick a leaf might point to leaves with too small weights to bring overall boosting compliance. Criterion $J(.)$ strikes a balance weight vs error in the choice.\\
\noindent $\triangleright$ \underline{\textit{Boosting rate}}: We have $u_t \geq 1/t$, with $\sum_{t=1}^T u_t \geq \int_0^T \mathrm{d}z/(1+z) = \log(1+T) \defeq U(T)$, and so we are guaranteed that $\pr_{i\sim [m]}[y^*_i H_T(\ve{x}_i) \leq \theta] <  \varepsilon$ if
      \begin{eqnarray}
T & \geq & \underbrace{\exp\left(\frac{2\left(\popsur(H_{0}, \mathcal{S}) - C\right)}{\kappa \cdot \varepsilon^2 \underline{w}(\theta)^2 \gammawla^2}\right)}_{\defeq \bcdt} = \exp(\bcls) = \exp \tilde{O} \left(\frac{1}{\varepsilon^2 \gammawla^2}\right)\label{boost-rate-DT}
      \end{eqnarray}
     ($\bcls$ defined in \eqref{boost-rate-LS}), which is comparable at $\theta = 0$ to the bound of \cite[Theorem 1]{kmOT} for CART and otherwise generalizes their results to margin/edge-based bounds. \\
      \noindent $\triangleright$ \underline{\textit{Miscellaneous}}:  we finish by a last analogy between \topdowngen~and classical \cdt~induction algorithms: there is a simple closed form solution for the leveraging coefficients $\alpha_.$ that simplifies the loss.
\begin{lemma}\label{lem-BOO2}
  Running \topdowngen~to learn a decision tree $H$ gives
  \begin{eqnarray}
\popsur(H, \mathcal{S}) & = & \expect_{\leaf \sim \leafset(H)}\left[\poibayesrisk (p^+_\leaf) \right], \label{eqCC}
  \end{eqnarray}
  where we recall $p^+_\leaf \defeq m^+_\leaf / m_\leaf$ and the weight of $\leaf$ is $m_\leaf / m$. Furthermore, the \topdowngen~prediction computed at leaf $\leaf$, $H_\leaf$, is:
  \begin{eqnarray}
    H_{\leaf} & = & ({-\poibayesrisk'}) \left(p^+_\leaf\right).
  \end{eqnarray}
\end{lemma}
Proof in \supplement, Section \ref{proof-lem-BOO2}. We conclude that running \topdowngen~to learn a decision tree is largely equivalent to the minimisation of classical \cdt~induction criteria \cite{bfosCA,qC4,kmOT,nwLO}, and our boosting rate analysis generalizes those to asymmetric losses and edge / margin bounds. One can also finally notice that we can easily transform a \cdt~learned using \topdowngen~to a classical \cdt~by "percolating" values down to the leaves, see Figure \ref{f-DT-transformation}. such a connection between both types of models is not new as it dates back to \cite{hnnRB} and was later exploited in various work (\textit{e.g.} \cite{lguedjsfsvBM}).
\begin{figure}[t]
  \begin{center}
    \begin{tabular}{c|c|c}
      \includegraphics[trim=5bp 350bp 370bp 10bp,clip,height=0.27\linewidth]{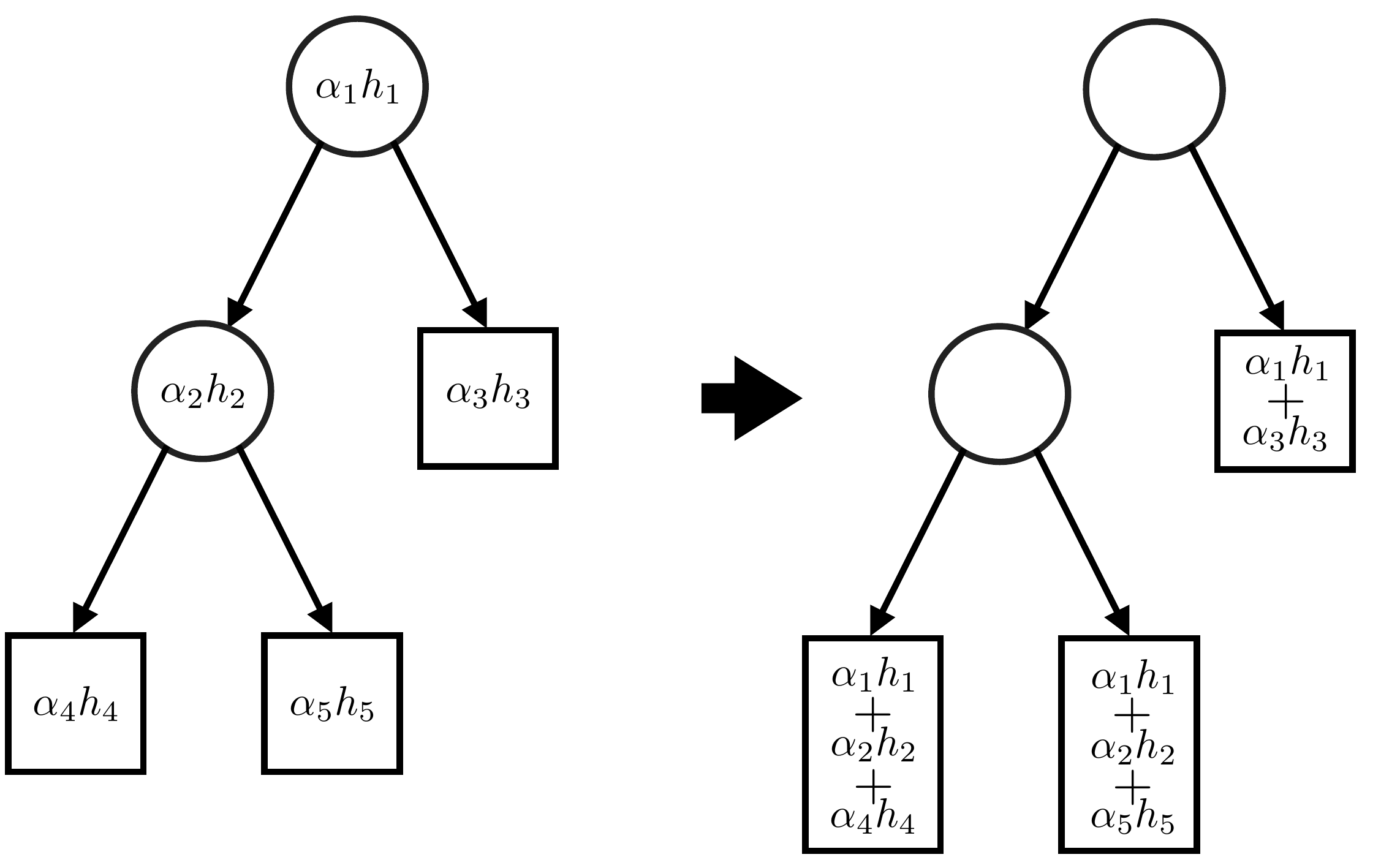} & \includegraphics[trim=5bp 390bp 680bp 20bp,clip,height=0.27\linewidth]{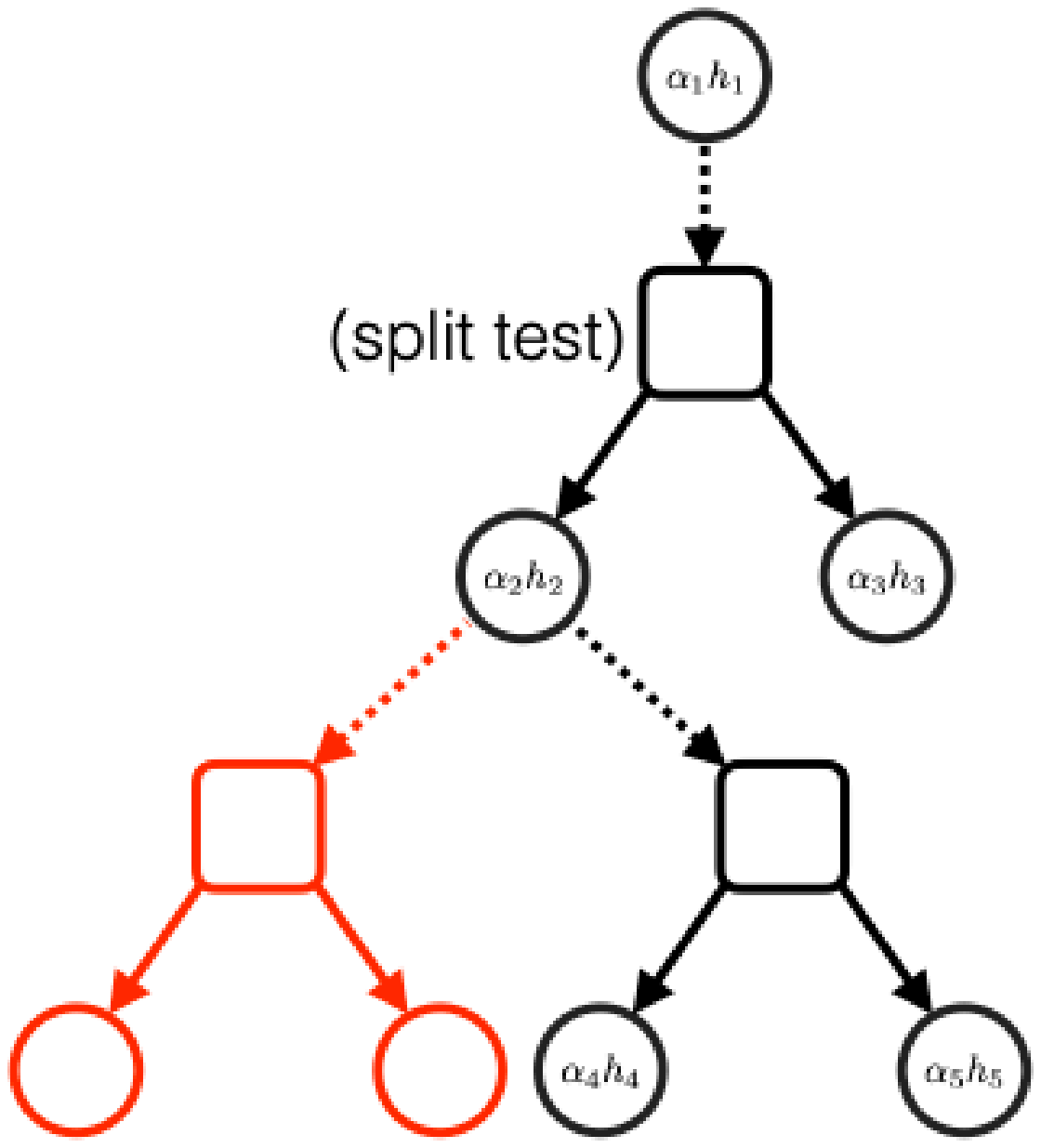} & \includegraphics[trim=50bp 380bp 630bp 20bp,clip,height=0.27\linewidth]{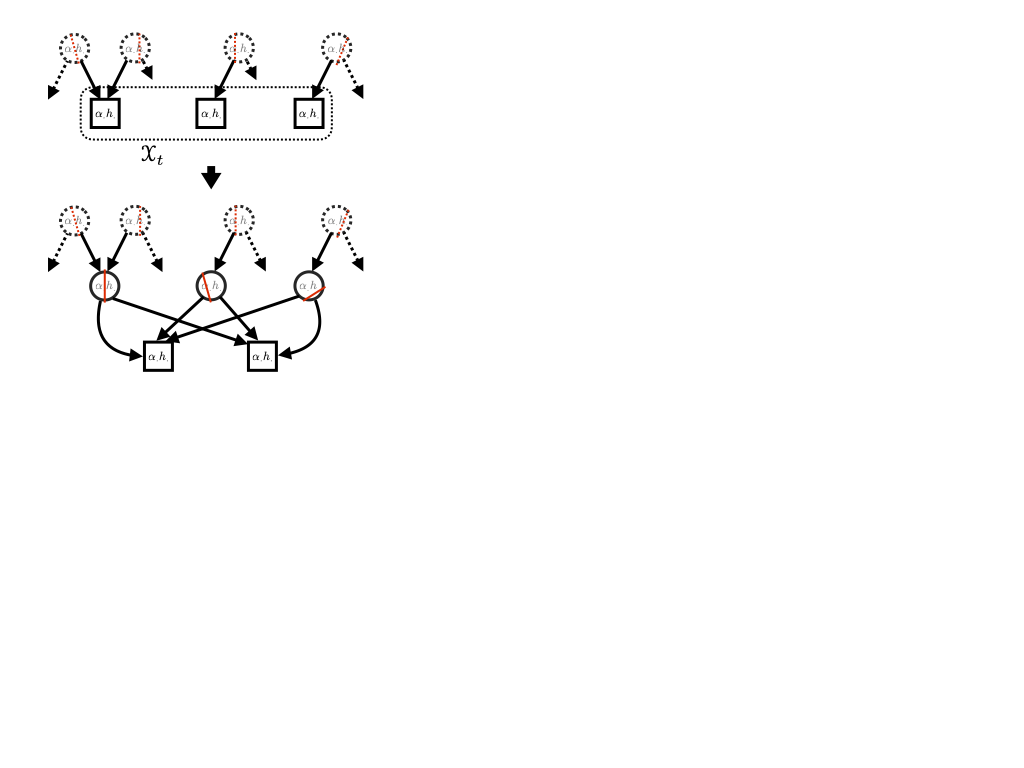} \\
      \end{tabular}
\end{center}
\caption{\textit{Left}: \cdt~learned using \topdowngen~and its "classical" equivalent \cdt. \textit{Center}: an equivalent representation using an alternating decision tree (black) and a more general \cadt~(black + red). \textit{Right}: Using \topdowngen~to learn a \clbp: the main difference with \cdt~is the current $\mathcal{X}_t$ is the union of the domains of several leaves, yielding larger $u_t$s and more efficient boosting.}
  \label{f-DT-transformation}
\end{figure}
      
\noindent $\triangleright$ \underline{\textit{Effect of Long and Servedio's data}}: consider the more general setting where we have noise-free data with $\{0,1\}$ posterior values, and then labels are flipped independently with probability $\etanoise < 1/2$. The noisy proportion of positive examples at a leaf $\leaf$, $\tilde{p}_\leaf^+$, satisfies $\tilde{p}_\leaf^+ = \etanoise + (1-2\etanoise) p_\leaf^+$ ($p_\leaf^+$ = noise-free proportion). Assume for simplicity we know the noise rate in advance (like in \cite{ksBI}) and do not have generalisation issues -- we learn on the whole domain (like \cite[Theorems 4, 6]{ksBI} or on the dataset of \cite{lsRC}). We say that the decision tree learned by \topdowngen~with noise is \textit{not affected by noise} if the \textit{sign} of leaves would be the same as if they were computed without noise. In the context of Long and Servedio's data, this implies $100\%$ accuracy on on the noise-free data.
\begin{lemma}  \label{lem-NoNoise}
  If one of the two conditions is satisfied:
  \begin{itemize}
  \item [(\textbf{S})] the loss $\ell$ is symmetric and $\etanoise < 1/2$, or
  \item [(\textbf{A})] the loss $\ell$ is asymmetric, $\etanoise < \underline{w}(0)$ and leaves are split until ($p^* \defeq ({-\poibayesrisk'})^{-1}(0)$):
    \begin{eqnarray}
\forall \leaf \in \leafset(H), \left(p^+_\leaf \leq \min \left\{\frac{p^*-\etanoise}{1-2\etanoise}, \frac{1}{2}\right\}\right) \vee \left(1-p^+_\leaf \leq \min \left\{\frac{(1-p^*)-\etanoise}{1-2\etanoise}, \frac{1}{2}\right\}\right), \label{condP}
      \end{eqnarray}
    \end{itemize}
    then the \cdt~learned by \topdowngen~is not affected by noise.
  \end{lemma}
  Proof in \supplement, Section \ref{proof-lem-NoNoise}. We have left cases distinct for readability but in fact case (\textbf{A}) encompasses case (\textbf{S}), as this latter implies $\underline{w}(0) = 1/2 = p^*$, making \eqref{condP} always true. If we denote $\mathrm{E}(H)$ the \textit{error} of the tree $H$ learned by \topdowngen~without noise and $\tilde{\mathrm{E}}(H)$ its error over noisified data, then we obtain from (\textbf{S}) that at any stage of the induction,
  \begin{eqnarray}
\tilde{\mathrm{E}}(H) & = & \etanoise + (1-2\etanoise) \mathrm{E}(H),
  \end{eqnarray}
  which is optimal \cite[Section 5]{ksBI}. It is not hard to mirror the "extreme" negative result of Lemma \ref{lem-SQL-run} for \cls~into an extreme positive result for \cdt~on \cite{lsRC}'s data: one easily obtains (from Lemma \ref{lem-BOO2}) that the root prediction of a \cdt, after leveraging in Step 2.3 of \topdowngen~, is equal to Bayes posterior (which is constant over the whole domain). Hence, \topdowngen~converges to Bayes prediction in a \textit{single iteration}, which is confirmed experimentally (Section \ref{sec-toy-exp}).

  \paragraph{Application of \topdowngen~$\#3$: alternating decision trees (\cadt)} Alternating decision trees were introduced in \cite{fmTA}. An \cadt~roughly consists of a root constant prediction and a series of stumps branching from their leaf prediction nodes in a tree graph, see Figure \ref{f-DT-transformation}. The equivalent \cadt~representation of a \cdt~would have outgoing degree 1 for all these stumps' leaves. A general \cadt~makes this outdegree variable and a prediction is just the sum of the prediction along all paths an observation can follow from the \cadt's root node. If a stumps' leaf branches on $N$ stumps, then we sum the $N$ corresponding predictions (and not just 1 for a \cdt).   While using such models is interesting in terms of model's parameterisation, one also sees advantages in terms of boosting, since summing boosted predictions \eqref{boost-rate-LS} is more efficient than branching on boosted predictions \eqref{boost-rate-DT}, but the paper of \cite{fmTA} contains no such rate (note that the loss optimized here is AdaBoost's exponential loss, which is not proper canonical).

  \noindent $\triangleright$ \underline{\textit{$u_t$ compliance and the weak learner}}: these are just combinations of those for \cls~(when increasing a stump's leaf outgoing degree with a new stump) and \cdt~(when finding the test of a stump). Denote $N$-\cadt~the set of \cadt s where non-leaf prediction nodes' outdegree is fixed to be $N$ (inclusive of the root node). Notice that we can then boost while guaranteeing that $u_t = 1$ for $N$ boosting iterations (at the root), then $u_t \geq 1/2$ for $N$ boosting iterations and so on until the last $N$ iterations with $u_t \geq N/T$.

  \noindent $\triangleright$ \underline{\textit{Boosting rate}}: assuming wlog $T$ a multiple of $N$, we have thus $\sum_{t=1}^T u_t \geq N \cdot \sum_{t=1}^{T/N} 1/t \geq N \cdot \int_0^{T/N} \mathrm{d}z/(1+z) = N \cdot \log(1+(T/N)) \defeq U(T)$, and so we are guaranteed that $\pr_{i\sim [m]}[y^*_i H_T(\ve{x}_i) \leq \theta] <  \varepsilon$ if
      \begin{eqnarray}
T & \geq & \underbrace{N \cdot \exp\left(\frac{2\left(\popsur(H_{0}, \mathcal{S}) - C\right)}{N \kappa \cdot \varepsilon^2 \underline{w}(\theta)^2 \gammawla^2}\right)}_{\defeq \bcadt} = N\exp\left(\frac{\bcls}{N}\right) = N \exp \tilde{O} \left(\frac{1}{N \varepsilon^2 \gammawla^2}\right).\label{boost-rate-ADT}
      \end{eqnarray}
      This rate is really meaningful only when $N \leq \bcls$ since otherwise the \cadt~would meet \eqref{boundEdges} from the linear combinations of stumps branched at the root already. Bearing in mind that $\bcdt, \bcadt$ are non-tight lowerbounds, in such a regime, it is easy to see that an \cadt~can be exponentially more efficient than a \cdt, boosting-wise: for example, letting $N = \sqrt{\bcls}$, we obtain $\bcadt \leq \exp(-M\sqrt{\bcls}) \cdot \bcdt$ for some constant $M>0$.

      \noindent $\triangleright$ \underline{\textit{Effect of Long and Servedio's data}}: since learning a \cdt~achieves Bayes optimal prediction with a single root \cdt, the same happens for a single root \cadt, and we get that the \cadt~learned by \topdowngen~is not affected by noise.

      \paragraph{Application of \topdowngen~$\#4$: (leveraged) nearest neighbors (\cnn)} nearest neighbor classification is one of the oldest supervised learning techniques \cite{chNN}. Since we consider real-valued prediction, we implement \cnn~classification by summing a real constant prediction at one observation's neighbors and assume that tie neighbors are included in the voting sample (so one observation can end up with more than $K$ neighbors). Local predictions can have varying magnitudes, which represents a generalisation of nearest neighbor classification where magnitude is constant, but we still call such classifiers nearest neighbors, omitting the "leveraging" part.

      \noindent $\triangleright$ \underline{\textit{$u_t$ compliance and the weak learner}}: the weak learner returns an example to leverage and thus $\mathcal{X}_t$ is its \textit{reciprocal neighborhood} (the set of examples for which it belongs to the $K$-\cnn). We assume wlog there are no "outliers" for classification, so the minimum size of this neighborhood is some $K_{\mbox{\tiny rec}}>0$, yielding $u_t = K_{\mbox{\tiny rec}}/m, \forall t$.

      \noindent $\triangleright$ \underline{\textit{Boosting rate}}: we immediately get $\pr_{i\sim [m]}[y^*_i H_T(\ve{x}_i) \leq \theta] <  \varepsilon$ if
      \begin{eqnarray}
T & \geq & \underbrace{\frac{2m\left(\popsur(H_{0}, \mathcal{S}) - C\right)}{K_{\mbox{\tiny rec}} \kappa \cdot \varepsilon^2 \underline{w}(\theta)^2 \gammawla^2}}_{\defeq \bcnn} = \frac{m \bcls}{K_{\mbox{\tiny rec}}} = \tilde{O} \left(\frac{m}{K_{\mbox{\tiny rec}} \varepsilon^2 \gammawla^2}\right),\label{boost-rate-NN}
      \end{eqnarray}
a bound substantially better and more general than the one of \cite[Theorem 4]{nbanbGN}, which was established for $\theta = 0$ (namely, our assumptions are weaker, our result cover asymmetric losses and the dependency of \eqref{boost-rate-NN} in $K_{\mbox{\tiny rec}}$ is better). It is important to remind at this point that this result is relevant to training and the margin bound holds over the training sample.

 \noindent $\triangleright$ \underline{\textit{Effect of Long and Servedio's data}}: it is not hard to see that the problem is equivalent to leveraging a constant prediction using all examples with a specific observation and the leveraging coefficient is the same as for a cdt~where the root node's support is restricted to the given observation, yielding optimal leveraging (after leveraging, each example's new prediction is Bayes optimum) and so the \cnn~leveraging learned by \topdowngen~is not affected by noise. This applies for any choice of $K\geq 1$ neighbors for \cnn.

\paragraph{Application of \topdowngen~$\#5$: labeled branching programs (\clbp)} A labeled branching program is a branching program with prediction values at each node, just like our encoding of \cdt, with the same way of classifying an observation -- sum an observation's path values from the root to a leaf. The key difference with classical branching programs is that to one leaf can correspond as many possible predictions as there are paths leading to it. See Figure \ref{f-DT-transformation} for an example. 

\noindent $\triangleright$ \underline{\textit{$u_t$ compliance and the weak learner}}: the weak learner is the same as for \cdt, \textit{except} it looks for a split over the \textit{union} of a set of leaves in the current \clbp, with the constraint that this split has to cut every leaf's domain in two (this requirement can be removed if the user is comfortable that some inner nodes in the \clbp~may have out-degree 1). After the split is found, it is carried at each node and the outgoing arcs get to two new leaves only by merging the leaves of the stumps accordingly (call this procedure the \textit{split-merge} process), as displayed in Figure \ref{f-DT-transformation}. This makes the weak learner have the same properties as for \cdt, but of course, yields larger $u_t$ compliance than for \cdt, and so bring better boosting rates as we now show.

\noindent $\triangleright$ \underline{\textit{Boosting rate}}: suppose we run \topdowngen~as for \cdt~and start to merge nodes to always ensure $u_t\geq \upbeta$ for some $\upbeta>1/T$. We get $\sum_{t=1}^T u_t \geq \sum_{t=1}^{\lfloor 1/\upbeta\rfloor} (1/t) + \upbeta (T-\lfloor 1/\upbeta\rfloor) \geq \log(1+\lfloor 1/\upbeta\rfloor) + \upbeta T -1 \geq \log(1/\upbeta) + \upbeta T - 1$. The choice $\upbeta = T^{-c}$ for a constant $c\in (0,1)$ immediately leads that $\pr_{i\sim [m]}[y^*_i H_T(\ve{x}_i) \leq \theta] <  \varepsilon$ if
      \begin{eqnarray}
T & \geq & \underbrace{\left(\frac{2\left(\popsur(H_{0}, \mathcal{S}) - C\right)}{\kappa \cdot \varepsilon^2 \underline{w}(\theta)^2 \gammawla^2}\right)^{\frac{1}{1-c}}}_{\defeq \bclbp} = \left(\bcls\right)^{\frac{1}{1-c}} = \tilde{O} \left(\frac{1}{\varepsilon^{\frac{2}{1-c}} \gammawla^{\frac{2}{1-c}}}\right),\label{boost-rate-LBP}
      \end{eqnarray}
      a bound which is exponentially better than \eqref{boost-rate-DT} for \cdt. While it does extend previous boosting rates to margins / edges \cite{ksBI,mmBU}, \eqref{boost-rate-LBP} is suboptimal compared to the $\tilde{O}(\log^2 (1/\varepsilon))$ dependence of \cite{mmBU} shown for $\theta = 0$.

      \noindent $\triangleright$ \underline{\textit{Effect of Long and Servedio's data}}: since learning a \cdt~achieves Bayes optimal prediction with a single root \cdt, the same happens for a single root \clbp, and we get that the \clbp~learned by \topdowngen~is not affected by noise.

%% file: content-arxiv/experiments.tex
\section{Toy experiments}\label{sec-toy-exp}

We have implemented \topdowngen~and performed toy experiments it in the framework of Long and Servedio \cite{lsRC}, specifically train on $\mathcal{S}_{\mbox{\tiny{noisy}}}$ and compute, as a function of $\gamma$ in \eqref{defSclean} (i) the number of iterations of \topdowngen~until convergence, (ii) the accuracy on $\mathcal{S}_{\mbox{\tiny{clean}}}$ of the final classifier and (iii) the expected posterior of the classifier's prediction compared to Bayes'. (i) is equivalent to the number of calls to $\weak$ until the weak learning assumption "breaks", \textit{i.e.} until all possible weak hypotheses fail to meet \eqref{defWLA3} for some $\gammawla = 0.001$. We test three model classes that \topdowngen~is able to boost, \cls, \cdt~and \cnn~(with $K=1$), and consider four different proper loss functions in \topdowngen, three of which are symmetric and mainstream in ML (Matusita, log, square) and a fourth one, asymmetric, that we have engineered for the occasion. All details about those losses are in \supplement, Table \ref{fig:all-cbr}.

We have crammed the results in Table \ref{tab:all-results} to get the overall picture. Results are otherwise presented in split format in \supplement, Tables \ref{tab:iter-not-WLA}, \ref{tab:accuracy}, \ref{tab:postest}. As predicted by theory, there is a stark contrast between \cls~on one hand and \cdt~and \cnn~on the other hand: while \cdt~and \cnn~consistently achieve Bayes prediction~and thus are not affected by noise, \cls~learned always have very substantial degradation in their estimated posterior below a threshold $\gamma$, which translates to classifiers as accurate as the unbiased coin on $\mathcal{S}_{\mbox{\tiny{clean}}}$. Remarkably, the number of iterations until the weak learner is "exhausted" of options meeting \eqref{defWLA3} for $\gammawla = 0.001$ is very low: it takes just a few iterations for \topdowngen~to be stuck with a bad \cls, confirming a remark of Schapire \cite{sEA} (See \supplement, Section \ref{sec-wtps}). Notably, the use of an asymmetric proper loss does not break any of these patterns. There are interesting differences appearing among choices for the noise parameter $\etanoise$: increasing it tends to increase the threshold for phase transition in accuracy with \cls~and tends to reduce the number of calls to \weak~in this region for the square loss and our asymmetric loss. This last observation makes sense because increasing noise reduces the absolute edge of the weak classifiers. We also note that the result on \cnn~for the number of iterations until \weak~is "exhausted" displays that the dependency in $m$ in \eqref{boost-rate-NN}'s rate is pessimistic as it should depend on the number of distinct observations (=3 in \cite{lsRC}'s domain) rather than $m$.

  \newcommand{\picwidth}{0.104}
  \newcommand{\basicnegativespace}{\hspace{-0.2cm}}
  \newcommand{\basicnegativespacebefore}{\hspace{-0.6cm}}
\begin{table*}
  \centering
  \begin{tabular}{rccc||ccc||ccc}
    & \multicolumn{3}{c||}{Accuracy on $\mathcal{S}_{\mbox{\tiny{clean}}}$} & \multicolumn{3}{c||}{Expected posterior}& \multicolumn{3}{c}{$\#$ Calls to \weak}\\
    & \cls \basicnegativespacebefore & \basicnegativespacebefore \cdt \basicnegativespace & \basicnegativespacebefore 1-\cnn \basicnegativespace & \basicnegativespace \cls \basicnegativespacebefore & \basicnegativespacebefore \cdt\basicnegativespace & \basicnegativespacebefore 1-\cnn \basicnegativespace & \basicnegativespace \cls \basicnegativespacebefore & \basicnegativespacebefore \cdt \basicnegativespacebefore & \basicnegativespacebefore 1-\cnn \basicnegativespace\\
    \rotatebox{90}{{\tiny \texttt{Matusita loss}}} & \basicnegativespace \includegraphics[trim=50bp 0bp 10bp 10bp,clip,width=\picwidth\textwidth]{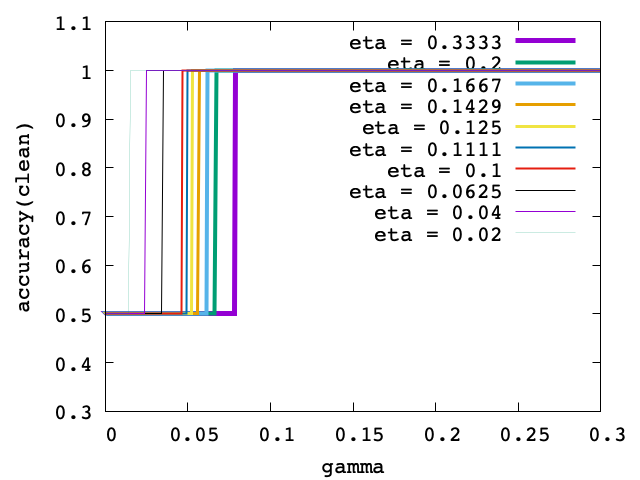} \basicnegativespacebefore & \basicnegativespacebefore \includegraphics[trim=50bp 0bp 10bp 10bp,clip,width=\picwidth\textwidth]{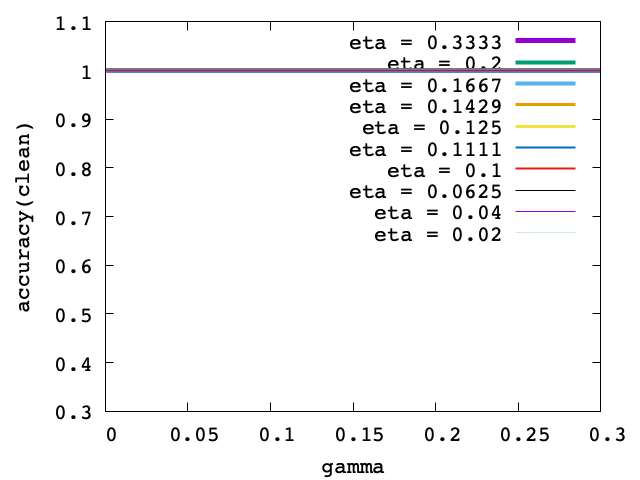} \basicnegativespacebefore & \basicnegativespacebefore \includegraphics[trim=50bp 0bp 10bp 10bp,clip,width=\picwidth\textwidth]{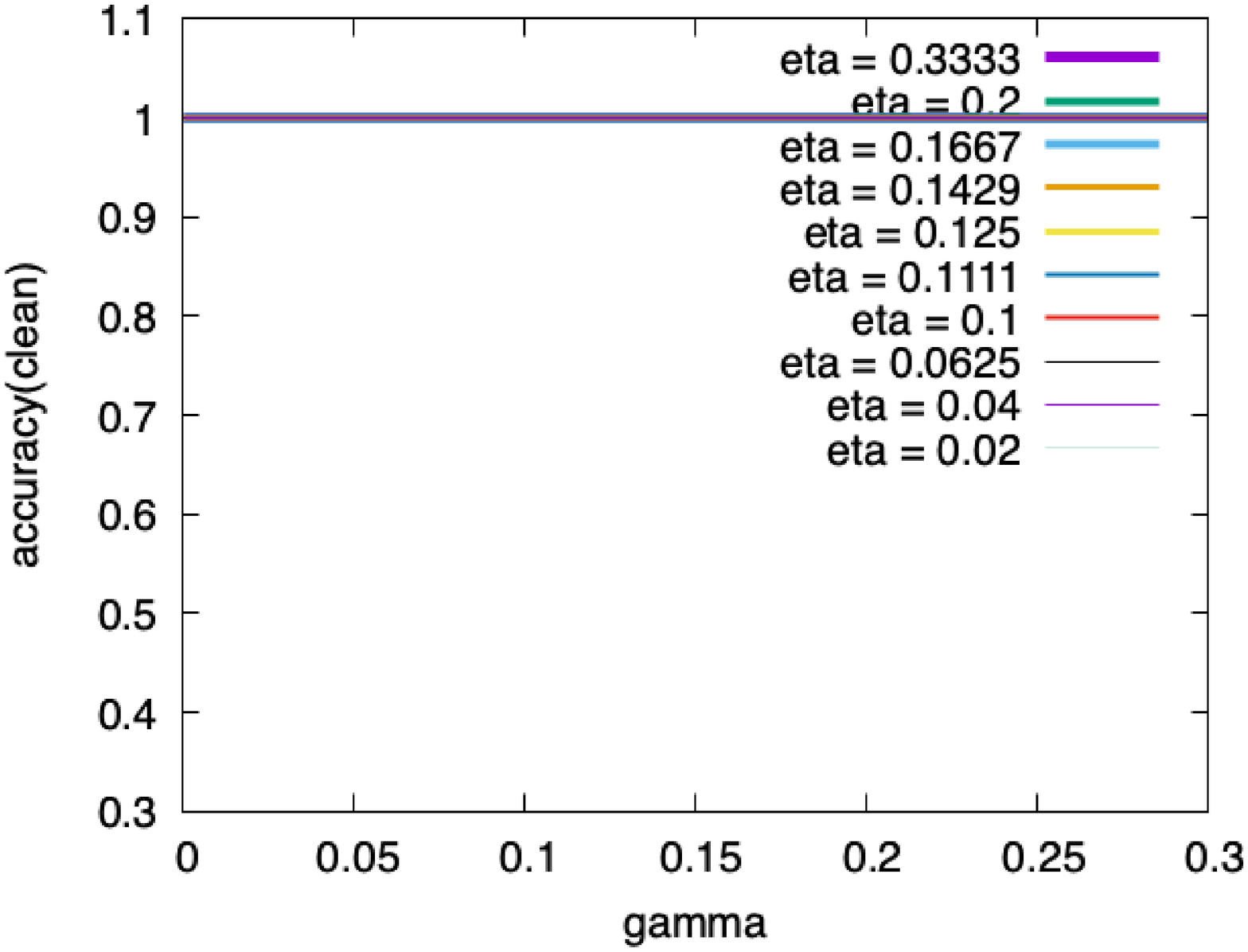} \basicnegativespacebefore & \basicnegativespace \includegraphics[trim=50bp 0bp 10bp 10bp,clip,width=\picwidth\textwidth]{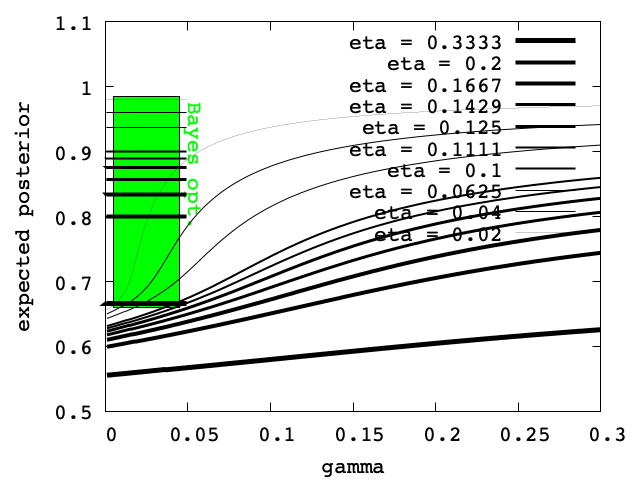} \basicnegativespacebefore & \basicnegativespacebefore \includegraphics[trim=50bp 0bp 10bp 10bp,clip,width=\picwidth\textwidth]{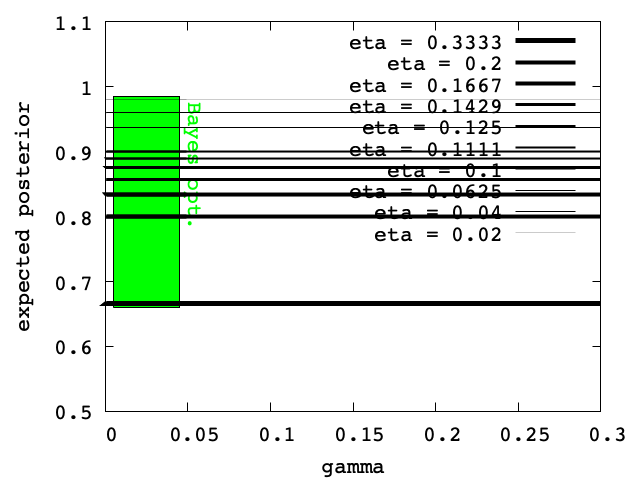} \basicnegativespace  & \basicnegativespacebefore \includegraphics[trim=50bp 0bp 10bp 10bp,clip,width=\picwidth\textwidth]{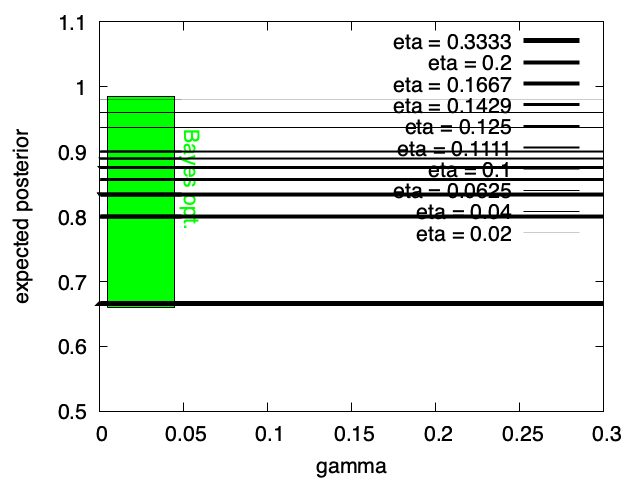} \basicnegativespace  & \basicnegativespace \includegraphics[trim=50bp 0bp 10bp 10bp,clip,width=\picwidth\textwidth]{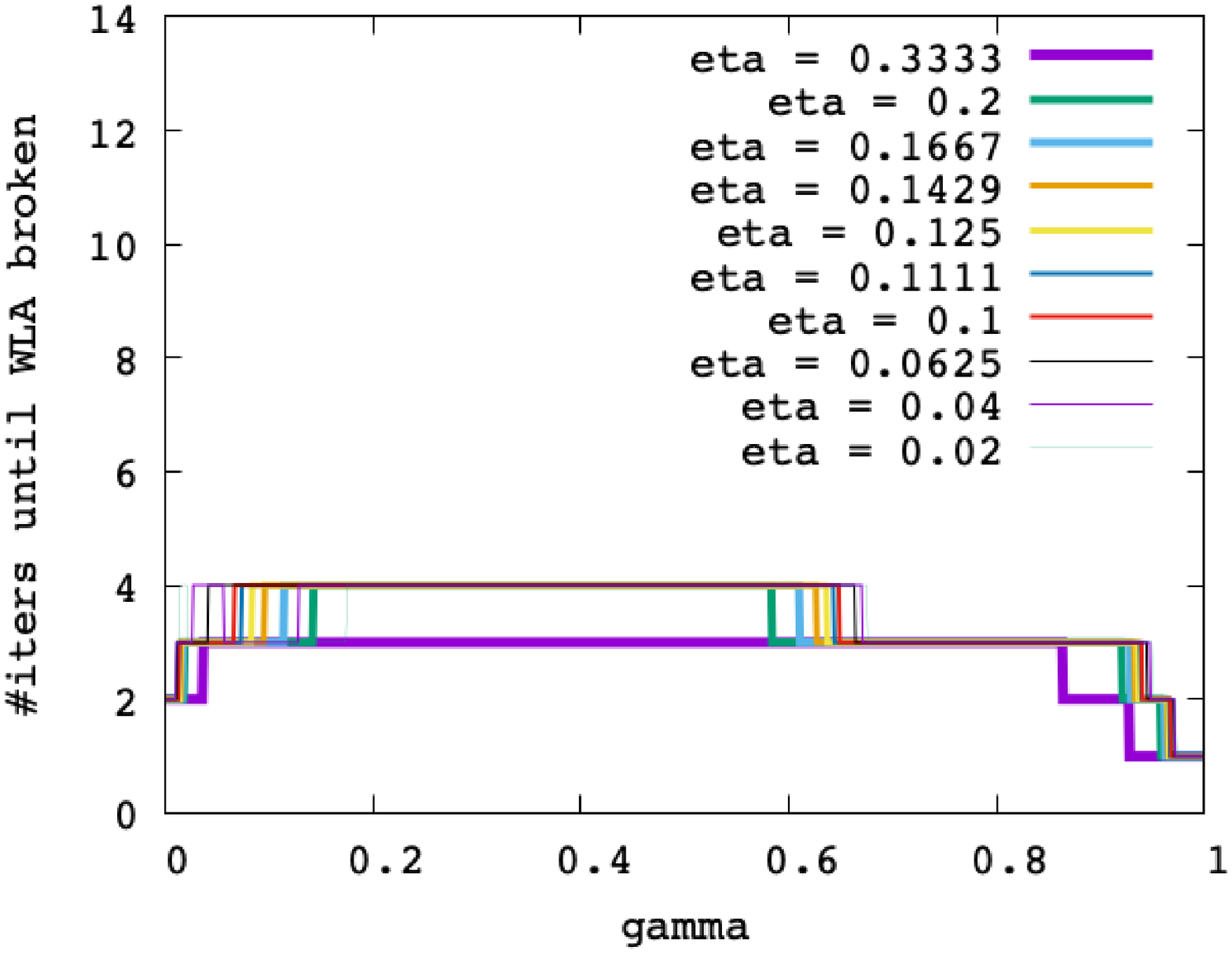} \basicnegativespacebefore & \basicnegativespacebefore \includegraphics[trim=50bp 0bp 10bp 10bp,clip,width=\picwidth\textwidth]{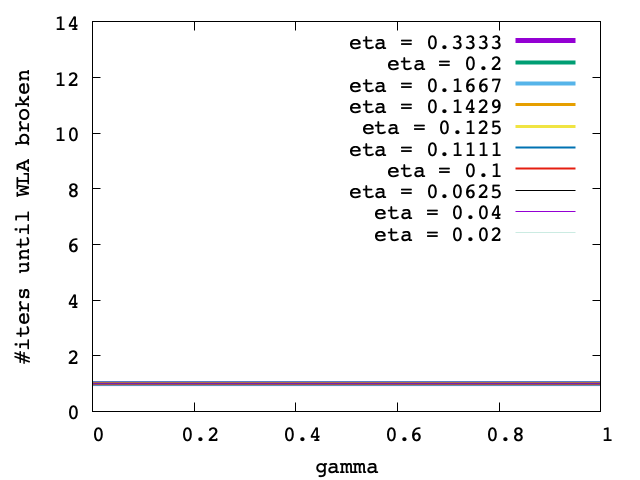} \basicnegativespace  & \basicnegativespacebefore \includegraphics[trim=50bp 0bp 10bp 10bp,clip,width=\picwidth\textwidth]{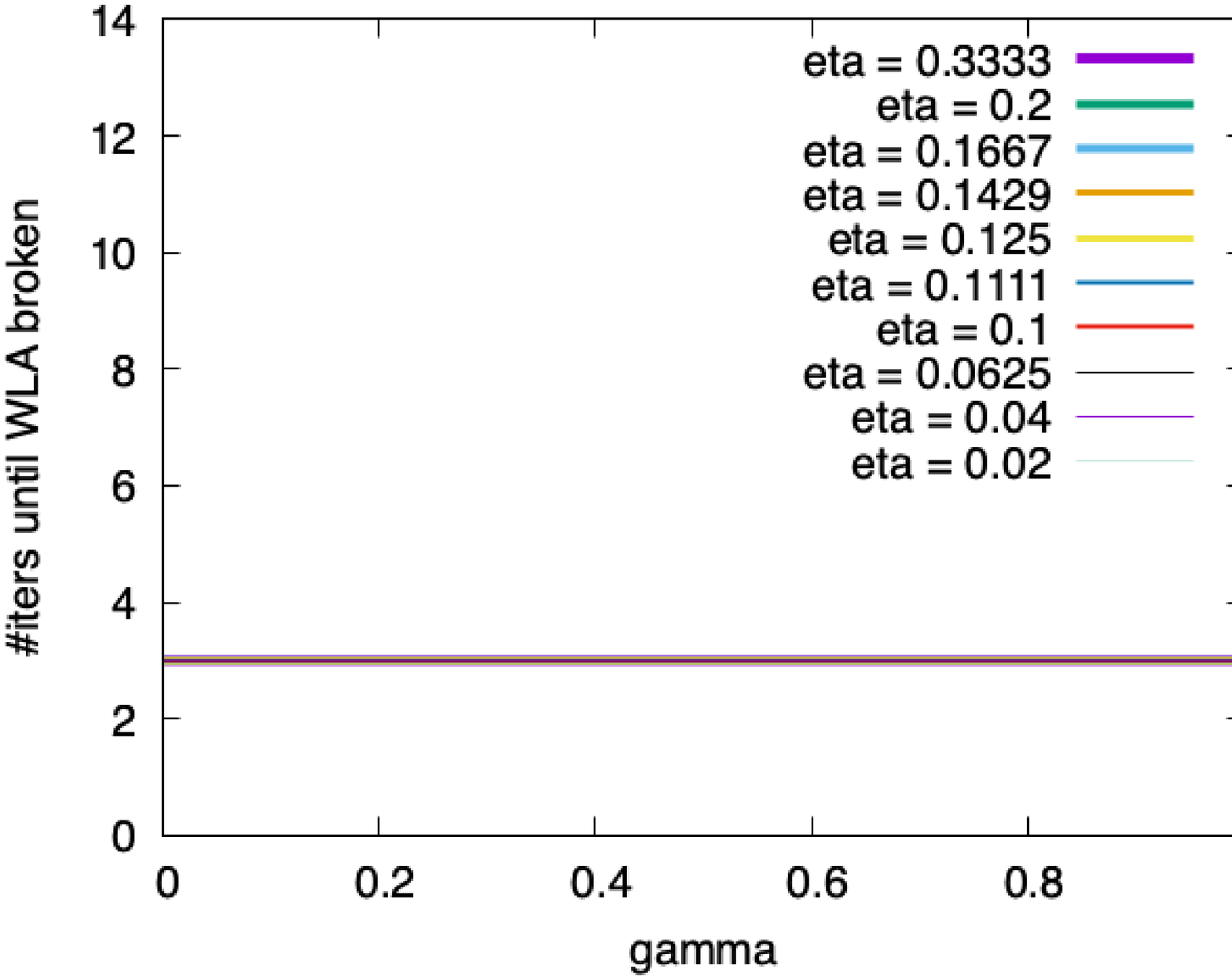} \basicnegativespace \\
    \rotatebox{90}{{\tiny \texttt{Log loss}}} & \basicnegativespace \includegraphics[trim=50bp 0bp 10bp 10bp,clip,width=\picwidth\textwidth]{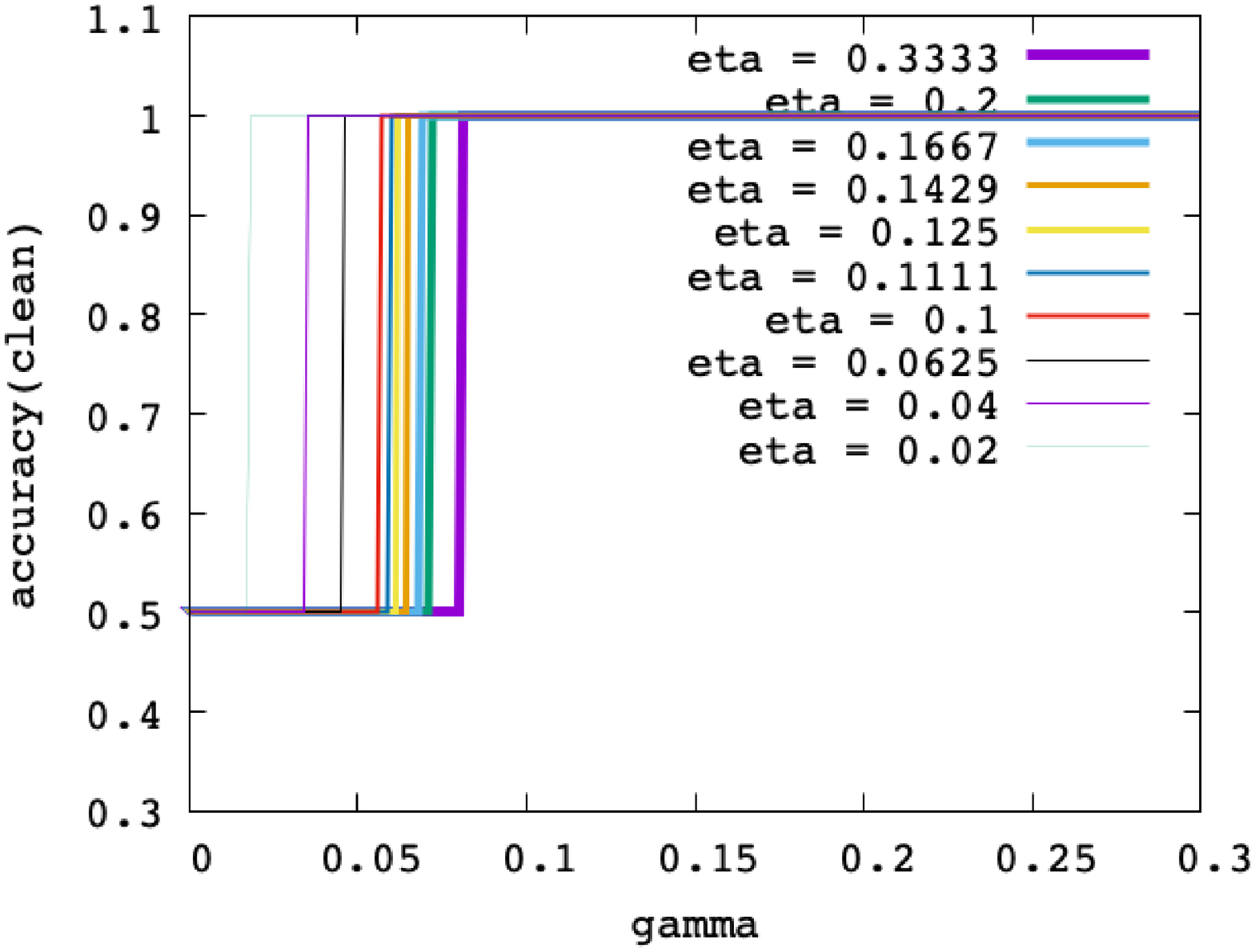} \basicnegativespacebefore & \basicnegativespacebefore \includegraphics[trim=50bp 0bp 10bp 10bp,clip,width=\picwidth\textwidth]{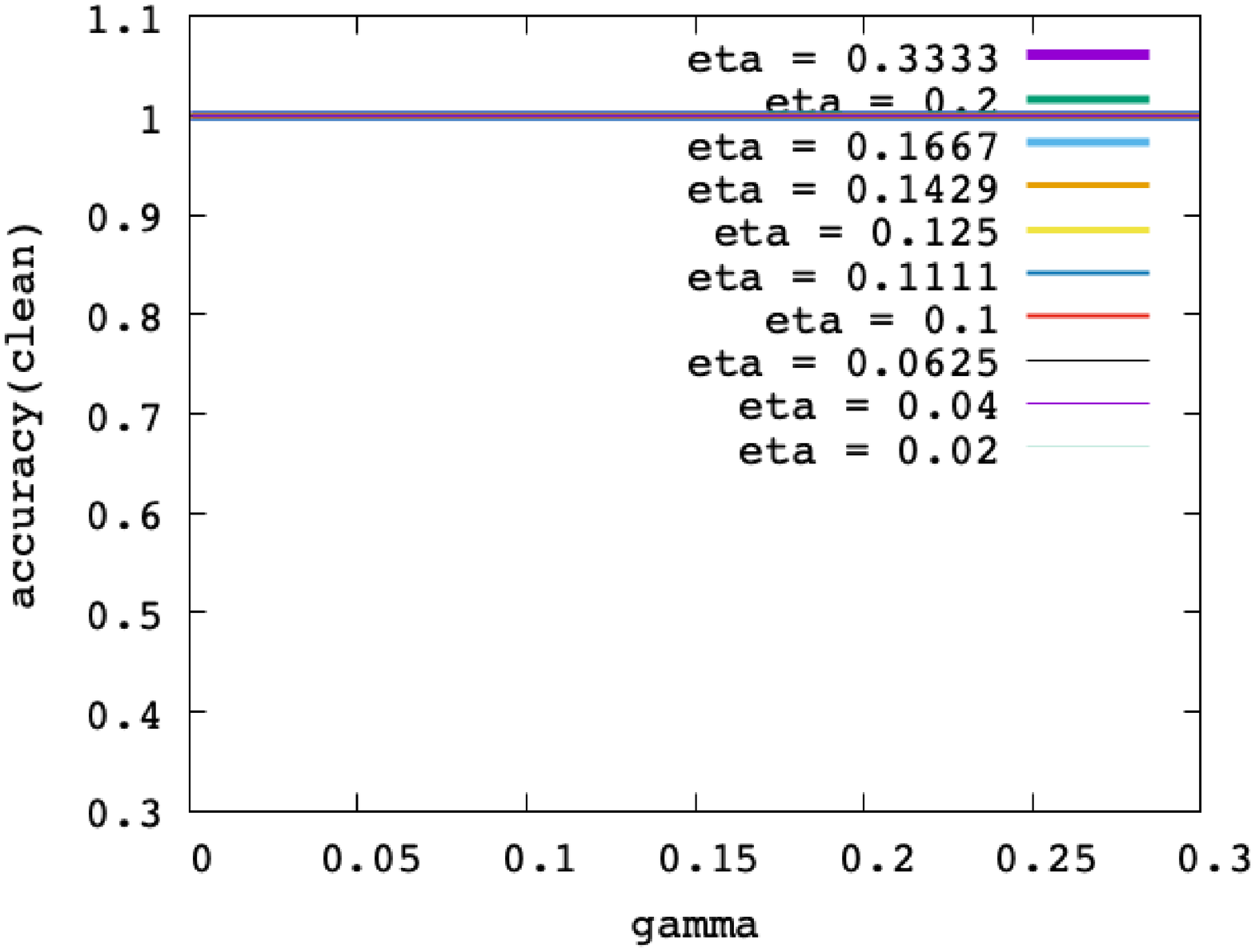} \basicnegativespacebefore & \basicnegativespacebefore \includegraphics[trim=50bp 0bp 10bp 10bp,clip,width=\picwidth\textwidth]{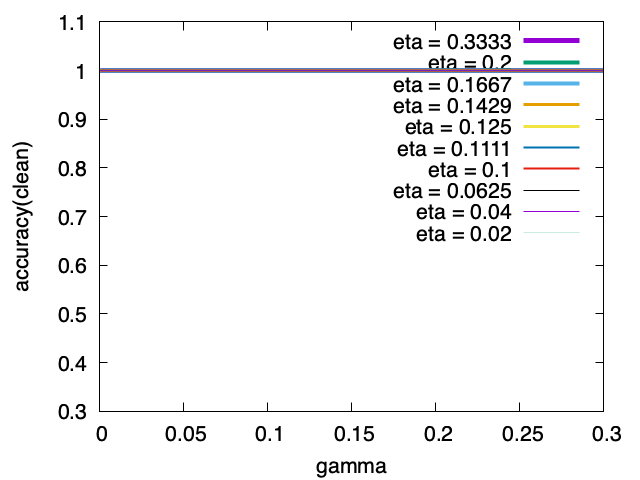} \basicnegativespacebefore & \basicnegativespace \includegraphics[trim=50bp 0bp 10bp 10bp,clip,width=\picwidth\textwidth]{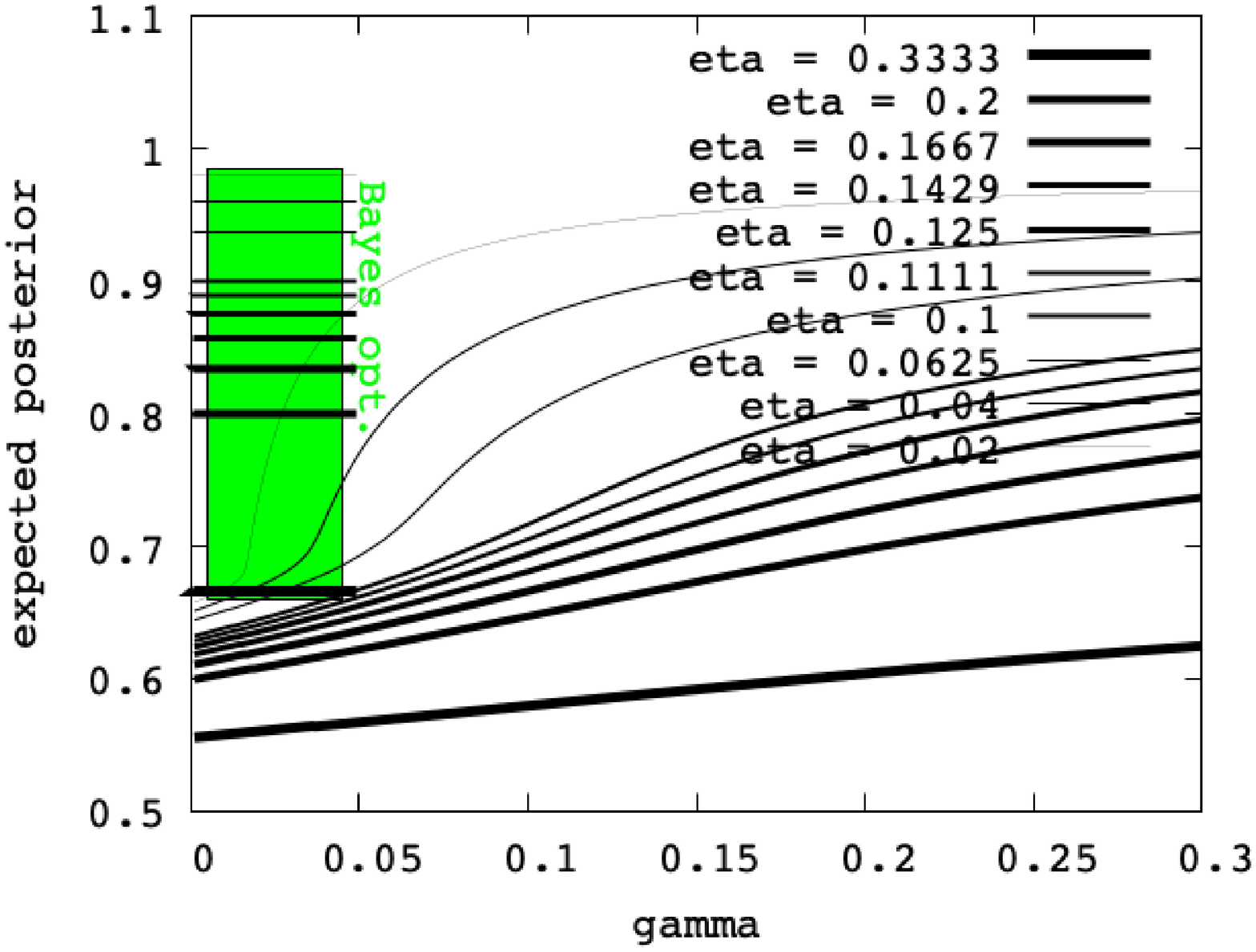} \basicnegativespacebefore & \basicnegativespacebefore \includegraphics[trim=50bp 0bp 10bp 10bp,clip,width=\picwidth\textwidth]{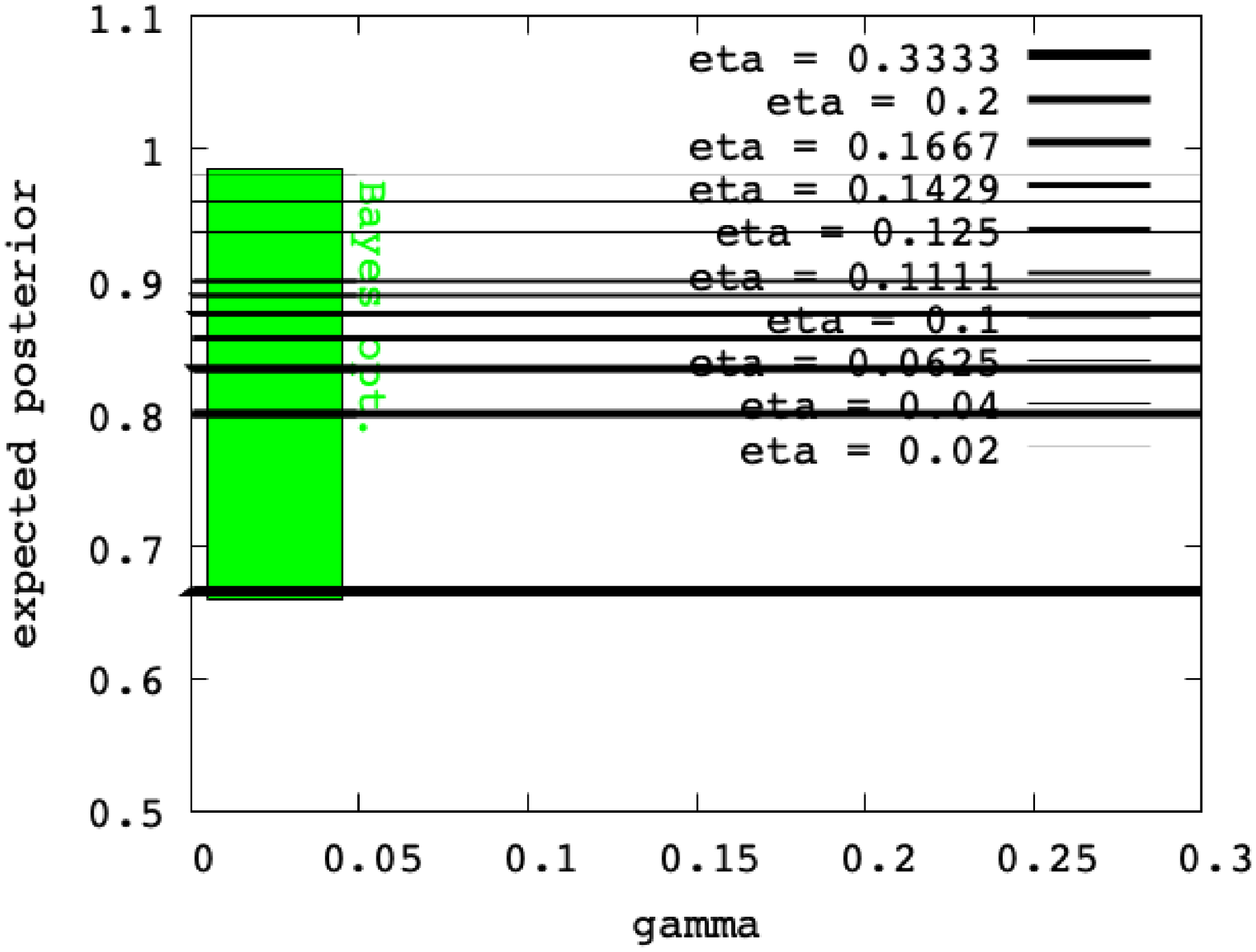} \basicnegativespace  & \basicnegativespacebefore \includegraphics[trim=50bp 0bp 10bp 10bp,clip,width=\picwidth\textwidth]{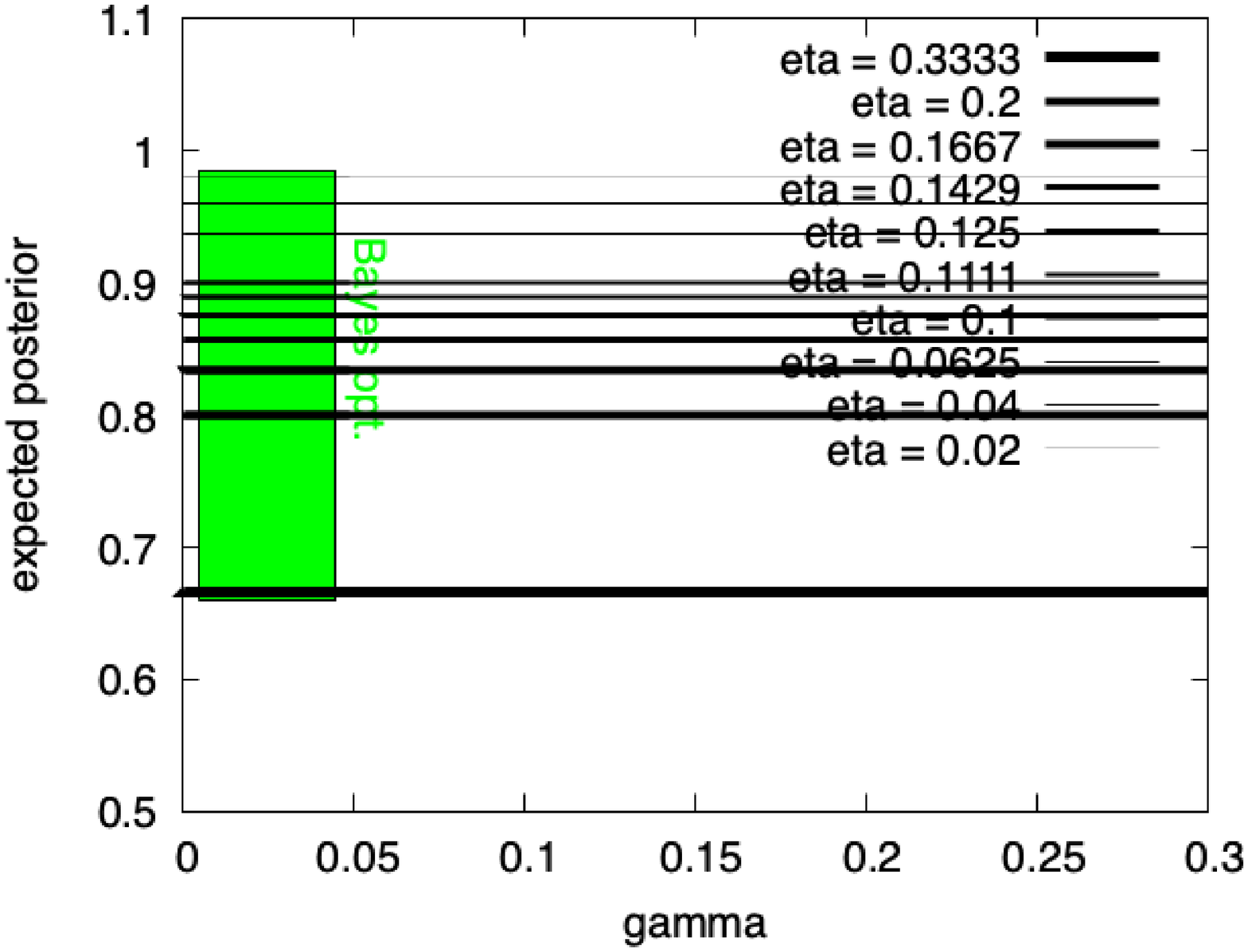} \basicnegativespace  & \basicnegativespace \includegraphics[trim=50bp 0bp 10bp 10bp,clip,width=\picwidth\textwidth]{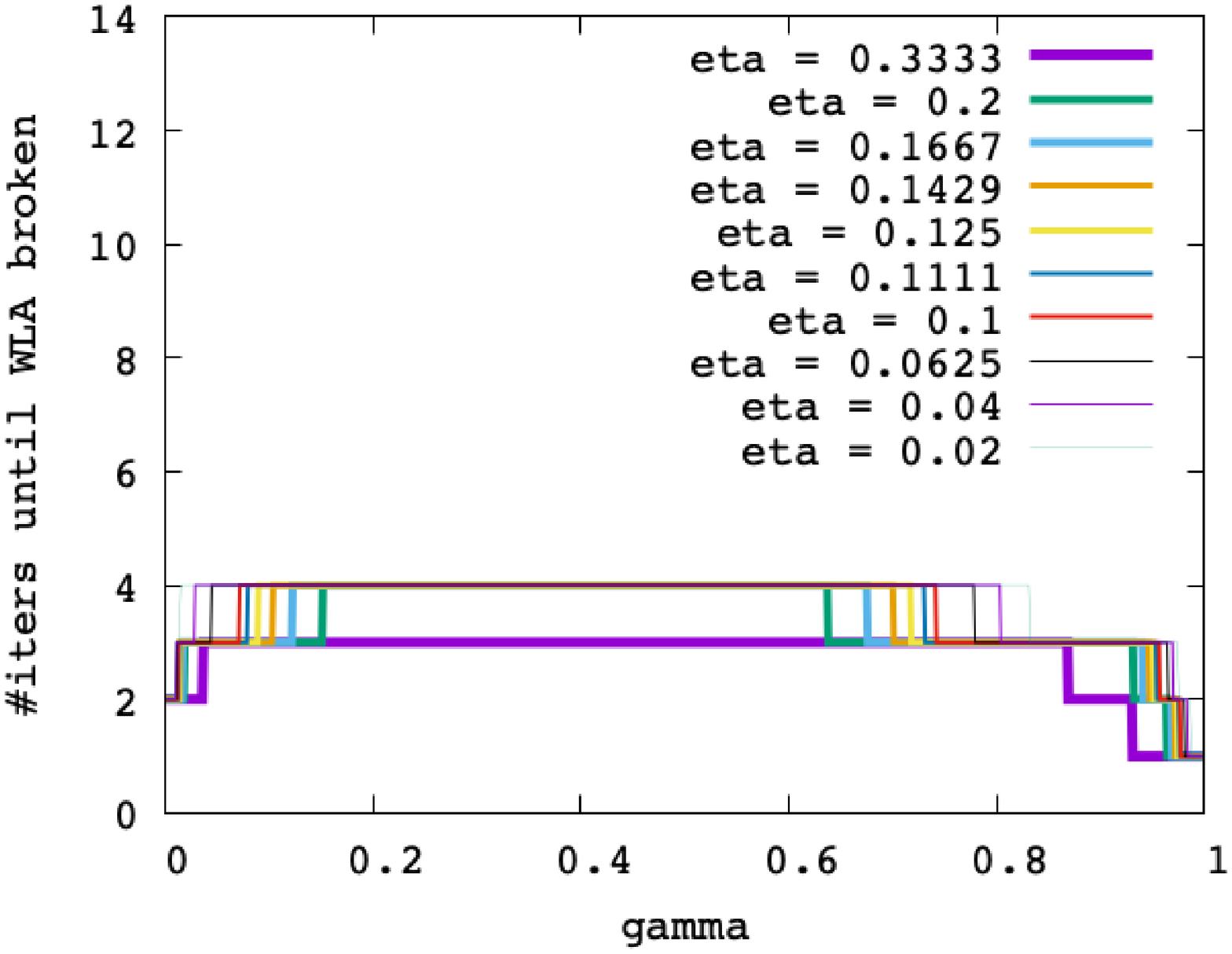} \basicnegativespacebefore & \basicnegativespacebefore \includegraphics[trim=50bp 0bp 10bp 10bp,clip,width=\picwidth\textwidth]{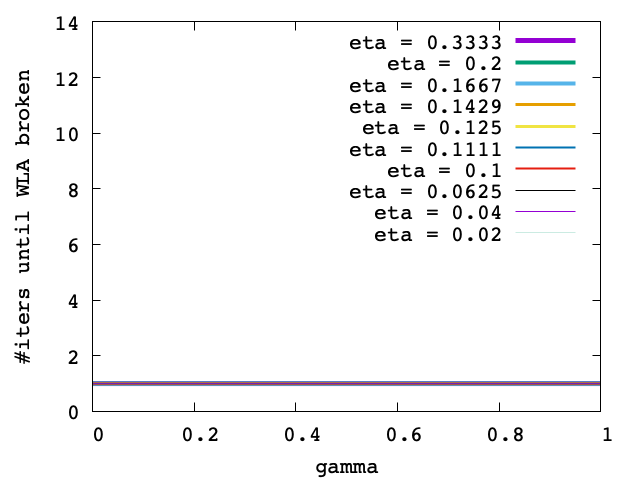} \basicnegativespace & \basicnegativespacebefore \includegraphics[trim=50bp 0bp 10bp 10bp,clip,width=\picwidth\textwidth]{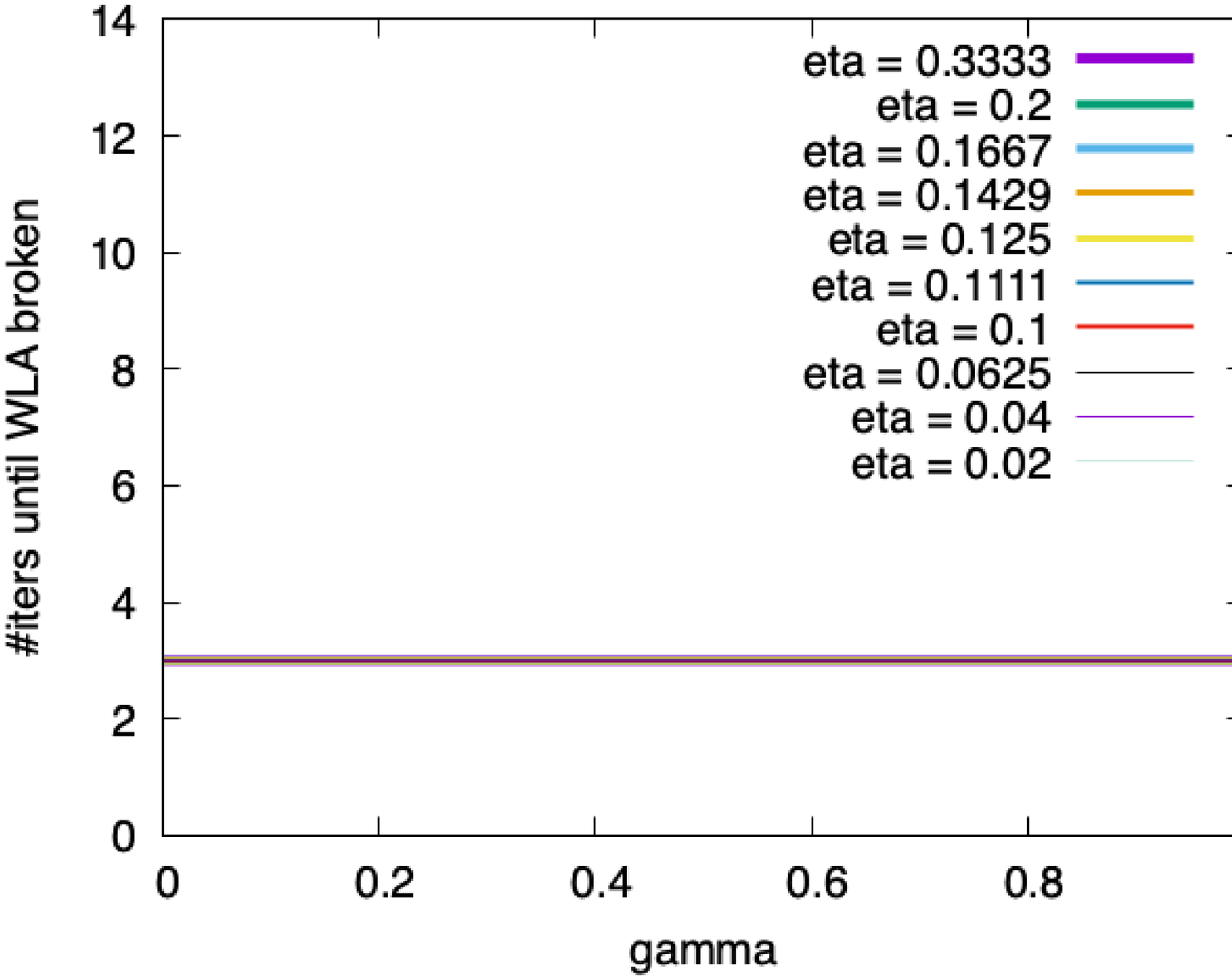} \basicnegativespace \\
    \rotatebox{90}{{\tiny \texttt{Square loss}}} & \basicnegativespace \includegraphics[trim=50bp 0bp 10bp 10bp,clip,width=\picwidth\textwidth]{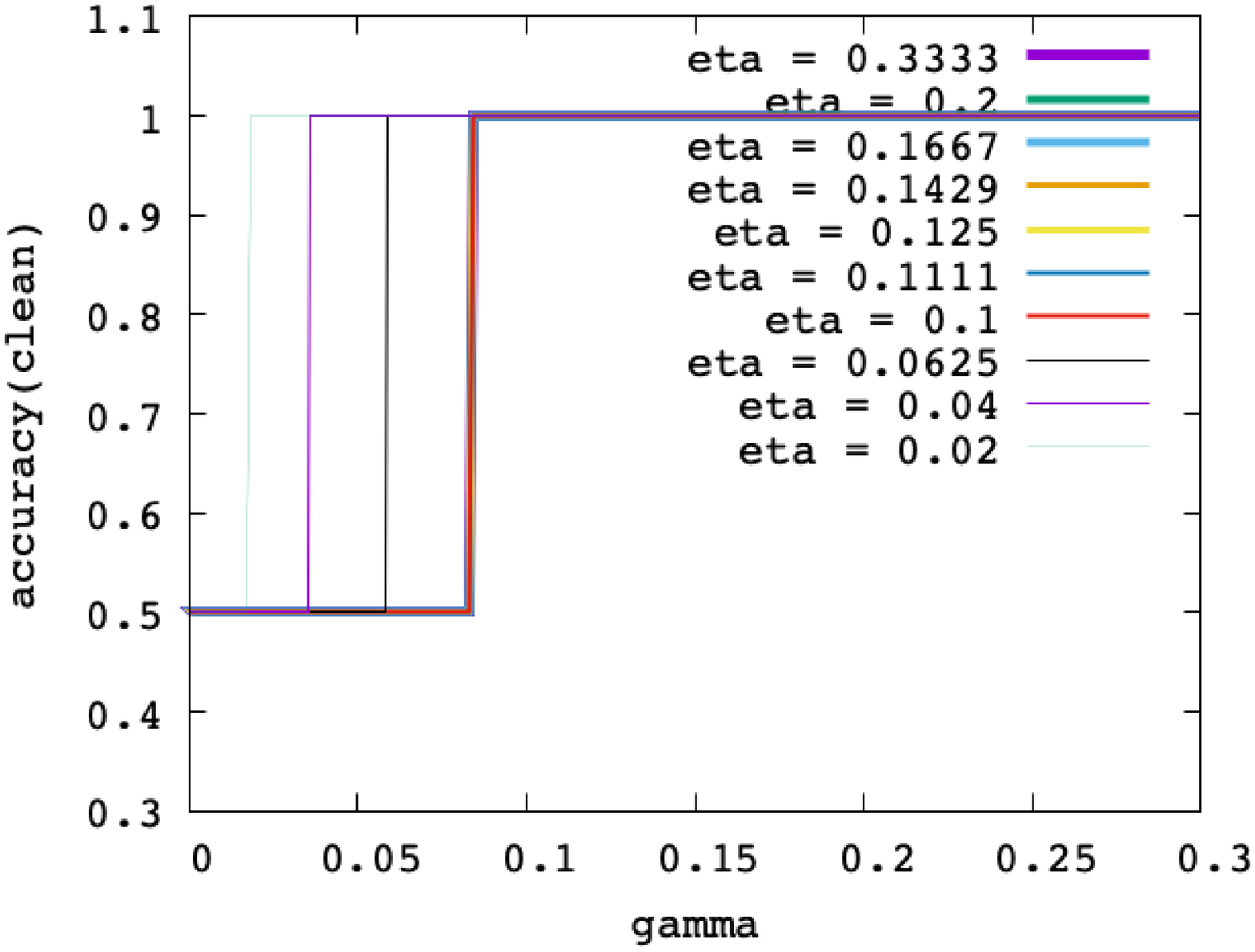} \basicnegativespacebefore & \basicnegativespacebefore \includegraphics[trim=50bp 0bp 10bp 10bp,clip,width=\picwidth\textwidth]{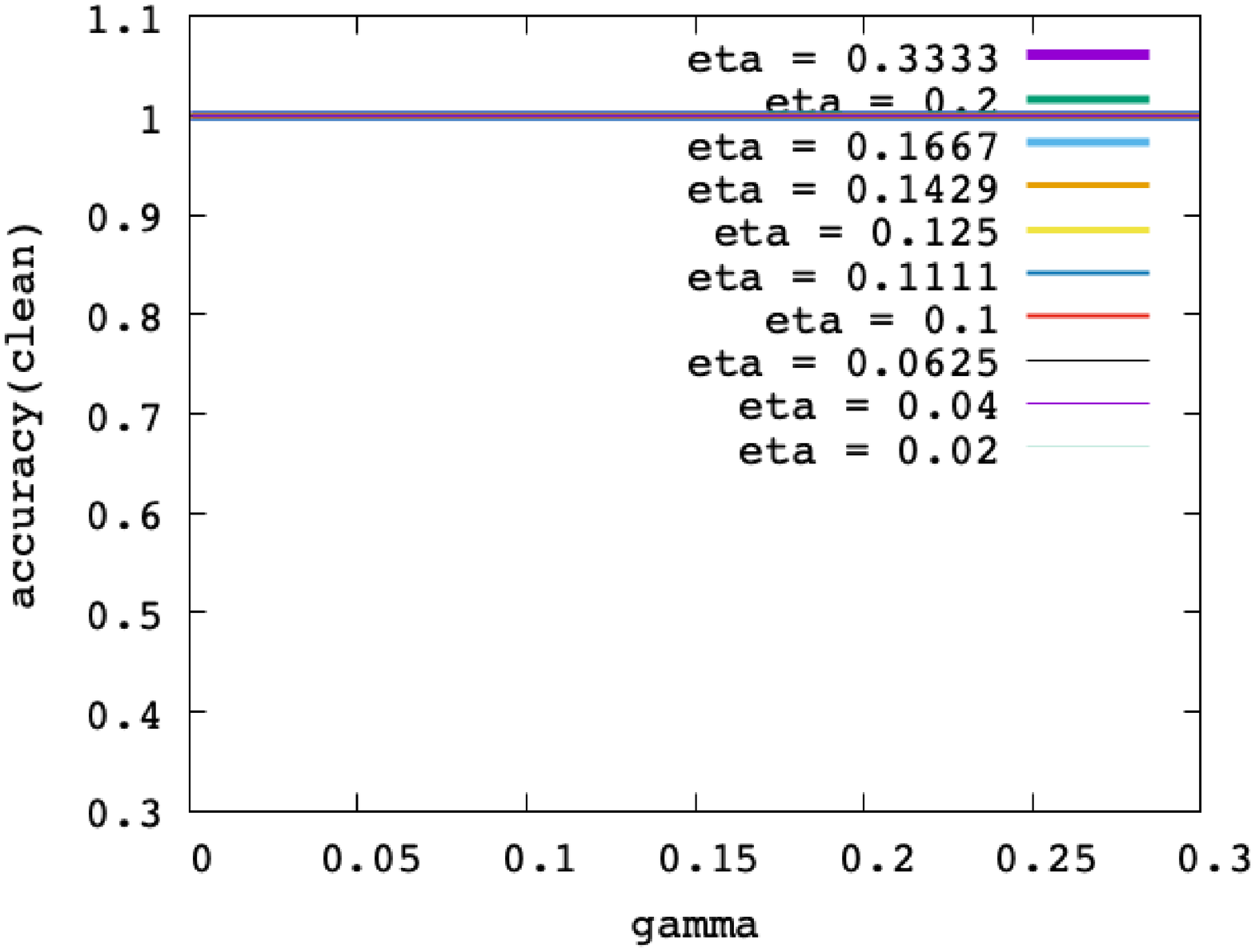} \basicnegativespacebefore& \basicnegativespacebefore \includegraphics[trim=50bp 0bp 10bp 10bp,clip,width=\picwidth\textwidth]{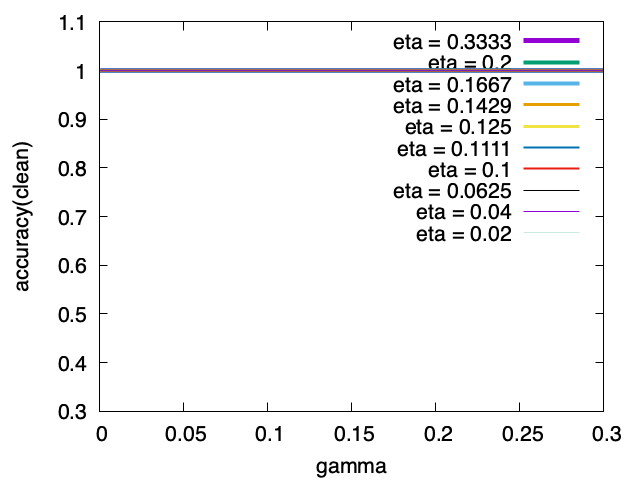} \basicnegativespacebefore & \basicnegativespace \includegraphics[trim=50bp 0bp 10bp 10bp,clip,width=\picwidth\textwidth]{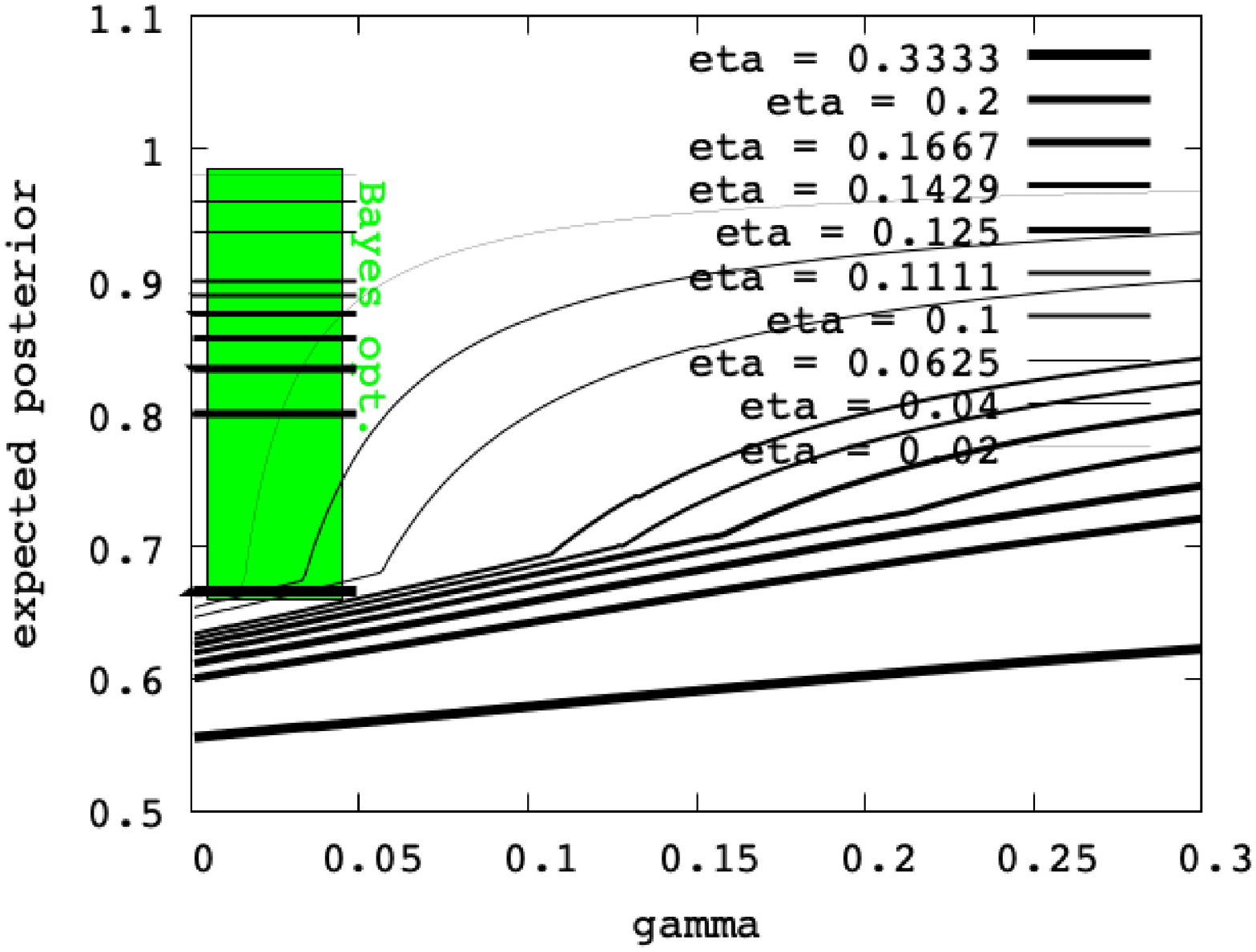} \basicnegativespacebefore & \basicnegativespacebefore \includegraphics[trim=50bp 0bp 10bp 10bp,clip,width=\picwidth\textwidth]{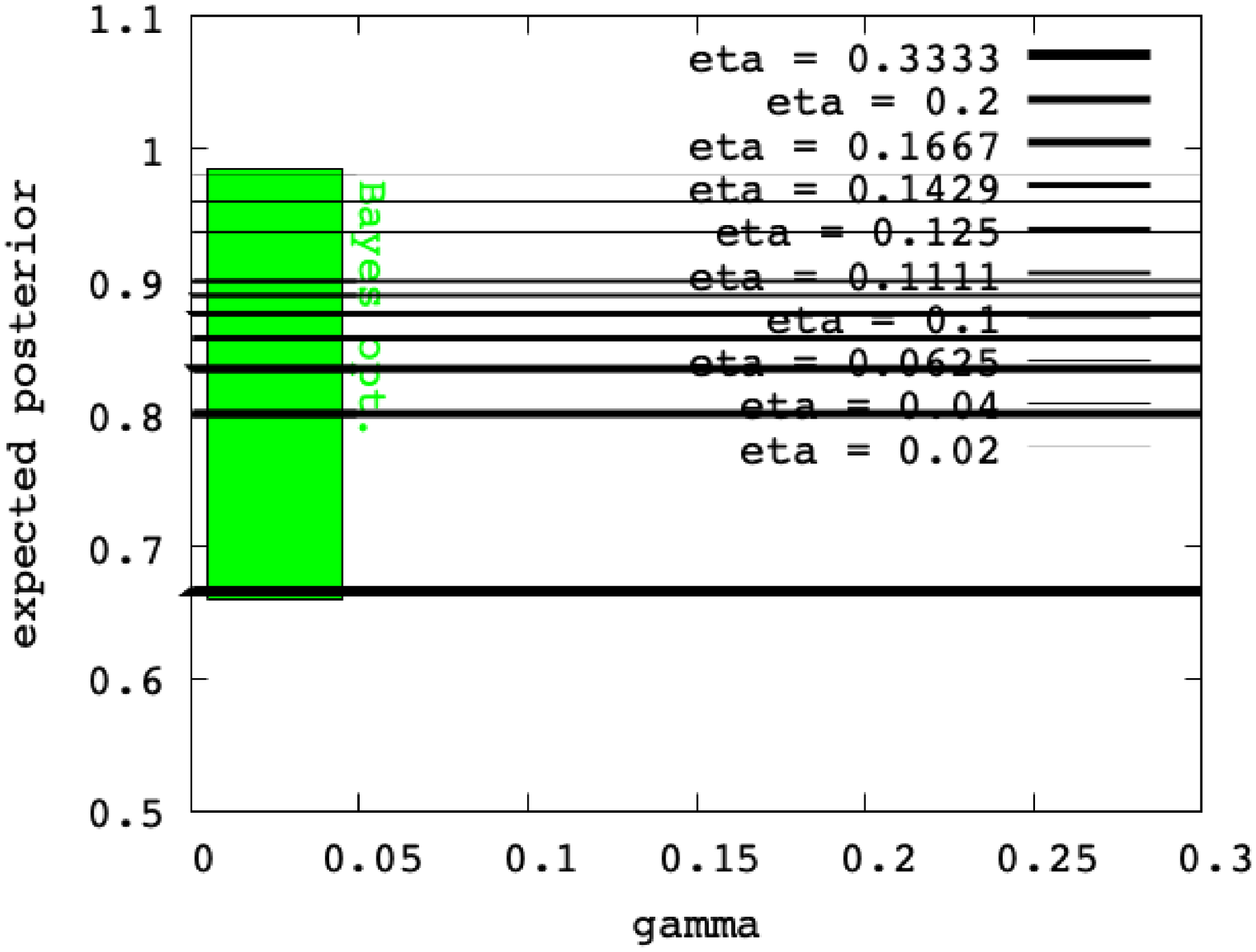} \basicnegativespace  & \basicnegativespacebefore \includegraphics[trim=50bp 0bp 10bp 10bp,clip,width=\picwidth\textwidth]{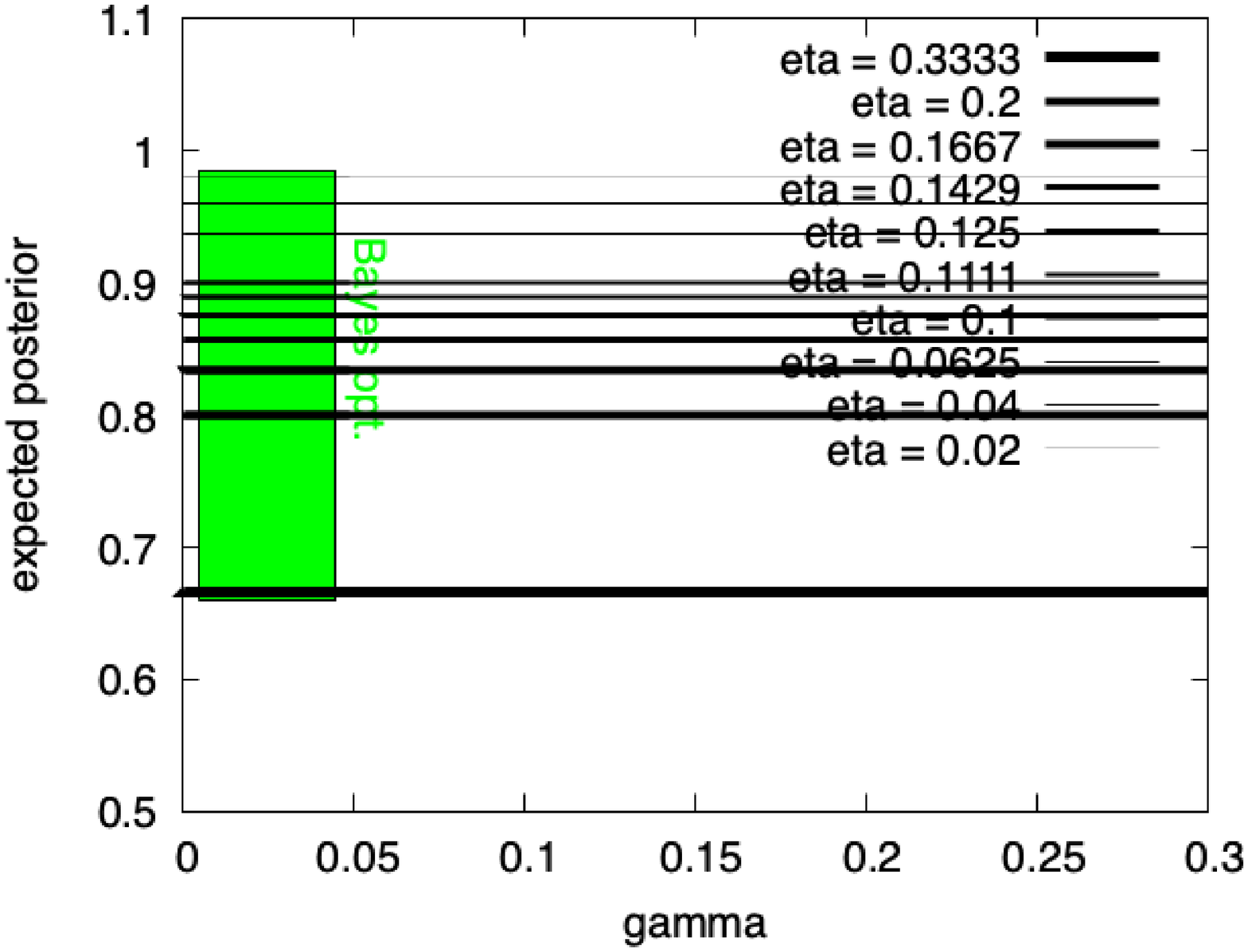} \basicnegativespace  & \basicnegativespace \includegraphics[trim=50bp 0bp 10bp 10bp,clip,width=\picwidth\textwidth]{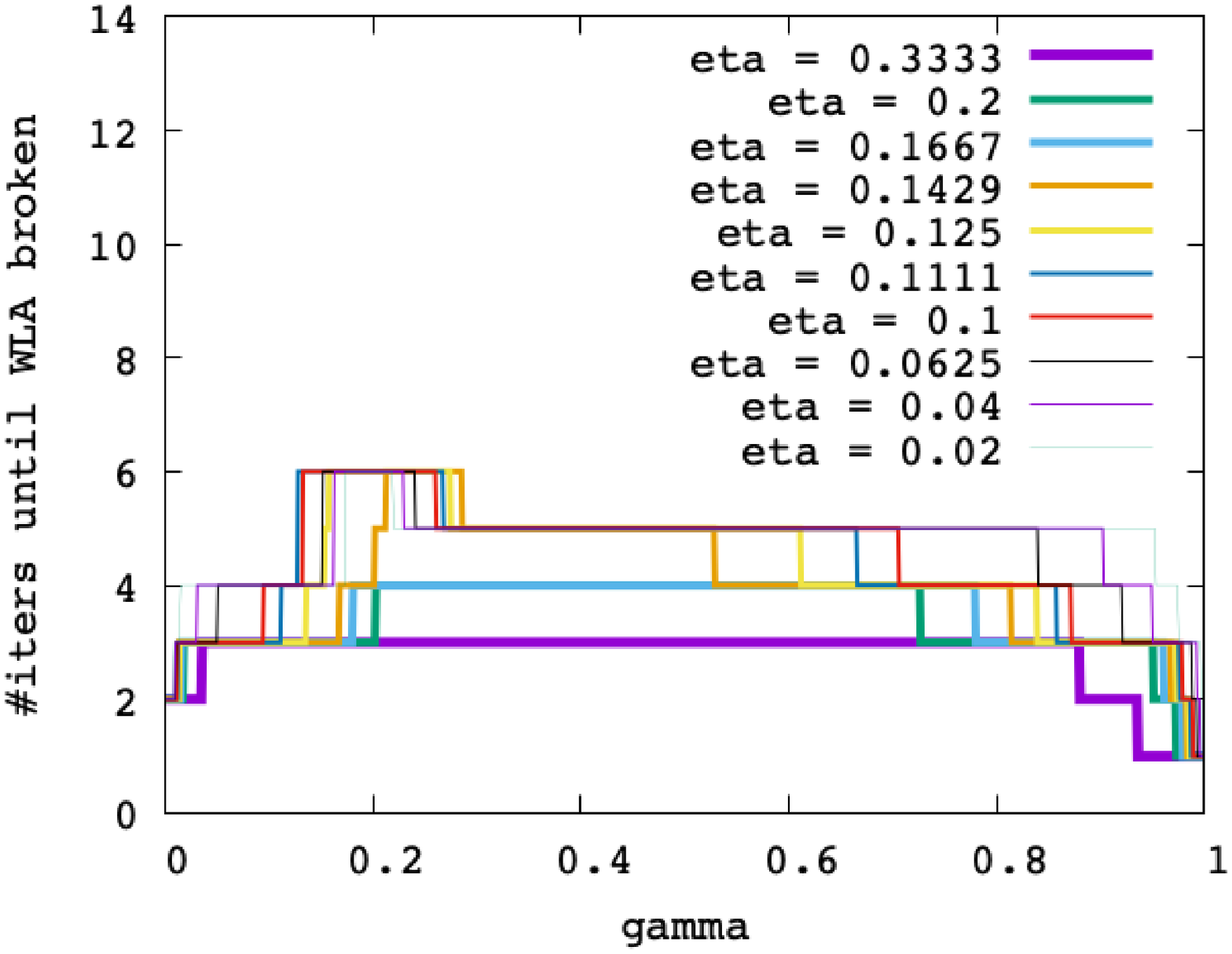} \basicnegativespacebefore & \basicnegativespacebefore \includegraphics[trim=50bp 0bp 10bp 10bp,clip,width=\picwidth\textwidth]{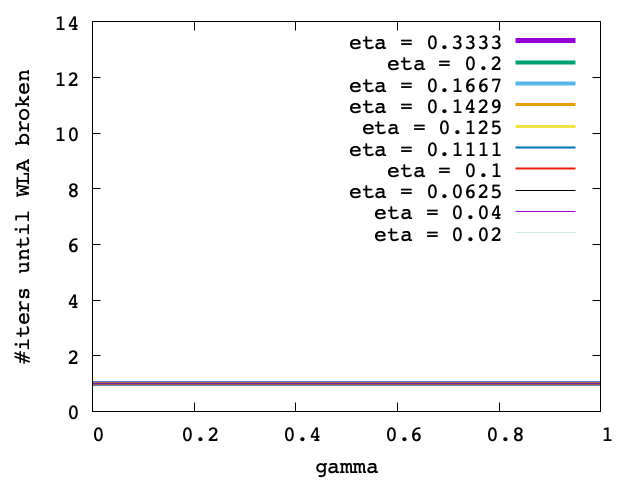} \basicnegativespace  & \basicnegativespacebefore \includegraphics[trim=50bp 0bp 10bp 10bp,clip,width=\picwidth\textwidth]{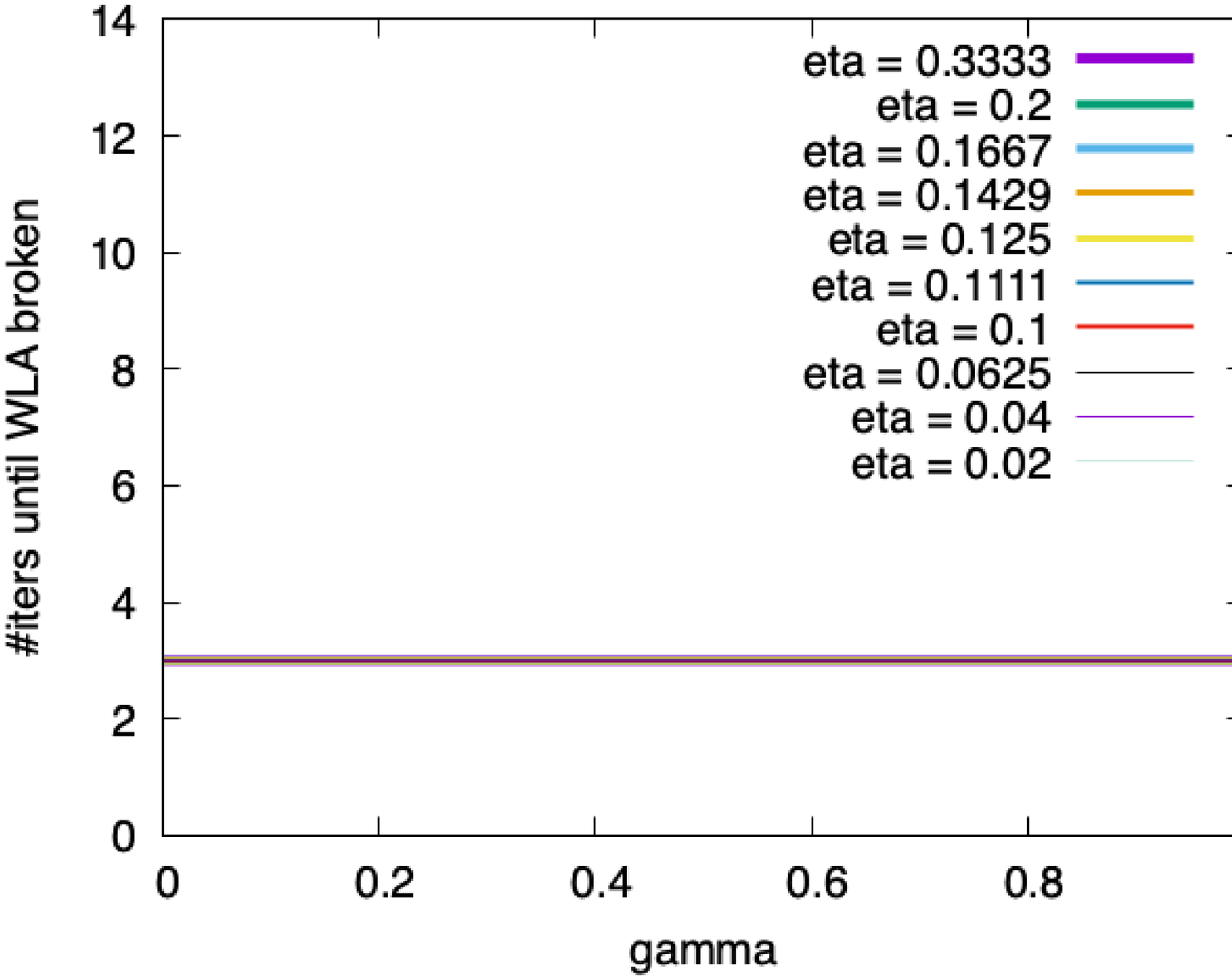} \basicnegativespace \\
    \rotatebox{90}{{\tiny \texttt{Asymmetric loss 1}}} & \basicnegativespace \includegraphics[trim=50bp 0bp 10bp 10bp,clip,width=\picwidth\textwidth]{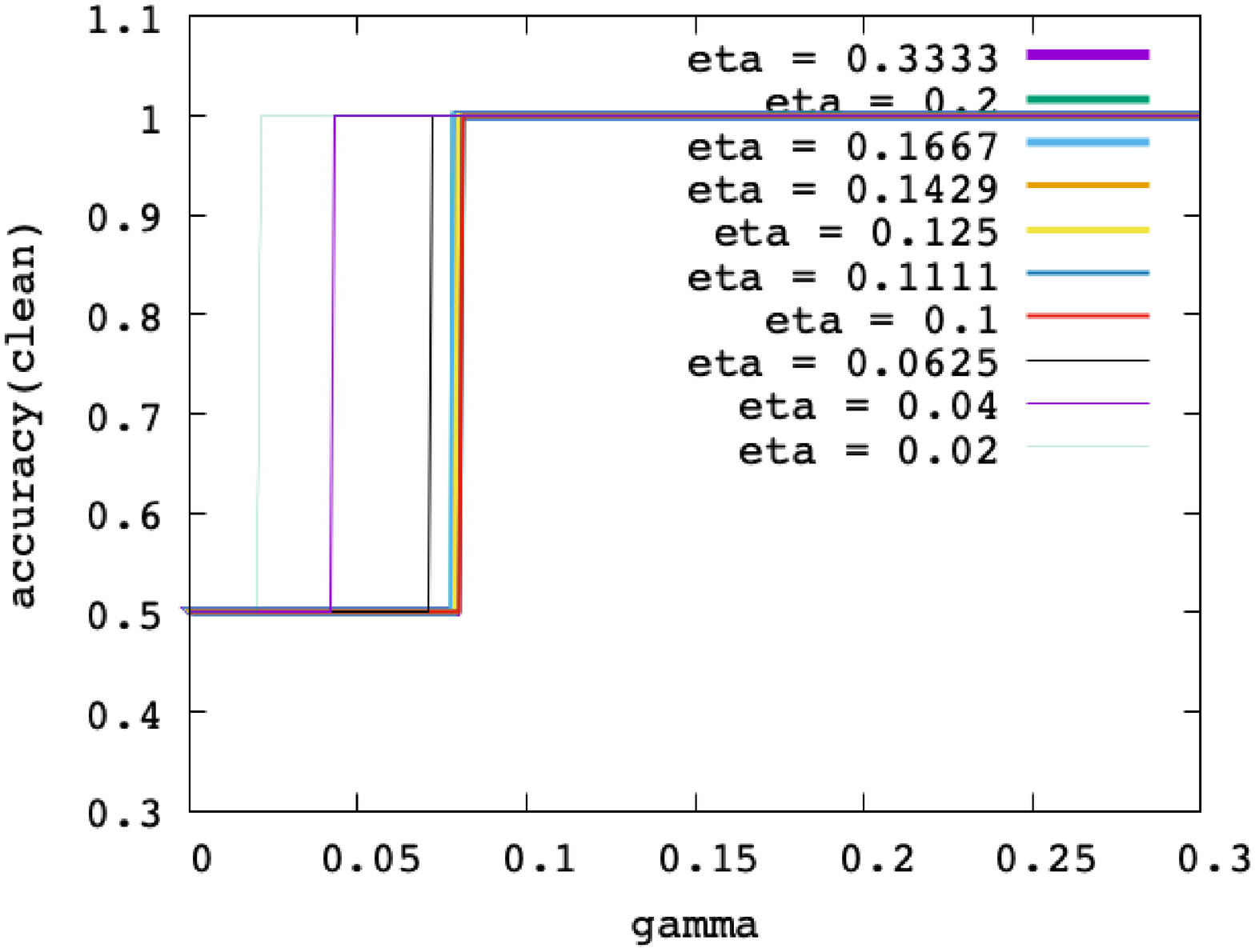} \basicnegativespacebefore & \basicnegativespacebefore \includegraphics[trim=50bp 0bp 10bp 10bp,clip,width=\picwidth\textwidth]{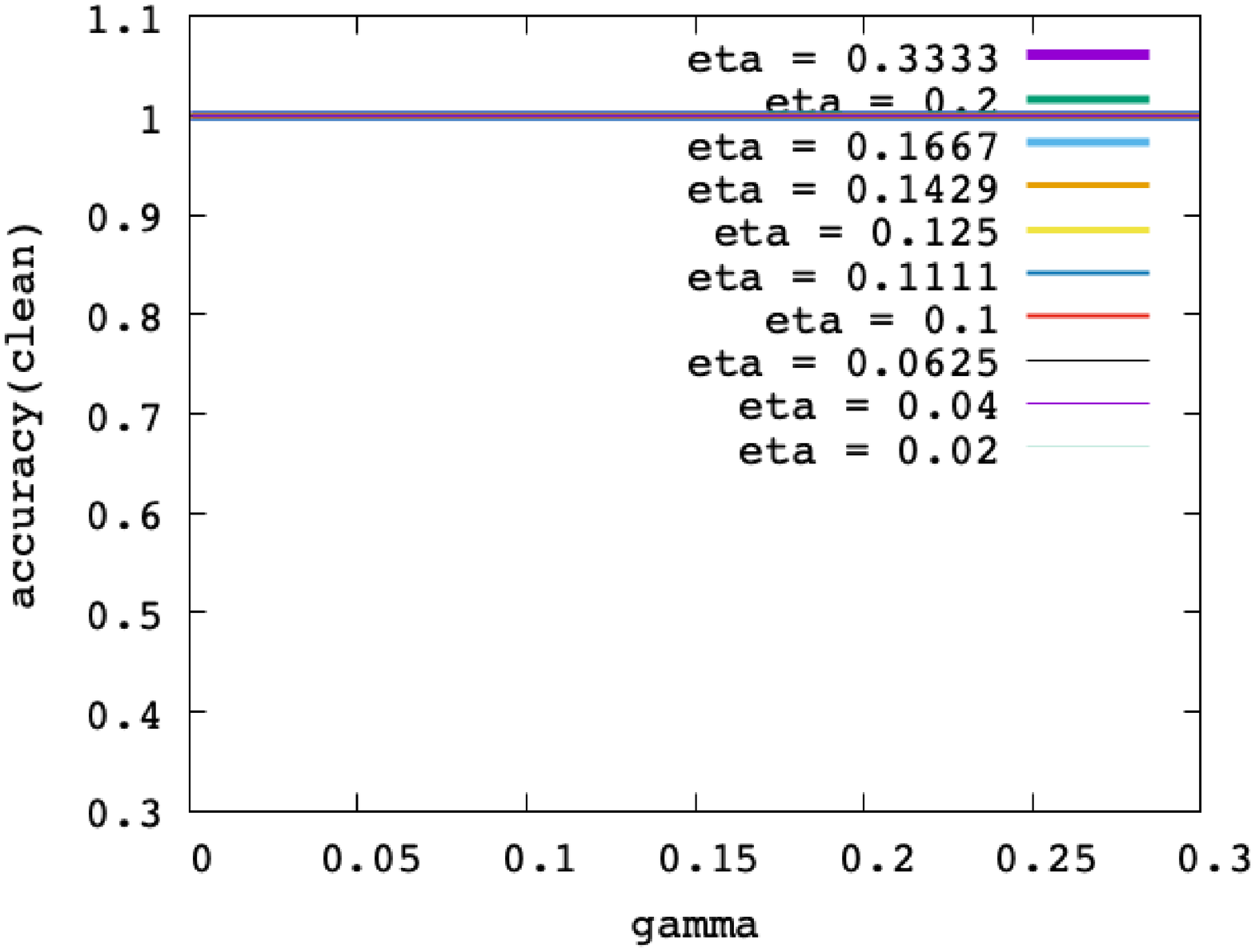} \basicnegativespacebefore & \basicnegativespacebefore \includegraphics[trim=50bp 0bp 10bp 10bp,clip,width=\picwidth\textwidth]{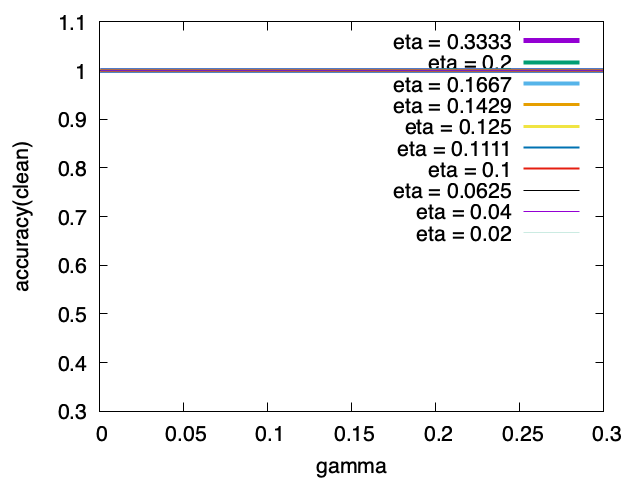} \basicnegativespacebefore & \basicnegativespace \includegraphics[trim=50bp 0bp 10bp 10bp,clip,width=\picwidth\textwidth]{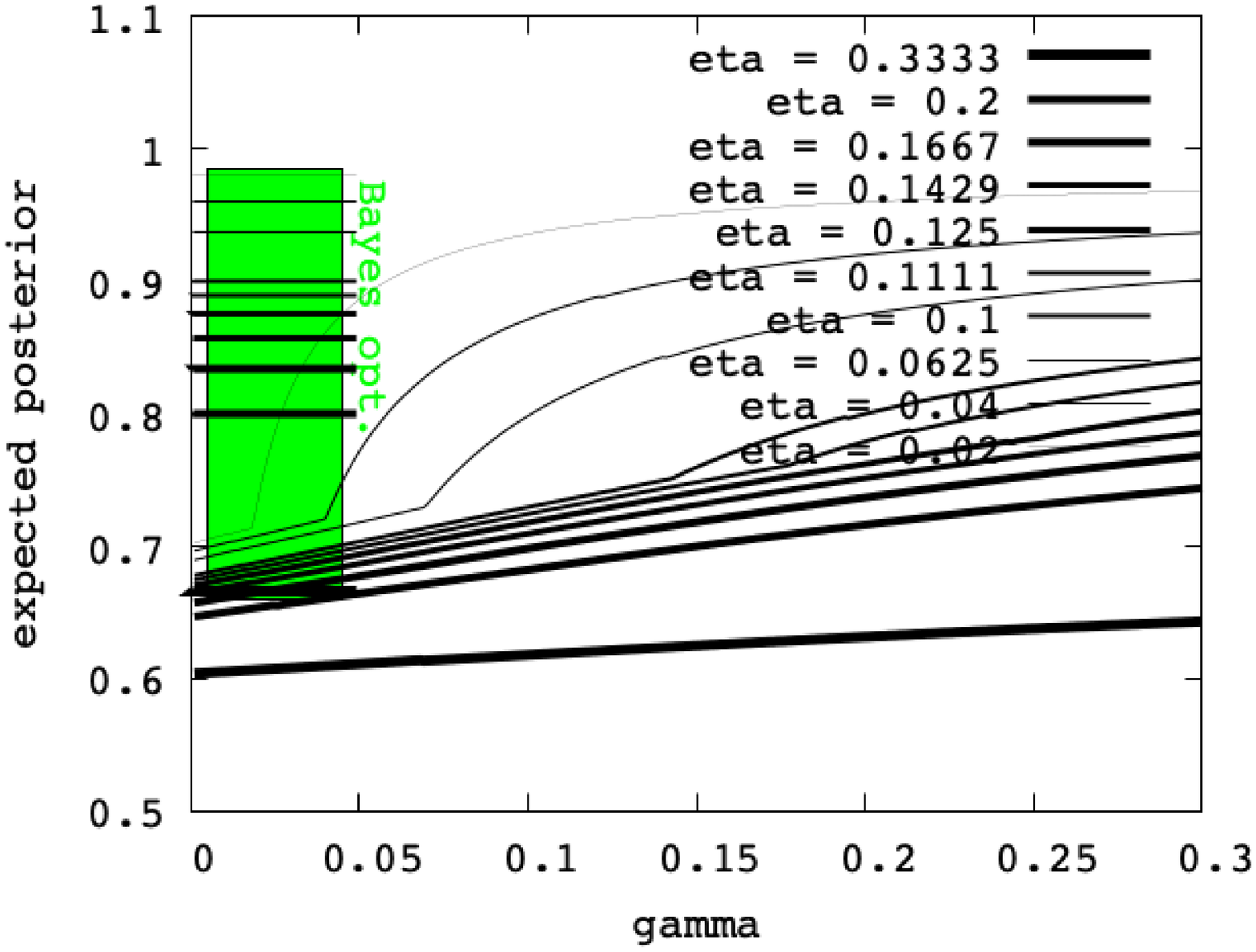} \basicnegativespacebefore & \basicnegativespacebefore \includegraphics[trim=50bp 0bp 10bp 10bp,clip,width=\picwidth\textwidth]{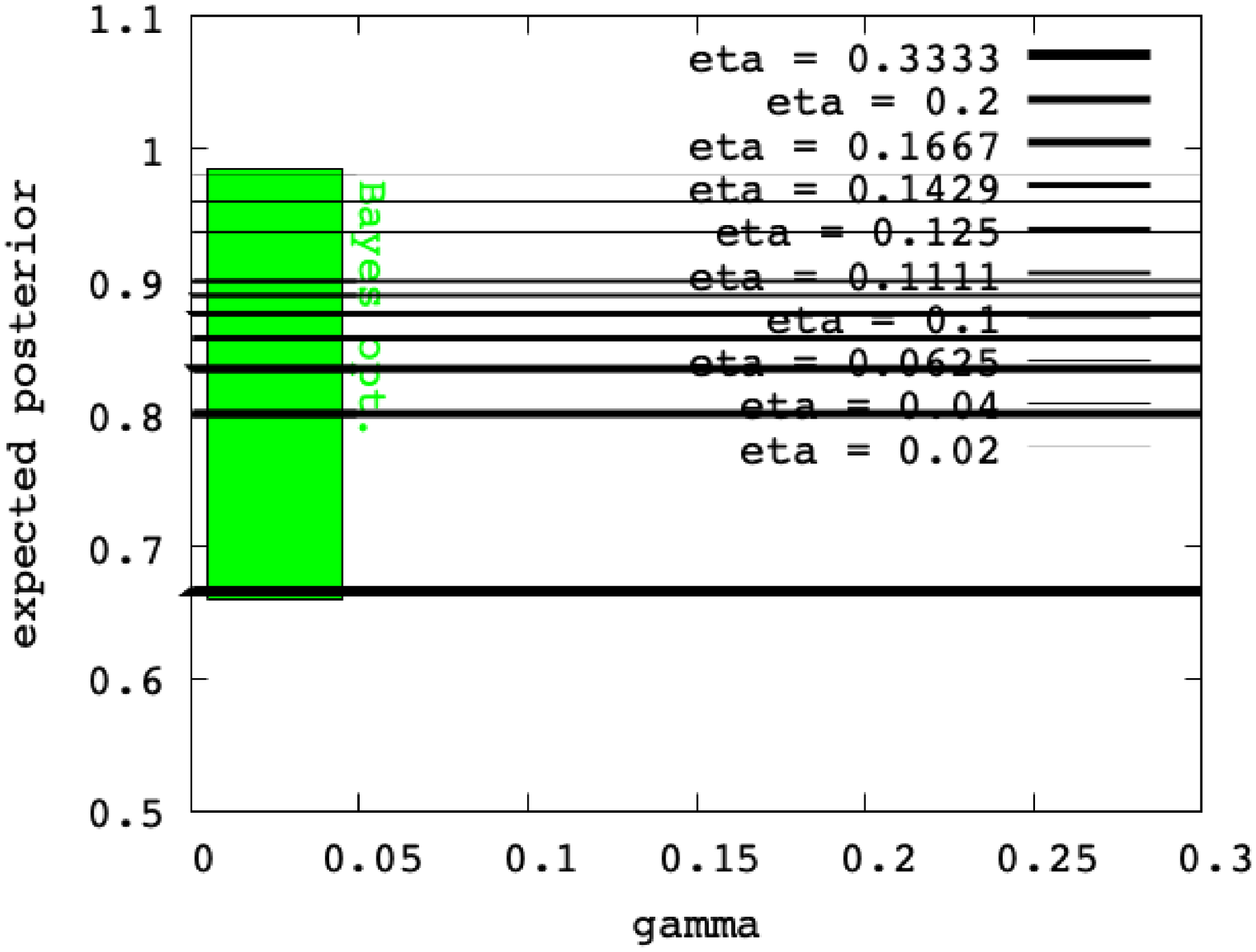} \basicnegativespace  & \basicnegativespacebefore \includegraphics[trim=50bp 0bp 10bp 10bp,clip,width=\picwidth\textwidth]{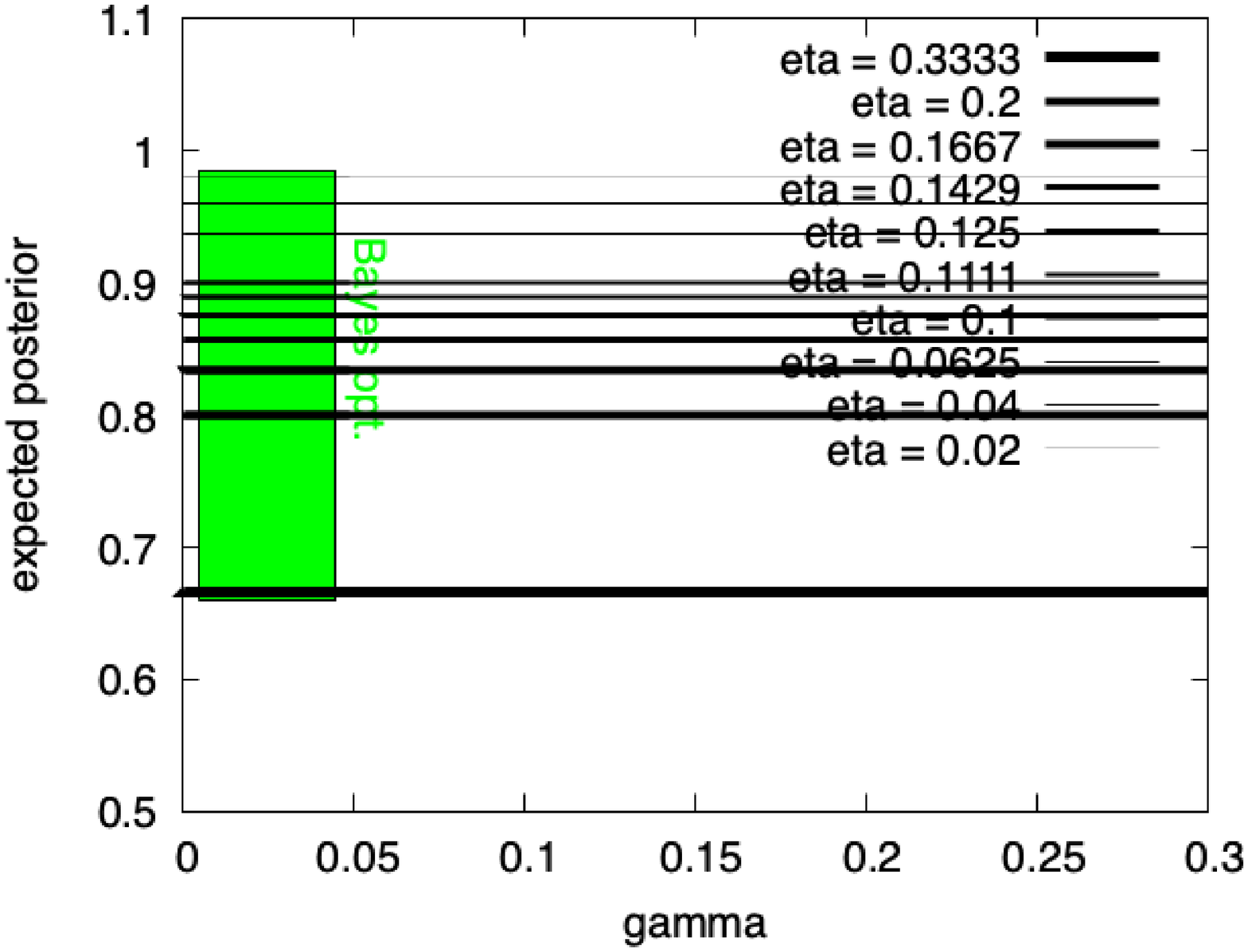} \basicnegativespace  & \basicnegativespace \includegraphics[trim=50bp 0bp 10bp 10bp,clip,width=\picwidth\textwidth]{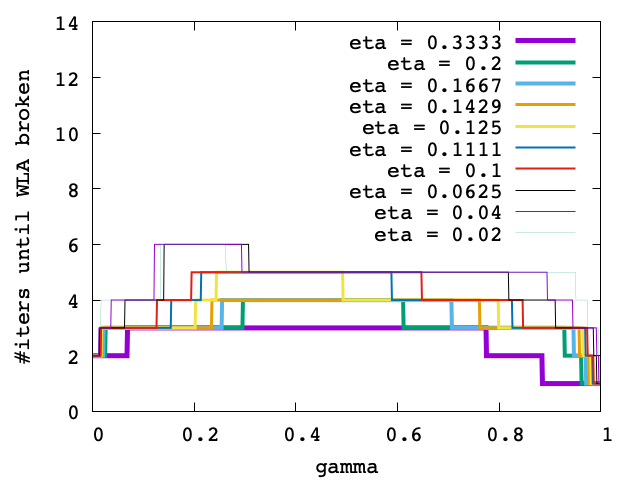} \basicnegativespacebefore & \basicnegativespacebefore \includegraphics[trim=50bp 0bp 10bp 10bp,clip,width=\picwidth\textwidth]{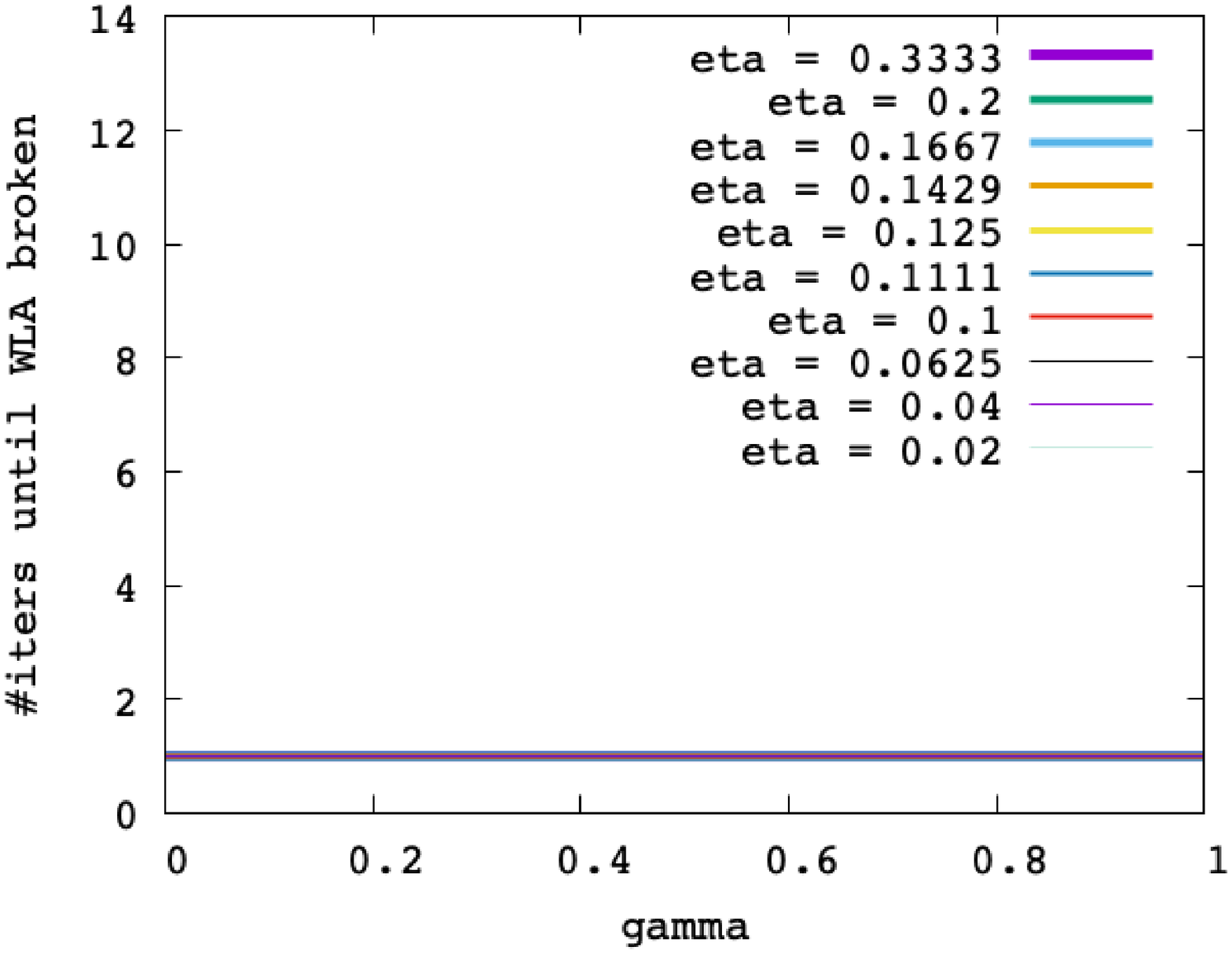} \basicnegativespace & \basicnegativespacebefore \includegraphics[trim=50bp 0bp 10bp 10bp,clip,width=\picwidth\textwidth]{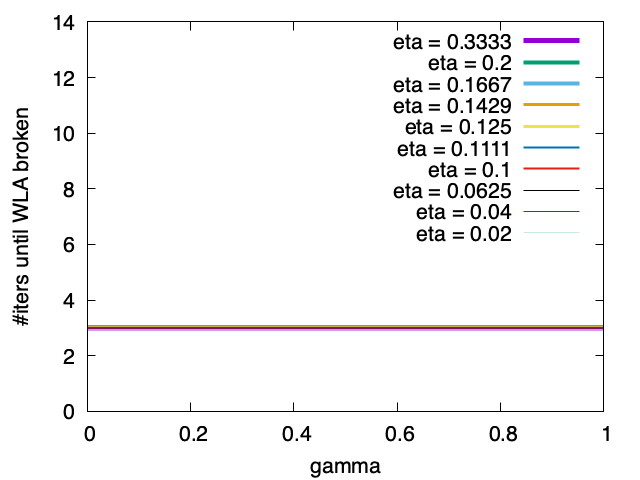} \basicnegativespace \\ 
\end{tabular}
\caption{Comparison of \topdowngen's results with the induction of linear separators (\cls), decision trees (\cdt) and 1-nearest neighbor (\cnn) after training on $\mathcal{S}_{\mbox{\tiny{noisy}}}$; each plot has parameter $\gamma$ in abscissa \eqref{defSclean} and in ordinate respectively the accuracy on $\mathcal{S}_{\mbox{\tiny{noisy}}}$ (left pane), the expected posterior estimation from $\mathcal{S}_{\mbox{\tiny{noisy}}}$ \eqref{maxliketa} (Bayes' optimum indicated in green rectangles, center pane), the number of iterations until the weak learner is "exhausted" (does not find a weak classifier with advantage $\gammawla \geq 0.001$, right pane). Different curves correspond to different values of the noise parameter $\etanoise$.}
    \label{tab:all-results}
  \end{table*}

%% file: content-arxiv/discussion.tex
\section{Discussion and conclusion: on parameterisation}\label{sec-discussion}

A partial explanation for the confusion about the results of \cite{lsRC} can be offered via the notion of \emph{parametrisation}. 
We elaborate this in terms of the three key ingredients of losses, models and algorithms. 
(We omit discussion of the other central ingredient of any learning problem, namely the data, or its 
theoretical representation as statistical experiments, which too, can be thought of parametrically \cite{torgersen1991cse}.)
Much ML research seems blind to the difference between a change of object, and a change in the parametrisation of 
an object. The clearest example of this is with loss functions, where it is known \cite{vanerven12a}  that the 
features of a loss function that govern its mixability depend only upon the induced geometry of its superprediction set.  
(Mixability controls whether one attains fast or slow learning in the online worst-case sequence prediction setting \cite{Vovk:1995}; 
there are generalisations of the notion that apply to the batch statistical setting \cite{vanerven15a}). 
From this perspective, losses are better thought of, and analysed in terms of the sets that they induce \cite{cranko2019proper}. 
The commonplace desire that the loss \emph{function} be convex (as a function) is controllable, independently, 
via a link function \cite{rwCB,Vernet2016}. Mathematically, the introduction 
of the link is tantamount to an (invertible, smooth) reparametrisation of the loss.  

There is one other point to be made about the loss function. The abstract idea of a loss function was developed by 
Wald \cite{Wald1950}  as a formalisation of the notion that when solving a data-driven problem, 
one ultimately has some goal in mind, and that can be captured by an outcome-contingent utility 
\cite{Berger1985}, or `loss'. 
Thus the loss is part of the problem statement. In contrast, in the ML literature, such as that arising from 
\cite{lsRC}, a loss function is considered as part of the specification of a `learning algorithm' (means of solving the problem).
From Wald's perspective, all of the work inspired by \cite{lsRC} is a perhaps not so surprising side-effect of 
attempting to solve one problem (classification using 0-1 loss) by using a method that utilises a 
\emph{different} loss. If one tries to repeat the negative example of \cite{lsRC} without the use of a 
surrogate, and always in terms of the Bayes optimal, there is nothing to see. When one adds some noise, 
the Bayes risk may change, but one will not see the apparent paradoxes of \cite{lsRC}. Recently, there has been a burst of research around new loss functions whole formulation aims to reduce the difficulty of the learning task, some becoming overwhelmingly popular \cite{lgghdFL}. One can see benefits of such a substantial shift from the normative view (of properness) to a more user-centric "\textit{\`a-la-Wald}'' design, but it usually comes with overloading loss functions with new hyperparameters. Technically, quantifing properties of the minimizers --- in effect, answering the question "\textit{what can be learned from this loss}" --- can be non-trivial \cite{snsBP} but it is an important task: Long and Servedio's result brightly demonstrates that we cannot reasonably stick with the choice "classifier = linear" and "loss = convex" (and eventually "algorithm = boosting") if the data is subject to corruption. One would have many reasons to stick with linear separators \textit{e.g.} for their simplicity and interpretability. In such a case, changing the loss, breaking properness and eventually convexity might just be a requirement.

A less widespread example is the reparametrisation of a model class. It is known (starting with Vapnik 
and a long line of refinements in the ML literature) that the statistical complexity  of a learning 
problem in the statistical batch setting is controlled by the complexity of the model class. This complexity 
is in terms of the class considered as a set, and is not influenced by how the elements of the class are 
parametrised. (The hardness is additionally influenced by the loss function, as per the above paragraph, 
however this is rarely made explicit; confer \cite{vanerven15a}). Thus the statistical complexity of learning 
with a model class comprising rational functions of degree $n$ will not depend upon whether the functions 
are parametrised in factored form, as partial fractions, or as ratios of polynomials in canonical sum form. 
Lest it be objected that no-one would use such a strange class, we note that in the simplest case 
analysable, classical sigmoidal neural networks can be reparametrised in terms of rational functions, 
and thus at least these three parametrisations are open to use \cite{williamson1995existence}. 
The parametrisation, whilst not changing 
the model class (or, say its VC dimension) \emph{will} change the behaviour of learning algorithms, 
in particular gradient based algorithms, which can misbehave due to attractors at infinity \cite{blackmore1996local}
--- a phenomenon caused by the choice of parametrisation.    

The final ingredient to consider is the algorithm. We have demonstrated that a boosting algorithm can 
be constructed that works successfully in the noisy situation (when using suitable model classes), 
but we have not really addressed head-on the perhaps more direct response to \cite{lsRC}, which is to 
challenge the definition of what is, and what is not, a `boosting algorithm.' There are several obvious 
ways to proceed here (e.g.~in terms of what the boosting algorithm is provided as input, in the form 
of weak learners). But all such attempts stumble over a more challenging issue: namely that there is no 
sensible way to compare algorithms --- we cannot even say `when is one algorithm equal to another?' \cite{blass2009two}. 
The irony is that the object that is most valorised in machine learning research, namely the algorithm, 
hardly satisfies the conceptual properties one demands of any `object' --- namely that we can tell 
when two objects are the same or different. We do not attempt to resolve this challenge here; 
indeed we think it is intrinsically unresolvable except up to a family of canonical isomorphisms, 
which need to be made explicit to really qualify as a legitimate answer \cite{Mazur:2008aa} -- perhaps this will give 
some insight into `natural' parametrisations of different learning algorithms. 
Of course, instead of attempting to do this, 
we could just focus our attention on other perspectives!

One promising perspective, from the algorithmic standpoint is, we believe, the need for boosting algorithms for more complex / overparameterized architectures. Quite
remarkably, in the dozens of references citing \cite{lsRC} we compile in \supplement, only one alludes to the key sufficient condition to solve \cite{lsRC}'s problem: Schapire mentions the potential lack of "richness" of hypotheses available to the weak learner in the context of AdaBoost \cite{sEA}. This resonates with a comment precisely in \cite{lsRC} whereby linear separators lack capacity to control confidences as would richer classes do, a model's flaw exploited by the negative results. Richness can be related to the choice in the weak hypotheses' architecture and thus their model's parameterisation; also, in the context of boosting, if we stick to the idea that the weak learner ultimately manipulates simple hypotheses -- which is important \textit{e.g.} in the context of \cdt~where this entails a substantial combinatorial search of splits --, then the question of \textit{where} to put the weak learning assumption arises. \topdowngen~shows that we can consider that boosting a linear combination of models with the logistic loss, each of which is a top-down \cdt~learned by minimizing C4.5's binary entropy, is in fact a \textit{recursive} application of \topdowngen~with the \textit{same loss} (log-loss) but different architectures, where the weak learning assumption only needs to be carried out at the tree's splits, then used to learn trees with \topdowngen, then used to learn the linear combination of trees also with \topdowngen. Hence, complex architectures could then be learned using recursive calls to boosters like \topdowngen, progressively building the architecture from basic building blocks at the bottom of the recursion searchable by a weak learner. This is convenient but still far from the "architecture's swiss army knife" optimization tool for ML that is stochastic gradient descent: can we make boosting "\textit{\`a-la-Kearns}" \cite{kTO} work for more architectures ?

To conclude, in this paper, we have used the founding theory of losses for class
probability estimation (properness) and a new boosting algorithm to demonstrate that the source of the negative result in the context of Long and Servedio's results is the model class. More than shining a new light on the model class' "responsibility" for the negative result, we believe our results rather show pitfalls of general parameterisations of a ML problem, including model class but also the learning algorithm and the loss function it optimizes. It remains an open question as to whether such stark "breakdowns" can be shown for highly ML-relevant triples (algorithm, loss, model) different from ours / \cite{lsRC}.

%% file: content-arxiv/acknowledgements.tex
\section*{Acknowledgments}

Many thanks to Phil Long for stimulating discussions around the material presented and tipping us on the rotation argument for the proof of Lemma \ref{lem-Rot}.

%% file: content-arxiv/whatthepaperssay.tex
\section{What the papers say}\label{sec-wtps}

Disclaimer: these are cut-paste exerpts of many papers citing \cite{lsRC} (or the earlier NeurIPS version)\footnote{Source: \url{https://scholar.google.com/scholar?oi=bibs\&hl=en\&cites=14973709218743030313\&as\_sdt=5}}, with emphasis on (i) most visible venues, (ii) variability (not just papers but also patents, etc.). Apologies for the eventual loss of context due to cut-paste.\\

\noindent "\textit{Servedio and Long [8] proved that, in general, any boosting algorithm that uses a convex
  potential function can be misled by random label noise}''  --- \cite{cefNC}\\

\noindent "\textit{Long and Servedio [2010] prove that any method based on a convex potential is inherently ill-suited to
  random label noise}" --- \cite{ndrtLW}\\

\noindent "\textit{Robustness of risk minimization depends on the loss function. For binary classification, it is shown that 0–1 loss is
robust to symmetric or uniform label noise while most of
the standard convex loss functions are not robust (Long and
Servedio 2010; Manwani and Sastry 2013)}" --- \cite{gksRL}\\

\noindent "\textit{Furthermore, the assumption of sufficient richness among the
weak hypotheses can also be problematic.
Regarding this last point, Long and Servedio [18] presented an example of a
learning problem which shows just how far off a universally consistent algorithm
like AdaBoost can be from optimal when this assumption does not hold, even when
the noise affecting the data is seemingly very mild.}" --- \cite{sEA}\\

\noindent "[...] \textit{it was shown that some boosting algorithms including AdaBoost are extremely
  sensitive to outliers [30].}" --- \cite{wwLT}\\

\noindent "\textit{Long and Servedio [2010] showed that there exist linearly separable $D$ where, when the learner
observes some corruption $\tilde{D}$ with symmetric label noise of any nonzero rate, minimisation of any
convex potential over a linear function class results in classification performance on D that is equivalent to random guessing. Ostensibly, this establishes that convex losses are not “SLN-robust” and
motivates the use of non-convex losses [Stempfel and Ralaivola, 2009, Masnadi-Shirazi et al., 2010,
Ding and Vishwanathan, 2010, Denchev et al., 2012, Manwani and Sastry, 2013].}" --- \cite{vmwLW}\\

\noindent "\textit{Long and Servedio (2008) have shown that boosting with convex potential functions (i.e., convex margin losses) is not robust to random class noise}" --- \cite{rwCB}\\

\noindent "\textit{Negative results for convex risk minimization in the presence of label noise have been
established by Long and Servido (2010) and Manwani and Sastry (2011). These works
demonstrate a lack of noise tolerance for boosting and empirical risk minimization based on
convex losses, respectively, and suggest that any approach based on convex risk minimization
will require modification of the loss, such that the risk minimizer is the optimal classifier
with respect to the uncontaminated distributions}" --- \cite{sbhCW}\\

\noindent "\textit{Boosting with convex loss functions is proven to
  be sensitive to outliers and label noise [19].}" --- \cite{sgplbOM}\\

\noindent "\textit{While
hinge loss used in SVMs (Cortes \& Vapnik, 1995) and
log loss used in logistic regression may be viewed as
convex surrogates of the 0–1 loss that are computationally efficient to globally optimize (Bartlett et al.,
2003), such convex surrogate losses are not robust to
outliers (Wu \& Liu, 2007; Long \& Servedio, 2010; Ding
\& Vishwanathan, 2010)}" --- \cite{nsAF}\\

\noindent "[...] \textit{For Theorem 29 to hold for AdaBoost, the richness assumption (72) is necessary,
since there are examples due to Long and Servedio (2010) showing that the theorem may not hold
when that assumption is violated}" --- \cite{msAT}\\

\noindent "[...] \textit{Long \& Servedio
(2008) essentially establish that if one does not assume
that margin error, $\nu$, of the optimal linear classifier is
small enough then any algorithm minimizing any convex loss $\phi$ (which they think of as a “potential”) can be
forced to suffer a large misclassification error.}" --- \cite{blssMT}\\

\noindent "\textit{The advantage of using a symmetric loss was investigated
in the symmetric label noise scenario (Manwani \& Sastry,
2013; Ghosh et al., 2015; Van Rooyen et al., 2015a). The
results from Long \& Servedio (2010) suggested that convex
losses are non-robust in this scenario}" --- \cite{clsOS}\\

\noindent "\textit{Overall, label noise
is ubiquitous in real-world datasets and will undermine the
performance of many machine learning models (Long \&
Servedio, 2010; Frenay \& Verleysen, 2014).}" --- \cite{clrtLW}\\

\noindent "\textit{Although desirable from an optimization standpoint, convex losses have been shown to be prone
  to outliers [15]}" --- \cite{awakRB}\\

\noindent "\textit{This is in contrast to recent
work by Long and Servedio, showing that convex potential boosters cannot work in the presence of
random classification noise [12].}" --- \cite{kkPB}\\

\noindent "\textit{The second strand has focussed on the design of surrogate losses robust to label noise.
Long and Servedio [2008] showed that even under symmetric label noise, convex potential minimisation
with such scorers will produce classifiers that are akin to random guessing.}" --- \cite{mvnLF}\\

\noindent "\textit{Negative results for convex risk minimization in the presence of label noise
have been established by Long and Servido [26] and Manwani and Sastry [27].
These works demonstrate a lack of noise tolerance for boosting and empirical
risk minimization based on convex losses, and suggest that any approach based
on convex risk minimization will require modification of the loss, [...]}" --- \cite{bfhps}\\

\noindent "\textit{For example, the random noise (Long and Servedio 2010) defeats
  all convex potential boosters [...]}" --- \cite{gwlzRM}\\

\noindent "\textit{Long and Servedio (2010) proved that any convex potential loss is not robust to uniform or symmetric label noise.}" --- \cite{gmsOT}\\

\noindent "\textit{We previously [23] showed that any boosting algorithm that works by stagewise minimization of a
  convex “potential function” cannot tolerate random classification noise}" --- \cite{lsLL}\\

\noindent "\textit{However, the convex loss
functions are shown to be prone to mistakes when outliers
exist [25].}" --- \cite{zllAS}\\

\noindent "\textit{[...] However, Long and Servedio
(2010) pointed out that any boosting algorithm with convex loss functions is highly susceptible to a
random label noise model.}" --- \cite{lbBI}\\

\noindent "\textit{One drawback of many standard boosting techniques, including AdaBoost, is that
  they can perform poorly when run on noisy data [FS96, MO97, Die00, LS08].}" --- \cite{lsAM}\\

\noindent "\textit{Therefore, it has been shown
  that the convex functions are not robust to noise [13].}" --- \cite{awsTT}\\

\noindent "\textit{This is because many boosting algorithms are vulnerable to noise (Dietterich, 2000; Long
  and Servedio, 2008).}" --- \cite{cbcCE}\\

\noindent "\textit{Long
  and Servedio (2010) showed that there is no convex loss that is robust to label noises.}" --- \cite{bssCS}\\

\noindent "\textit{[...] However, as was recently shown by Long and Servedio
  [4], learning algorithms based on convex loss functions are not robust to noise}" --- \cite{dvTL}\\

\noindent "\textit{[...] For instance, several papers show how outliers and noise can cause
linear classifiers learned on convex surrogate losses to suffer high zero-one loss (Nguyen and Sanner,
2013; Wu and Liu, 2007; Long and Servedio, 2010).}" --- \cite{mlUS}\\

\noindent "\textit{This is as opposed to most boosting algorithms that are highly susceptible to outliers [24].}" --- \cite{ncRF}\\

\noindent "\textit{Moreover, in the case of boosting, it has been shown that convex
  boosters are necessarily sensitive to noise (Long and Servedio 2010 [...]}" --- \cite{gSM}\\

\noindent "\textit{Ostensibly, this result establishes that convex losses are not robust to symmetric label noise, and motivates using non-convex losses [40, 31, 17, 15,
  30].}" --- \cite{vmAA}\\

\noindent "\textit{Interestingly, (Long and Servedio,
2010) established a lower bound against potential-based convex boosting techniques in the presence
of RCN.}" --- \cite{diklstBI}\\

\noindent "\textit{However, it was shown in (Long \& Servedio,
2008; 2010) that any convex potential booster can be easily
defeated by a very small amount of label noise}" --- \cite{ppIN}\\

\noindent "\textit{A major
roadblock one has to get around in label noise algorithms is the non-robustness
of linear classifiers from convex potentials as given in [10]. }" ---
\cite{thCS}\\

\noindent "\textit{Coming from the other end, the main argument for non-convexity is that
a convex formulation very often fails to capture fundamental properties of a real problem
(e.g. see [1, 2] for examples of some fundamental limitations of convex loss functions).}" ---
\cite{lpcLA}\\

\noindent "\textit{A theoretical
analysis proposed in [21] proves that any method based on convex surrogate loss is inherently ill-suited to
random label noise.}" ---
\cite{xhCA}\\

\noindent "\textit{It has been observed that application of Friedman's stochastic
gradient boosting to deep neural network training often led to training instability . See , e.g.  Philip M. Long , et al , “ Random Classification Noise Defeats
All Convex Potential Boosters , ” in Proceedings of the 25th International Conference on Machine Learning}" ---
\cite{omlSG}\\

\noindent "\textit{Long and Servedio [2010] showed that random classification
noise already makes a large class of convex boosting-type algorithms fail.}" ---
\cite{tOT}\\

\noindent "\textit{On the other hand, it has been known that boosting methods work rather
poorly when the input data is noisy. In fact, Long and Servedio show that
any convex potential booster suffer from the same problem [6].}" ---
\cite{cLO}\\

\noindent "\textit{Noise-resilience also appears to make CTEs outperform one of their
most prominent competitors – boosting – whose out-of-sample AUC estimates appear to be held back
by the level of noise in macroeconomic data (also see Long and Servedio, 2010) }" ---
\cite{wEI}\\

\noindent "\textit{The brittleness of convex surrogates is not unique to ranking, and plagues their
use in standard binary classification as well (Long and Servedio 2010; Ben-David et al.
2012). }" ---
\cite{mTR}\\

%% file: content-arxiv/appendix.tex
\section{Supplementary material on proofs} \label{sec-sup-pro}

\subsection{Proof of Lemma \ref{lem-PHI1}}\label{proof-lem-PHI1}

Strict convexity follows from its definition. Letting $\mathbb{I}\defeq \poibayesrisk'([0,1])$, we observe:
\begin{eqnarray}
  \philoss(z) \defeq\sup_{u\in [0,1]} \{-zu + \poibayesrisk(u)\} =
                         \left\{
                         \begin{array}{ccl}
                           -z +\poibayesrisk(1)& \mbox{ if } & z\leq \inf \mathbb{I}\\
                           -z \cdot \estposterior(-z) + \poibayesrisk(\estposterior(-z)) & \mbox{ if } & z \in \mathbb{I}\\
                           \poibayesrisk(0) & \mbox{ if } & z\geq \sup \mathbb{I}
                           \end{array}
                         \right. .\label{eqvconvsur}
\end{eqnarray}
This directly establishes $\lim_{+\infty} \philoss(z) = \poibayesrisk(0)$. Strict properness and differentiability ensure $\poibayesrisk'$ strictly decreasing. We also have
  \begin{eqnarray}
\philoss'(z)  = \left\{
                         \begin{array}{ccl}
                           -1 & \mbox{ if } & z\leq \inf \mathbb{I}\\
                           - ({\poibayesrisk'}^{-1})(-z) & \mbox{ if } & z \in \mathbb{I}\\
                           0 & \mbox{ if } & z\geq \sup \mathbb{I}
                           \end{array}\right.,
  \end{eqnarray}
  which shows $\philoss'(z) \leq 0, \forall z \in \mathbb{R}$ and so $\philoss$ is decreasing. The definition of $\mathbb{I}$ ensures $\lim_{\inf \mathbb{I}} \philoss'(z) = -1, \lim_{\sup \mathbb{I}} \philoss'(z) = 0$ so $\philoss$ is differentiable. Convexity follows from the definition of $\philoss$.

  We now note the useful relationship coming from properness condition and \eqref{eqpoirisk} (main file):
  \begin{eqnarray}
  \poibayesrisk'(u) & = & \partialloss{1}(u) - \partialloss{-1}(u).\label{eqPLDiff}
\end{eqnarray}
This relationship brings two observations: first, the partial losses
being differentiable, they are continuous and thus $\poibayesrisk'$ is
continuous as well, which, together with $\mathrm{dom}(\poibayesrisk)
= [0,1]$ brings the continuity of $\philoss'$ and so $\philoss$ is
$C^1$. The second is $\philoss'(0)<0$. We first show $0 \in
\mathrm{int} \mathbb{I}$. Because of \eqref{eqPLDiff}, if $0 \not\in
\mathrm{int} \mathbb{I}$, we either have $\partialloss{1}(0) -
\partialloss{-1}(0) \leq 0$ or $\partialloss{1}(1) -
\partialloss{-1}(1) \geq 0$. The integral representation of proper
losses  \cite{rwCB}(Theorem 1) \cite{nmSL} (Appendix Section 9) yields that there exists a non-negative weight function $w : (0,1) \rightarrow \mathbb{R}_+$ such that
\begin{eqnarray}
\partialloss{1}(u) = \int_{u}^1 (1-t) w(t) \mathrm{d}t & ; & \partialloss{-1}(u) = \int_{0}^u t w(t) \mathrm{d}t .
\end{eqnarray}
The condition $\partialloss{1}(0) - \partialloss{-1}(0) \leq 0$ imposes
\begin{eqnarray}
\lim_{u\rightarrow 0} \int_{u}^1 (1-t) w(t) \mathrm{d}t = \partialloss{1}(0) & \leq & \partialloss{-1}(0) = \lim_{u\rightarrow 0} \int_{0}^u t w(t) \mathrm{d}t  = 0,
\end{eqnarray}
which imposes $w(.) = 0$ almost everywhere and $\partialloss{1}(u) = 0, \forall u$. Similarly, the condition $\partialloss{1}(1) - \partialloss{-1}(1) \geq 0$ imposes
\begin{eqnarray}
\lim_{u\rightarrow 1} \int_{u}^1 (1-t) w(t) \mathrm{d}t = \partialloss{1}(1) & \geq & \partialloss{-1}(1) = \lim_{u\rightarrow 1} \int_{0}^u t w(t) \mathrm{d}t  = 0,
\end{eqnarray}
which also imposes $w(.) = 0$ almost everywhere and $\partialloss{-1}(u) = 0, \forall u$. $w(.) = 0$ almost everywhere implies $\partialloss{1}(u) = \partialloss{-1}(u) = \poibayesrisk(u) = 0, \forall u$, which is impossible given strict properness. So we get $0 \in \mathrm{int} \mathbb{I}$ and since $\poibayesrisk'$ is strictly decreasing, ${\poibayesrisk'}^{-1}(0) > 0$, implying
\begin{eqnarray}
\philoss'(0)  =  - ({-\poibayesrisk'})^{-1}(0) = - ({\poibayesrisk'}^{-1})(0)  < 0,
\end{eqnarray}
and ending the proof of Lemma \ref{lem-PHI1}.

\subsection{Proof of Lemma \ref{lem-GEN1}}\label{proof-lem-GEN1}

 We first simplify \eqref{phitomin} to a criterion equivalent to \cite[eq. 5]{lsRC} (notations follow theirs):
\begin{eqnarray*}
  \tilde{\popsur}(h, \mathcal{S}) & = & (N+1) \philoss(-\alpha_1) - N \alpha_1 + 2(N+1) \philoss(-\alpha_1 \gamma +\alpha_2 \gamma) - 2N (\alpha_1 \gamma -\alpha_2 \gamma)\\
                                  & & + (N+1) \philoss(-\alpha_1\gamma - K\alpha_2 \gamma) - N (\alpha_1\gamma + K\alpha_2 \gamma)\\
                                  & = & (N+1)\cdot \left(\philoss(-\alpha_1)  + 2 \philoss(-\alpha_1 \gamma +\alpha_2 \gamma) + \philoss(-\alpha_1\gamma - K\alpha_2 \gamma)\right) \nonumber\\
  & & - N\cdot \left((1+3\gamma)\alpha_1 + (K-2)\alpha_2\gamma\right)
\end{eqnarray*}
We are interested in the properties of the linear classifier $h$ minimizing that last expression. Denote for short:
\begin{eqnarray*}
  \varphi(z) & \defeq & \philoss'(z) + (1-\etanoise),\\
  \tilde{P}_1(\alpha_1,\alpha_2) & \defeq & \frac{1}{N+1} \cdot \frac{\partial \tilde{\popsur}(h, \mathcal{S})}{\partial \alpha_1},\\
  \tilde{P}_2(\alpha_1,\alpha_2) & \defeq & \frac{1}{\gamma(N+1)} \cdot \frac{\partial \tilde{\popsur}(h, \mathcal{S})}{\partial \alpha_2}.
\end{eqnarray*}
We note $\varphi$ is increasing and satisfies $\lim_{-\infty} \varphi = -\etanoise, \lim_{+\infty} \varphi = 1 - \etanoise$. We get
\begin{eqnarray}
  \tilde{P}_1(\alpha_1,\alpha_2) & = & -\philoss'(-\alpha_1)- \gamma \cdot\left\{ 2 \philoss'((\alpha_2-\alpha_1) \gamma) + \philoss'(-(\alpha_1+K\alpha_2) \gamma) \right\}- \frac{N (1+3\gamma)}{N+1}\nonumber\\
  & = & -\varphi(-\alpha_1) - 2\gamma\varphi((\alpha_2-\alpha_1) \gamma) - \gamma \varphi(-(\alpha_1+K\alpha_2) \gamma),
\end{eqnarray}
and
\begin{eqnarray}
  \tilde{P}_2(\alpha_1,\alpha_2) & = & 2 \philoss'((\alpha_2-\alpha_1) \gamma) - K \philoss'(-(\alpha_1+K\alpha_2) \gamma) - \frac{N (K-2)}{N+1}\nonumber\\
  & = & 2 \varphi((\alpha_2-\alpha_1) \gamma) - K \varphi(-(\alpha_1+K\alpha_2) \gamma).
\end{eqnarray}
\begin{figure}[t]
  \begin{center}
    \begin{tabular}{ccc}
      \includegraphics[trim=5bp 520bp 540bp 10bp,clip,width=0.3\linewidth]{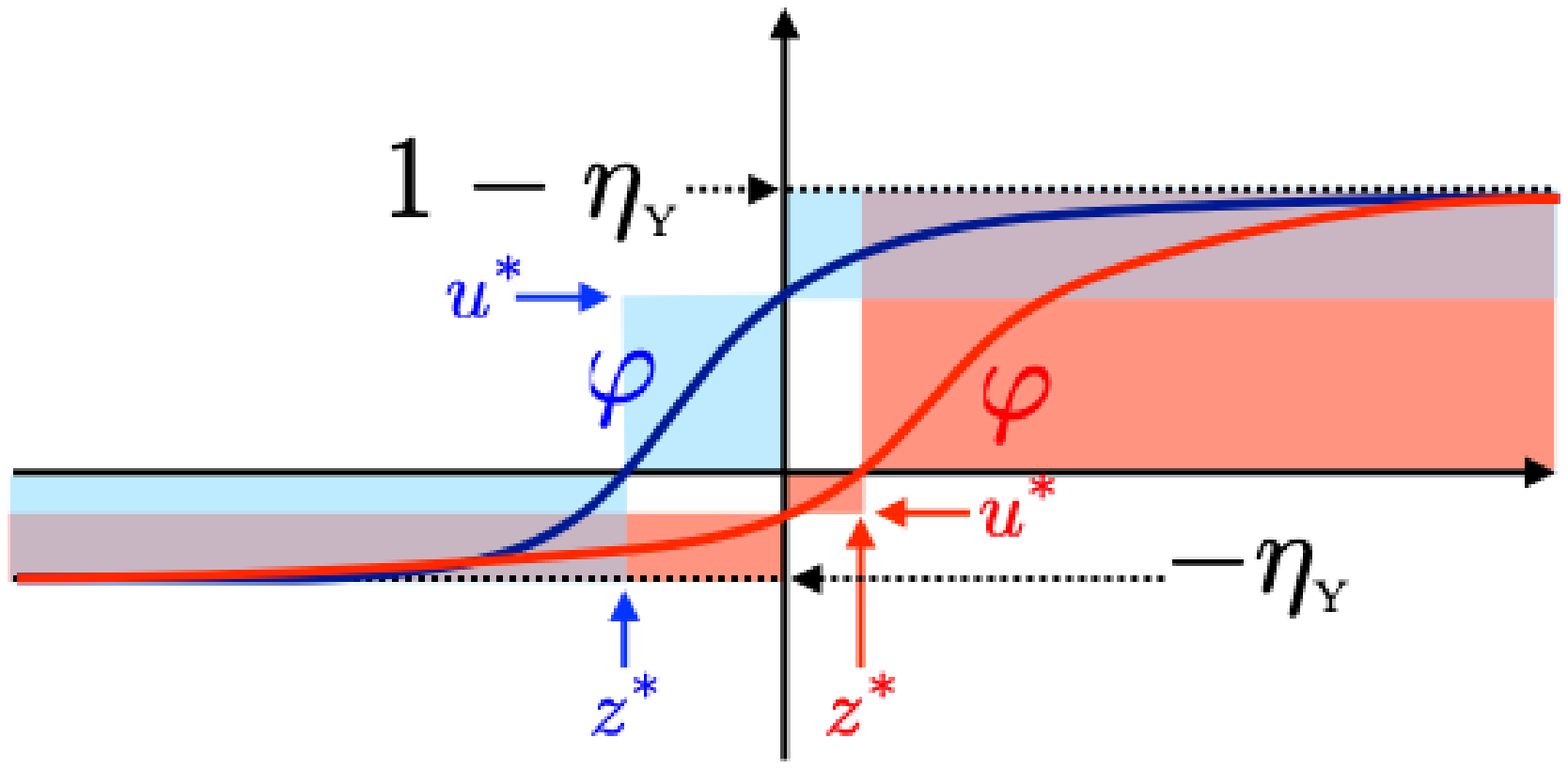} & \includegraphics[trim=5bp 520bp 540bp 10bp,clip,width=0.3\linewidth]{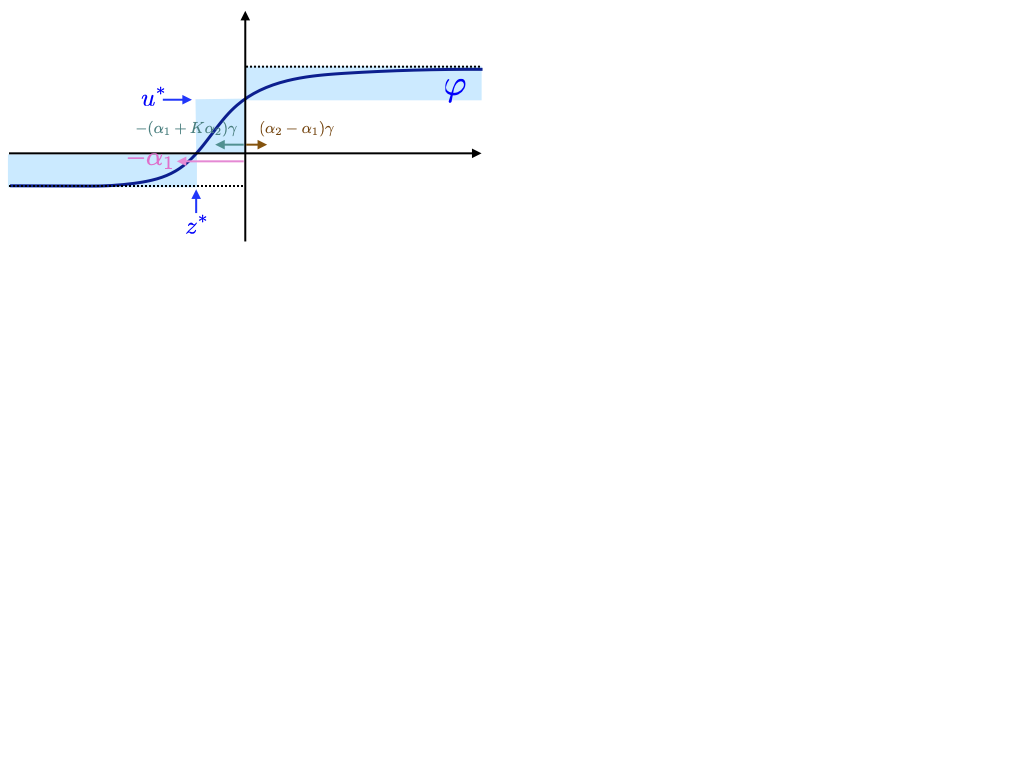} & \includegraphics[trim=5bp 520bp 540bp 10bp,clip,width=0.3\linewidth]{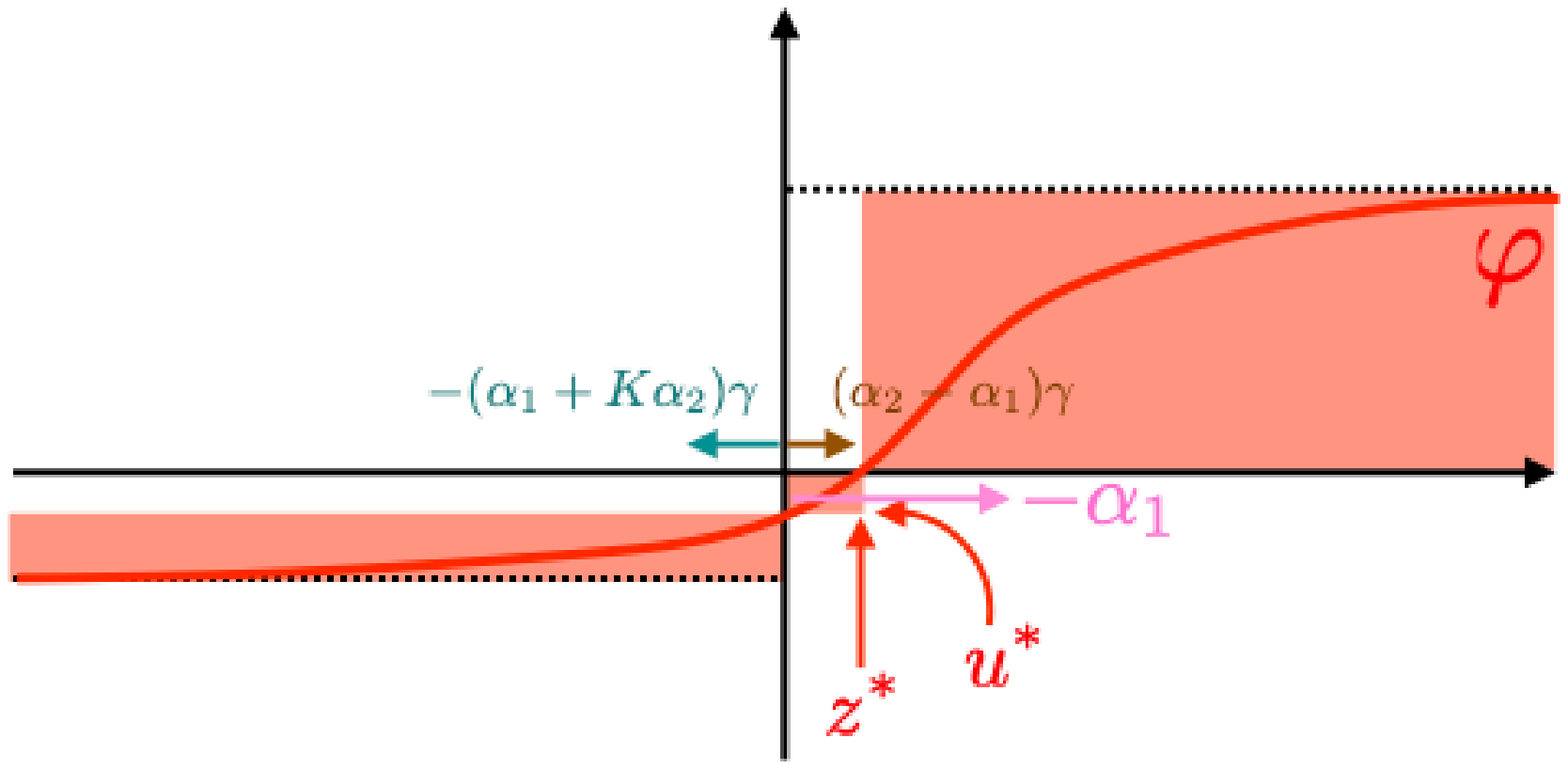}
    \end{tabular}
\end{center}
\caption{The two cases of our analysis for the proof of Lemma \ref{lem-GEN1}. In each case, we show the polarity of each of the arguments of system \eqref{eqSys}.}
  \label{f-varphi}
\end{figure}
The system that zeroes both functions $\tilde{P}_1(\alpha_1,\alpha_2), \tilde{P}_2(\alpha_1,\alpha_2)$ is thus equivalent to having
\begin{eqnarray}
  \left\{ \begin{array}{rrcl}
            (i) & \varphi(-(\alpha_1+K\alpha_2) \gamma) & = & \frac{2}{K} \cdot \varphi((\alpha_2-\alpha_1) \gamma) \\
            (ii) & \frac{-\varphi(-\alpha_1)}{\gamma} & = & \frac{2(K+1)}{K} \cdot \varphi((\alpha_2-\alpha_1) \gamma)
            \end{array}\right..\label{eqSys}
\end{eqnarray}
We have two cases to solve this system, presented in Figure \ref{f-varphi}: a "red" case, representing "high" noise, for which $\varphi(0) \defeq u^* < 0$, and a "blue" case, representing "low" noise, for which $\varphi(0) \defeq u^* > 0$.

\noindent \textbf{Red case}: we solve \eqref{eqSys} for the constraints $\alpha_2 > -\alpha_1, \alpha_1 < -z^*$; we pick $K = 2/(1+\epsilon)$ for some small $0<\epsilon<1$. Pick $\alpha_2  \defeq (1+\epsilon)(1+B)\cdot - \alpha_1 > -\alpha_1$ for $B\geq 0$.  The system \eqref{eqSys} becomes:
  \begin{eqnarray}
  \left\{ \begin{array}{rrcl}
            (i) & \varphi((1+2B)\alpha_1\gamma)& = & (1+\epsilon) \cdot \varphi( (1+(1+\epsilon)(1+B))\cdot - \alpha_1\gamma)\\
            (ii) & \frac{-\varphi(-\alpha_1)}{\gamma} & = & (3+\epsilon) \cdot \varphi( (1+(1+\epsilon)(1+B))\cdot - \alpha_1\gamma)
            \end{array}\right..\label{eqSys3}
  \end{eqnarray}
Suppose
  \begin{eqnarray*}
\alpha_1 \gamma & = & \delta < 0,
  \end{eqnarray*}
 for a small $|\delta|$. For any such constant $\delta > 0$, we see that
  \begin{eqnarray*}
\frac{-\varphi(-\alpha_1)}{\gamma} & = & \frac{\alpha_1 \cdot -\varphi(-\alpha_1)}{\delta} \defeq V(\alpha_1)
  \end{eqnarray*}
  and this time $V$ satisfies $\lim_{-z^*} V = 0, \lim_{-\infty} V = -\infty$ and $V$ is continuous because $\varphi$ is, so for any value of the RHS in $(ii)$ that keeps $(\alpha_2-\alpha_1) \gamma \in [0, z^*)$, the product $\alpha_1\gamma$ can be split in a couple $(\alpha_1, \gamma)$ for which the LHS in $(ii)$ equates its RHS. We then just have to find a solution to $(i)$ that meets our domain constraints. We observe that $(i)$ becomes:
    \begin{eqnarray}
\varphi((1+2B)\delta)& = & (1+\epsilon) \cdot \varphi( (1+(1+\epsilon)(1+B))\cdot - \delta), \label{simplI0}
    \end{eqnarray}
    whose quantities satisfy because of the monotonicity of $\varphi$,
    \begin{eqnarray}
\forall B\geq 0, \forall \delta\leq 0, \varphi((1+2B)\delta) & \leq & \varphi( (2+B)\cdot - \delta), \label{ineqD1}
    \end{eqnarray}
    which is \eqref{simplI0} for $\epsilon = 0$. We now show that there is a triple $(\epsilon, B, \delta)$ with $\delta < 0, 0<(\alpha_2-\alpha_1) \gamma = (1+(1+\epsilon)(1+B))\cdot - \delta < z^*, B\geq 0, \epsilon \geq 0$ which reverses the inequality, showing, by continuity of $\varphi$, a solution to \eqref{simplI0}. Fix small constants $\Delta_x, \Delta_y > 0$ such that we simultaneously have
    \begin{eqnarray}
      (1+2B)\delta & = & -\Delta_x,\label{constD1}\\
      \Delta_y & < & \frac{-u^*}{3}, \label{constD2}\\
      \varphi(-\Delta_x) & \geq & u^* - \Delta_y, \label{constD3}\\
      \varphi(\Delta_x) & \leq & u^* + \Delta_y. \label{constD4}
    \end{eqnarray}
    The RHS of \eqref{simplI0} becomes $(1+\epsilon) \cdot \varphi\left( J(\epsilon, B) \cdot \Delta_x\right)$ with
    \begin{eqnarray}
      J(\epsilon, B) & \defeq & \frac{1+(1+\epsilon)(1+B)}{1+2B}.
    \end{eqnarray}
    $ J(\epsilon, B)$ satisfies the following property (P):
    \begin{eqnarray*}
\forall 0\leq \epsilon < 1, \exists B > 0: J(\epsilon, B) = J(0,0) = 1.
    \end{eqnarray*}
    Thanks to (P) and the continuity of $J$ and $\varphi$, all we need to show for the existence of a solution to $(i)$ is that there exists $\epsilon < 1$ such that the central inequality underscored with "?" can hold,
    \begin{eqnarray}
(1+\epsilon) \cdot \varphi( \underbrace{J(\epsilon, B)}_{=J(0,0)=1} \cdot \Delta_x) \underbrace{\leq}_{\mbox{\eqref{constD4}}} (1+\epsilon) \cdot (u^* + \Delta_y) & \underbrace{\leq}_{?} & u^* - \Delta_y \underbrace{\leq}_{\mbox{\eqref{constD3}}} \varphi(-\Delta_x). \label{ineqD2}
    \end{eqnarray}
    \eqref{constD2} is equivalent to:
    \begin{eqnarray*}
\frac{2\Delta_y}{-(u^* + \Delta_y)} & < & 1,
    \end{eqnarray*}
    so any $2\Delta_y / -(u^* + \Delta_y) \leq \epsilon < 1$ brings equivalently $(1+\epsilon) \cdot (u^* + \Delta_y) \leq u^* - \Delta_y $, which is "?" above.

    Then, to solve $(i)$, we first choose $\Delta_y$ satisfying \eqref{constD2}, then pick $\Delta_x$ so that \eqref{constD3} and \eqref{constD4} are satisfied. This fixes the LHS of \eqref{simplI0}. From its minimal value $\epsilon = 0$, we progressively increase $\epsilon$ while computing $B$ so that (P) holds and getting $\delta$ from \eqref{constD1}; while for $\epsilon = 0$ \eqref{ineqD1} holds, we know that there is an $\epsilon < 1$ such that \eqref{ineqD2} holds, the continuity of $\varphi$ then showing there must be a value in the interval of $\epsilon$s for which equality, and thus $(i)$, holds.

    Then, from the value $\delta = \alpha_1 \gamma$ obtained, we compute the couple $(\alpha_1, \gamma), \alpha_1 <0, \gamma > 0$ such that $(ii)$ holds, and then get $\alpha_2$ from the identity $\alpha_2  \defeq (1+\epsilon)(1+B)\cdot - \alpha_1 > -\alpha_1$.\\
    
\noindent \textbf{Blue case}: we solve \eqref{eqSys} for the constraints $\alpha_2 > \alpha_1 > 0$; we pick $K = 2/(1-\epsilon)$ for some small $0<\epsilon<1$. Pick $\alpha_2  \defeq (1-\epsilon)(1+B)\alpha_1 > \alpha_1$ for $\alpha_1>0, B > \epsilon/(1-\epsilon)$.  The system \eqref{eqSys} becomes:
  \begin{eqnarray}
  \left\{ \begin{array}{rrcl}
            (i) & \varphi(-(1+2(1+B))\alpha_1\gamma)& = & (1-\epsilon) \cdot \varphi( (B(1-\epsilon) - \epsilon) \alpha_1\gamma)\\
            (ii) & \frac{-\varphi(-\alpha_1)}{\gamma} & = & (3-\epsilon) \cdot \varphi( (B(1-\epsilon) - \epsilon) \alpha_1\gamma)
            \end{array}\right..\label{eqSys2}
  \end{eqnarray}
  Suppose
  \begin{eqnarray*}
\alpha_1 \gamma & = & \delta > 0,
  \end{eqnarray*}
  a small constant. For any such constant $\delta > 0$, we see that
  \begin{eqnarray*}
\frac{-\varphi(-\alpha_1)}{\gamma} & = & \frac{\alpha_1 \cdot -\varphi(-\alpha_1)}{\delta} \defeq V(\alpha_1)
  \end{eqnarray*}
  and $V$ satisfies $\lim_{-z^*} V = 0, \lim_{+\infty} V = +\infty$ and $V$ is continuous because $\varphi$ is, so for any value of the RHS in $(ii)$, there exists a solution $(\alpha_1, \gamma)$ to $(ii)$. We just need to figure out a solution to $(i)$ for \textit{some} $\delta > 0$ such that all our domain constraints are met. We observe $(i)$ becomes
  \begin{eqnarray}
\varphi(-(1+2(1+B))\delta)& = & (1-\epsilon) \cdot \varphi( (B(1-\epsilon) - \epsilon) \delta) \defeq W(\epsilon). \label{simplI}
  \end{eqnarray}
  As $\delta \rightarrow 0^+$, the domain of solutions $(\epsilon,B)$ to $(i)$ converges to $\{0\} \times \mathbb{R}$. $B>0$ being fixed, we observe $W$ is continuous and (noting the constraint $\epsilon < B/(1+B)$)
  \begin{eqnarray*}
W(0) = \varphi(B \delta) \geq u^* &;& W\left(\frac{B}{1+B}\right) = \frac{u^*}{1+B}.
  \end{eqnarray*}
  Remark that if we pick $\delta$ such that
  \begin{eqnarray*}
\varphi(-(1+2(1+B))\delta) & \in & \left(\frac{u^*}{1+B}, u^*\right),
  \end{eqnarray*}
  then there exists a solution $0<\epsilon < B/1+B$ to $(i)$ so we get $K > 2$ and ratio $\alpha_2 / \alpha_1 > 1$. Then we solve $(ii)$ for $(\alpha_1, \gamma)$ and get $\alpha_1 > 0$ and $\gamma > 0$.
  
  \noindent \textbf{Summary}: accuracy of the optimal solution on $\mathcal{S}_{\mbox{\tiny{clean}}}$. In the \textbf{blue case}, we see that $\alpha_1 > 0, \alpha_2 > \alpha_1$, thus the accuracy is $50\%$. In the \textbf{red case} however, we see that, because $\alpha_1 < 0, \alpha_2 > -\alpha_1$, three examples of $\mathcal{S}_{\mbox{\tiny{clean}}}$ are badly classified and the accuracy thus falls to $25\%$.

  \subsection{Proof of Lemma \ref{lem-Rot}}\label{proof-lem-Rot}

  The trick we use is the same as in \cite{lsRC}: we rotate the whole sample (which rotates accordingly the optimum and thus does not change its properties, loss-wise) in such a way that any booster would pick a "wrong direction" to start, where the direction picked is the one with the largest edge \eqref{defWLA3}. Let the rotation matrix of angle $\theta$, with $c \defeq \cos \theta, s \defeq \sin \theta$,
  \begin{eqnarray}
    \matrice{R}_{\theta} & \defeq & \left[
                                    \begin{array}{cc}
                                      c & -s\\
                                      s & c
\end{array}
                                    \right].
  \end{eqnarray}
  Denoting the rotated sample
  \begin{mdframed}[style=MyFrame,nobreak=true,align=center]
  \begin{eqnarray}
    \mathcal{S}_{\mbox{\tiny{clean}},\theta} & \defeq & \left\{\left(\left[
                      \begin{array}{c}
                        c\\
                         s
                        \end{array}
  \right],1\right), \left(\left[
                      \begin{array}{c}
                        (c+s)\gamma\\
                         (s-c)\gamma
                        \end{array}
  \right],1\right), \left(\left[
                      \begin{array}{c}
                        (c+s)\gamma\\
                         (s-c)\gamma
                        \end{array}
  \right],1\right), \left(\left[
                      \begin{array}{c}
                        (c-Ks)\gamma\\
                         (s+Kc)\gamma
                        \end{array}
  \right],1\right)\right\}, 
  \end{eqnarray}
\end{mdframed}
We note the sum of weights $W$, letting $L \defeq ({-\poibayesrisk'})^{-1}(0) \in (0,1)$:
\begin{eqnarray}
W & = & 4(1-\etanoise)(1-L) + 4\etanoise L = 4(1-\etanoise-L+2\etanoise L),
  \end{eqnarray}
and we compute both edges \eqref{defWLA3} for both coordinates with the noisy dataset $\mathcal{S}_{\mbox{\tiny{noisy}},\theta}$ by ranging through left to right of the examples' observations in $\mathcal{S}_{\mbox{\tiny{noisy}},\theta}$:
\begin{eqnarray*}
  \texttt{e}_x & = & \frac{\left\{\begin{array}{c}(1-\etanoise)(1-L)c - \etanoise L c + 2(1-\etanoise)(1-L)(c+s)\gamma - 2\etanoise L (c+s)\gamma \\+ (1-\etanoise)(1-L)(c-Ks)\gamma - \etanoise L (c-Ks)\gamma\end{array}\right.}{W}\\
  & = & \frac{(1-\etanoise-L)(1+3\gamma)(c-a\cdot s)}{W},
  \end{eqnarray*}
and
\begin{eqnarray*}
  \texttt{e}_y & = & \frac{(1-\etanoise-L)(1+3\gamma)(a \cdot c + s)}{4},
\end{eqnarray*}
with
\begin{eqnarray*}
a & \defeq & \frac{(K-2)\gamma}{1+3\gamma}.
\end{eqnarray*}
We also remind from Lemma \ref{lem-GEN1} function $\varphi(z) \defeq\philoss'(z) + (1-\etanoise)$ and the proof of Lemma \ref{lem-PHI1} that $\philoss'(0) = - ({-\poibayesrisk'})^{-1}(0)$, so we remark the key identity:
\begin{eqnarray}
1-\etanoise-L & = & 1-\etanoise -  ({-\poibayesrisk'})^{-1}(0) = \philoss'(0) + 1-\etanoise = \varphi(0),
\end{eqnarray}
so the factor takes on two different signs in the \textbf{Blue} and \textbf{Red case} of Lemma \ref{lem-GEN1}. We thus distinguish two cases:\\

\noindent \textbf{Blue case}: we know (proof of Lemma \ref{lem-GEN1}) that $a>0, \varphi(0) > 0$ and we want $\texttt{e}_y > |\texttt{e}_x|$ under the constraint that both $y$ coordinates of the duplicated observations are negative: $(s-c)\gamma < 0$, so that the booster will pick the $y$ coordinate with a positive leveraging coefficient and thus will badly classify the duplicated examples of $\mathcal{S}_{\mbox{\tiny{clean}},\theta}$. We end up with the system (using $\gamma >0$)
\begin{eqnarray}
  \left\{
  \begin{array}{rcl}
    a\cdot c + s & > & \left|c - a\cdot s\right|,\\
    c - s & > & 0,\\
    c^2 + s^2 & = & 1.
    \end{array}
  \right. ,
  \end{eqnarray}
 which can be put in a vector form for graphical solving, letting $\ve{u} \defeq \left[\begin{array}{c}a\\ 1\end{array}\right], \ve{v} \defeq \left[\begin{array}{c}1\\ -a\end{array}\right]$ (note $\|\ve{u}\|_2 = \|\ve{v}\|_2$), $\ve{w} \defeq \left[\begin{array}{c}1\\ -1\end{array}\right]$ and $\ve{\theta} \defeq \left[\begin{array}{c}c\\ s\end{array}\right], $ the vector of unknowns (with a slight abuse of notation), yielding
  \begin{eqnarray}
    \left\{
  \begin{array}{rcl}
    \ve{\theta}^\top \ve{u} & > & |\ve{\theta}^\top \ve{v}|,\\
    \ve{\theta}^\top \ve{w} & > & 0,\\
    \|\ve{\theta}\|_2 & = & 1.
  \end{array}
                            \right.
  \end{eqnarray}
  Figure \ref{f-theta-blue} (left) presents the computation of solutions.
  \begin{figure}[t]
    \begin{center}
      \begin{tabular}{cc}
        \includegraphics[trim=30bp 500bp 660bp 20bp,clip,width=0.5\linewidth]{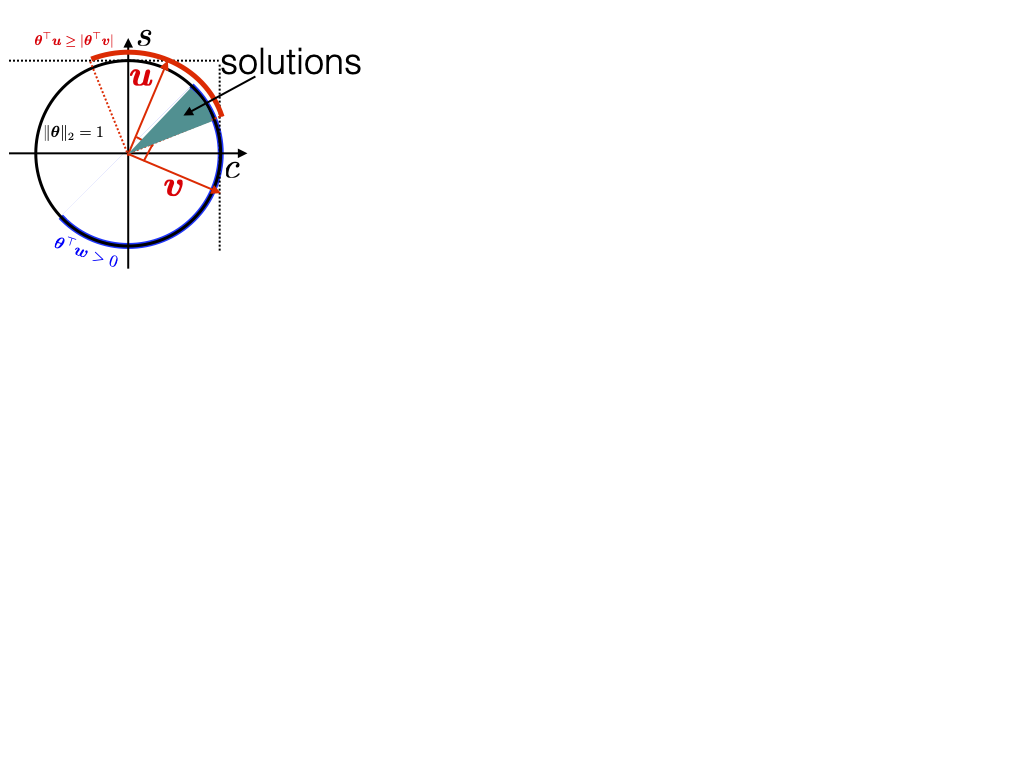} & \includegraphics[trim=30bp 500bp 660bp 20bp,clip,width=0.5\linewidth]{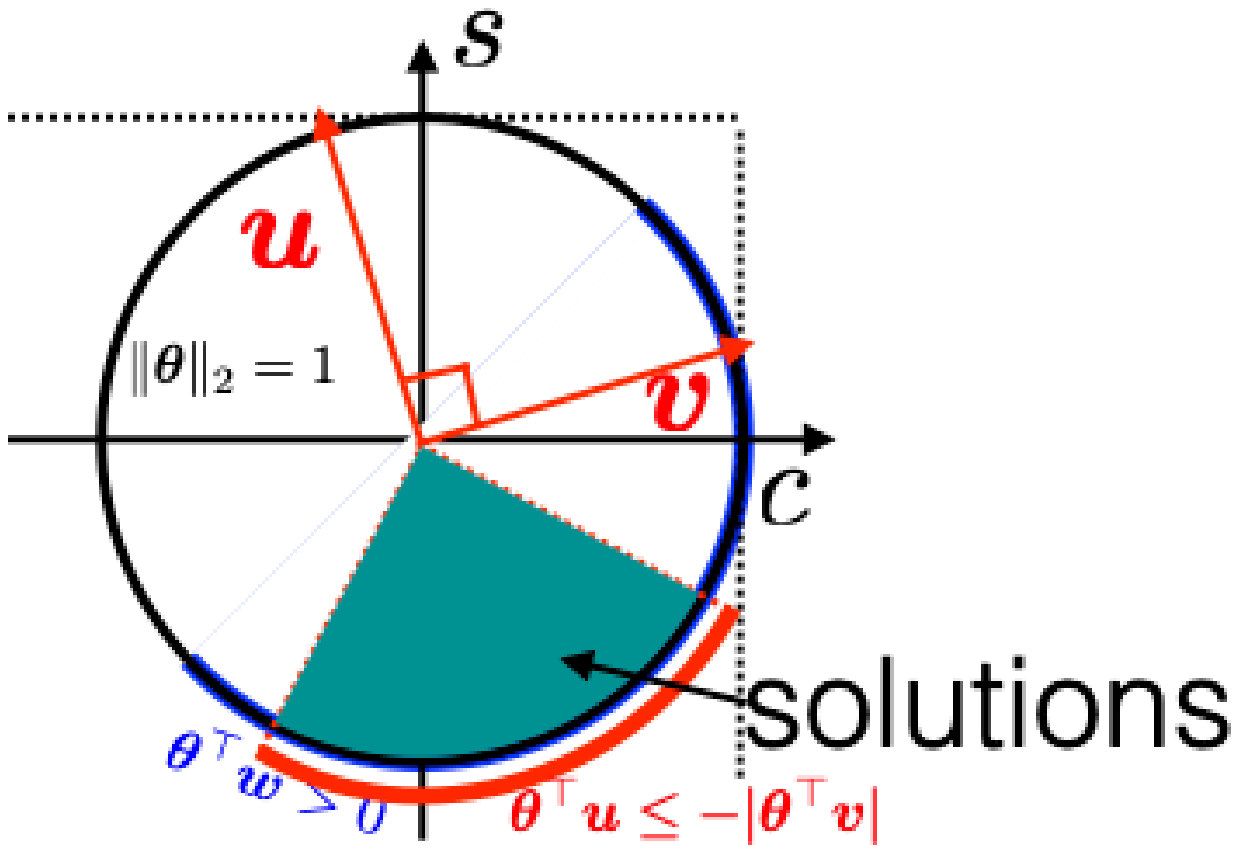}\\
        Blue case & Red case
      \end{tabular}
\end{center}
\caption{Solutions that trick the booster in picking a first update that misclassifies the examples in $\mathcal{S}_{\mbox{\tiny{clean}},\theta}$ sharing the same observation (the "penalizers" in \cite{lsRC}).}
  \label{f-theta-blue}
\end{figure}

\noindent \textbf{Red case}: we now have  (proof of Lemma \ref{lem-GEN1}) that $a<0, \varphi(0) < 0$. We now look to a solution to the following system:
\begin{eqnarray}
    \left\{
  \begin{array}{rcl}
    \ve{\theta}^\top \ve{u} & < & -|\ve{\theta}^\top \ve{v}|,\\
    \ve{\theta}^\top \ve{w} & > & 0,\\
    \|\ve{\theta}\|_2 & = & 1.
  \end{array}
                            \right.
  \end{eqnarray}
  Figure \ref{f-theta-blue} (right) presents the computation of solutions. The reason why it tricks again the booster in making at least $50\%$ error on its first update is that $\ve{\theta}^\top \ve{u} < 0$ and thus $\texttt{e}_y \propto \varphi(0) \cdot \ve{\theta}^\top \ve{u} > 0$ but also $|\texttt{e}_y| > |\texttt{e}_x| \propto |\varphi(0) \cdot \ve{\theta}^\top \ve{v}|$ and we check that because $\ve{\theta}^\top \ve{w} > 0$, the $y$ coordinate of the two examples sharing the same observation (the "penalizers" in \cite{lsRC}) is negative and so they are both misclassified.

  \begin{remark}\label{rem-not-red-blue}
The proof of Lemma \ref{lem-Rot} unveils what happens in the not-blue not-red case, when $\varphi(0) = 0$: in this case, the weak learner is totally "blind" as $\texttt{e}_x = \texttt{e}_y = 0$, so there is no possible update of the classifier as the weak learning assumption breaks down; the final classifier is thus the null vector = unbiased coin.
    \end{remark}

\subsection{A side negative result for \topdowngen~with \cls}\label{proof-sidenegative}

We can show an impeding result for \topdowngen~directly in the setting of Lemma \ref{lem-GEN1}: with the square loss (which allows to compute steps in closed form), \topdowngen~hits a classifier as bad as the fair coin on \cite{lsRC}'s noise-free data in \textit{at most 2 iterations only} for some values of the noise level and parameter $\gamma$ (there is thus no need to use the rotation argument of Lemma \ref{lem-Rot} for the booster to "fail"). 
 \begin{lemma}\label{lem-SQL-run}
Suppose \topdowngen~is run with the square loss to learn a linear separator on $\mathcal{S}_{\mbox{\tiny{noisy}}}$ and \weak~returns a scaled vector from the canonical basis of $\mathbb{R}^2$. Then there exists $N>1, 0<\gamma<1/6$ such that in at most \textbf{two} iterations, \topdowngen~hits a linear separator with 50$\%$ accuracy on $\mathcal{S}_{\mbox{\tiny{clean}}}$.
\end{lemma}
\begin{proof}
 We recall the key parameters of the square loss for some
 constant $L > 0$ (unnormalized):
 \begin{itemize}
 \item partial losses: $\partialloss{1}(u) = L(1-u)^2, \partialloss{-1}(u) = Lu^2$, pointwise bayes risk: $\poibayesrisk(u) = Lu(1-u)$, convex surrogate:
   \begin{eqnarray*}
     \philoss(z) & = & \left\{
                     \begin{array}{ccl}
                       -z & \mbox{ if } & z < -L\\
                       \frac{L}{4}\cdot\left(1 - \frac{z}{L}\right)^2 & \mbox{ if } & z \in [-L,L]\\
                       0 & \mbox{ if } & z > L
                       \end{array}
                     \right. .
   \end{eqnarray*}
 \item weight function $w(yH = z)$:
\begin{eqnarray*}
     w(z) & = & \left\{
                     \begin{array}{rcl}
                       1 & \mbox{ if } & z < -L\\
                       \frac{1}{2}\cdot\left(1 - \frac{z}{L}\right) & \mbox{ if } & z \in [-L,L]\\
                       0 & \mbox{ if } & z > L
                       \end{array}
                     \right. .
   \end{eqnarray*}
 \end{itemize}
 We recall and name the noisy examples, for $N>1$:
 \begin{itemize}
 \item $N$ copies of $(\gamma, 5\gamma)$ (call them $A$), $N$ copies of $(1,0)$ (call them $B$), $2N$ copies of $(\gamma, -\gamma)$ (call them $C$), all positive;
 \item $1$ copy of $(\gamma, 5\gamma)$ (call them $D$), $1$ copy of $(1,0)$ (call them $E$), $2$ copies of $(\gamma, -\gamma)$ (call them $F$), all negative;  
 \end{itemize}
 We have two cases:\\
 \noindent \textbf{Case 1}: suppose the first vector output by \weak~is proportional to $(1,0)$.\\
 \noindent \underline{Iteration 1}: \weak~returns vector $h_1 = (U,0)$ for $U>0$, which labels correctly all observations in the noise-free case. The weights all equal $w(0) = 1/2$. The edge of $h_1$ is (we note $\max h_1 = U$):
 \begin{eqnarray*}
   \texttt{e}_1(h_1) \defeq \left|\sum_{i\in [m]} \frac{w_{t,i}}{\sum_{j\in [m]} w_{t,j}} \cdot y^*_i \cdot \frac{h_{t}(\ve{x}_i)}{\max_{j\in [m]}|h_t(\ve{x}_j)|}\right| & = & \frac{\frac{N\gamma}{2} + \frac{N}{2} + \frac{2N\gamma}{2} - \frac{\gamma}{2}- \frac{1}{2} - \frac{2\gamma}{2}}{2(N+1)}\\
   & = & \frac{(1+3\gamma)}{4} \cdot \frac{N-1}{N+1}.
 \end{eqnarray*}
 The leveraging coefficient for $h_1$, $\alpha_1$, is the solution of
\begin{eqnarray*}
  \underbrace{N\left(\frac{1}{2} - \frac{U\gamma \alpha_1}{L}\right)U\gamma}_{A} + \underbrace{N\left(\frac{1}{2} - \frac{U\alpha_1}{L}\right)U}_{B} + \underbrace{2N\left(\frac{1}{2} - \frac{U\gamma \alpha_1}{L}\right)U\gamma}_{C}\\
 - \underbrace{\left(\frac{1}{2} + \frac{U\gamma \alpha_1}{L}\right)U\gamma}_{D}  -\underbrace{\left(\frac{1}{2} + \frac{U\alpha_1}{L}\right)U}_{E}  - \underbrace{2\left(\frac{1}{2} + \frac{U\gamma \alpha_1}{L}\right)U\gamma}_{F} & = & 0,
\end{eqnarray*}
giving
 \begin{eqnarray*}
\alpha_1 & = & \frac{2L U\left(\frac{N\gamma}{2} + \frac{N}{2} + \frac{2N\gamma}{2} - \frac{\gamma}{2}- \frac{1}{2} - \frac{2\gamma}{2}\right)}{U^2\left(\frac{(N+1)\gamma^2}{2} + \frac{N+1}{2} + \frac{2(N+1)\gamma^2}{2}\right)} = \frac{2 L (N-1)(1+3\gamma)}{U(N+1)(1+3\gamma^2)}.
 \end{eqnarray*}
We compute the new weight, with notation simplified to $w_2(.)$. For $A, C, D, F$, we remark 
\begin{eqnarray*}
\frac{|\alpha_1h_1|}{L} & = & \frac{2 \gamma (N-1)(1+3\gamma)}{(N+1)(1+3\gamma^2)} < \frac{1}{2},
\end{eqnarray*}
and so
\begin{eqnarray*}
w_2(A) = w_2(C) & = & \frac{1}{2} -  \frac{\gamma (N-1)(1+3\gamma)}{(N+1)(1+3\gamma^2)} \defeq \frac{1}{2} -  \gamma k_2,\\
  w_2(D) = w_2(F) & = & \frac{1}{2} +  \frac{\gamma (N-1)(1+3\gamma)}{(N+1)(1+3\gamma^2)} \defeq \frac{1}{2} + \gamma k_2,
\end{eqnarray*}
with
\begin{eqnarray*}
k_2 & \defeq & \frac{ (N-1)(1+3\gamma)}{(N+1)(1+3\gamma^2)};
  \end{eqnarray*}
while for $B, E$, we have
\begin{eqnarray*}
\frac{|\alpha_1h_1|}{L} & = & 2k_2.
\end{eqnarray*}
This, together with the fact that
\begin{eqnarray}
\frac{3+6\gamma + 3\gamma^2}{1+6\gamma - 3\gamma^2} \in [2,3], \forall \gamma \in [0,1],
\end{eqnarray}
yields that if $N>3$, then $w_2(B) = 0$, $w_2(E) = 1$ using the extreme expressions of the weight function. Let us assume $N \in \{2,3\}$ to prevent this from happening (this simplifies derivations), so that 
\begin{eqnarray*}
w_2(B) & = & \frac{1}{2} -  k_2,\\
  w_2(E) & = & \frac{1}{2} +  k_2 .
\end{eqnarray*}
\noindent \underline{Iteration 2}: suppose \weak~returns vector $h_2 = (0,U)$ for $U>0$. We note this time $\max h_2 = 5\gamma U$ and the edge is now
   \begin{eqnarray*}
     \texttt{e}_2(h_2) & = & \left| U\cdot \frac{5N\gamma w_2(A) - 2N \gamma w_2(C) - 5N\gamma w_2(D) + 2N\gamma w_2(F)}{20 \gamma N U} \right|\\
                       & = & \left| \frac{-  \frac{10N \gamma^2 (N-1)(1+3\gamma)}{(N+1)(1+3\gamma^2)} + \frac{4N\gamma^2 (N-1)(1+3\gamma)}{(N+1)(1+3\gamma^2)}}{20 \gamma N} \right|\\
     & = & \frac{3 \gamma k_2}{10}.
   \end{eqnarray*}
    The leveraging coefficient for $h_2$, $\alpha_2$, is the solution of
\begin{eqnarray*}
  \underbrace{N\left(\frac{1}{2} - \gamma k_2 - \frac{5U\gamma \alpha_2}{2L}\right)5U\gamma}_{A} + \underbrace{N\left(\frac{1}{2} - k_2\right)\cdot 0}_{B} - \underbrace{2N\left(\frac{1}{2} - \gamma k_2 + \frac{U\gamma \alpha_2}{2L} \right)\cdot (U\gamma)}_{C}\\
 - \underbrace{\left(\frac{1}{2} + \gamma k_2 + \frac{5U\gamma \alpha_2}{2L}\right)5U\gamma}_{D}  -\underbrace{\left(\frac{1}{2} - k_2\right)\cdot 0}_{E}  + \underbrace{2\left(\frac{1}{2} + \gamma k_2 - \frac{U\gamma \alpha_2}{2L} \right)\cdot (U\gamma)}_{F} & = & 0,
\end{eqnarray*}
giving
\begin{eqnarray*}
\alpha_2 & = & \frac{3L(N-1)(1-2\gamma-3\gamma^2)}{27U(N+1)\gamma (1+3\gamma^2)}.
\end{eqnarray*}
We check the new weights for $A, C, D, F$ (others do not change). We remark for $A, D$:
\begin{eqnarray*}
  \frac{|\alpha_1h_1 + \alpha_2 h_2|}{L} & = & \frac{2 \gamma (N-1)(1+3\gamma)}{(N+1)(1+3\gamma^2)} + \frac{15(N-1)(1-2\gamma-3\gamma^2)}{27(N+1) (1+3\gamma^2)}\\
                                         & = & \frac{N-1}{N+1} \cdot \frac{5 +8\gamma +39\gamma^2}{9(1+3\gamma^2)}\\
                                         & \leq & \frac{5 +8\gamma +39\gamma^2}{18(1+3\gamma^2)} \quad (N \in \{2,3\})\\
  & \leq & \frac{1}{2} \quad (\gamma \leq 1/3),
\end{eqnarray*}
while for $C, F$:
\begin{eqnarray*}
  \frac{|\alpha_1h_1 + \alpha_2 h_2|}{L} & = & \left|\frac{2 \gamma (N-1)(1+3\gamma)}{(N+1)(1+3\gamma^2)} - \frac{(N-1)(1-2\gamma-3\gamma^2)}{9(N+1) (1+3\gamma^2)}\right|\\
                                         & = & \frac{N-1}{N+1} \cdot \frac{|-1 +20\gamma +57\gamma^2|}{9(1+3\gamma^2)}\\
                                         & \leq &\frac{|-1 +20\gamma +57\gamma^2|}{18(1+3\gamma^2)}\quad (N \in \{2,3\})\\
                                                    & \leq & \frac{1}{2} \quad (\gamma \leq 1/3),
\end{eqnarray*}
so all the new weights are given not by the "extreme" formulas of the weight function. We check the vector $\ve{\theta}_2$ learned after two iterations:
\begin{eqnarray*}
  \ve{\theta}_2 & = &\left[
                      \begin{array}{c}
                        \frac{2 L (N-1)(1+3\gamma)}{(N+1)(1+3\gamma^2)}\\
                         \frac{L(N-1)(1-2\gamma-3\gamma^2)}{9(N+1)\gamma (1+3\gamma^2)}
                        \end{array}
  \right]\\
  & = & \frac{L(N-1)}{9(N+1)\gamma (1+3\gamma^2)} \cdot \left[
                      \begin{array}{c}
                        18\gamma(1+3\gamma)\\
                         (1-2\gamma-3\gamma^2)
                        \end{array}
  \right],
  \end{eqnarray*}
  and we check that if $\gamma \leq 1/23$, then $\ve{\theta}_2$ misclassifies both positive examples with observation $(\gamma, -\gamma)$ (called the "penalizers" in \cite{lsRC}) in the noise-free dataset, thereby having $50\%$ accuracy.

  \noindent \textbf{Case 2}: suppose the first vector output by \weak~is proportional to $(0,1)$.\\
 \noindent \underline{Iteration 1}: \weak~returns vector $h_1 = (0,U)$ for $U>0$. We note $\max h_1 = 5\gamma U$ and the edge is
   \begin{eqnarray*}
     \texttt{e}_1(h_1) & = & U\cdot \frac{\frac{5\gamma UN}{2} + \frac{2\gamma UN}{2} - \frac{5\gamma U}{2} - \frac{2\gamma U}{2}}{20 \gamma (N+1) U} \\
                       & = & \frac{7(N-1)}{40(N+1)}.
   \end{eqnarray*}
   The leveraging coefficient for $h_1$, $\alpha_1$, is the solution of
\begin{eqnarray*}
  \underbrace{N\left(\frac{1}{2} - \frac{5U\gamma \alpha_1}{2L}\right)5U\gamma}_{A} + \underbrace{0}_{B} - \underbrace{2N\left(\frac{1}{2} + \frac{U\gamma \alpha_1}{2L} \right)\cdot (U\gamma)}_{C}\\
 - \underbrace{\left(\frac{1}{2} + \frac{5U\gamma \alpha_2}{2L}\right)5U\gamma}_{D}  +\underbrace{0}_{E}  + \underbrace{2\left(\frac{1}{2} - \frac{U\gamma \alpha_2}{2L} \right)\cdot (U\gamma)}_{F} & = & 0,
\end{eqnarray*}
giving
\begin{eqnarray*}
\alpha_1 & = & \frac{L(N-1)}{9U(N+1)\gamma},
  \end{eqnarray*}
and we check the new weights for $A, C, D, F$ (others do not change); we remark for $A, D$:
\begin{eqnarray*}
  \frac{|\alpha_1h_1|}{L} & = & \frac{5(N-1)}{9(N+1)} < \frac{1}{2},
\end{eqnarray*}
while for $C, F$:
 \begin{eqnarray*}
  \frac{|\alpha_1h_1|}{L} & = & \frac{(N-1)}{9(N+1)} < \frac{1}{2},
 \end{eqnarray*}
 so after the first iteration, the vector $\ve{\theta}_1$ learned,
\begin{eqnarray*}
  \ve{\theta}_1 & = &\left[
                      \begin{array}{c}
                        0\\
                         \frac{L(N-1)}{9(N+1)\gamma}
                        \end{array}
  \right],
\end{eqnarray*}
misclassifies again both positive examples with observation $(\gamma, -\gamma)$ (called the "penalizers" in \cite{lsRC}) in the noise-free dataset, thereby having $50\%$ accuracy.
\end{proof}
\begin{remark}
  Remark that the edge substantially decreases between two iterations in Case 1 as:
  \begin{eqnarray}
\frac{\texttt{e}_2(h_2)}{\texttt{e}_1(h_1)} & = & \frac{6\gamma}{5(1+3\gamma^2)},
  \end{eqnarray}
  which indicates that if run for longer, the weak learning assumption will eventually end up being rapidly violated in \topdowngen, preventing the application of Theorem \ref{th-boost-pbls}.
  \end{remark}

\subsection{Proof of Lemma \ref{solALPHA}}\label{proof-lem-solALPHA}

The proof relies on five key observations (assuming wlog $h_t$ does not zero over $\mathcal{S}_t$):
\begin{enumerate}
\item [(1)] The equation can be written with $\alpha_t$ explicit as $\sum_{i\in [m]_{t}} (y_i - y_i^* ({-\poibayesrisk'})^{-1}(H_{t-1}(\ve{x}_i) + \alpha_{t} \cdot h_{t}(\ve{x}_i))) \cdot y_i^* h_{t}(\ve{x}_i)  =  0$, that is (since $(y_i^*)^2 = 1$),
  \begin{eqnarray}
\underbrace{\sum_{i\in [m]_{t}} ({-\poibayesrisk'})^{-1}(H_{t-1}(\ve{x}_i) + \alpha_{t} \cdot h_{t}(\ve{x}_i))\cdot h_{t}(\ve{x}_i)}_{\defeq J_t(\alpha_t)} & = & \sum_{i\in [m]_{t}, y^*_i = 1} h_{t}(\ve{x}_i). \label{alphaweightupdate}
  \end{eqnarray}
  \item [(2)] $\lim_{\alpha \nearrow} J_t(\alpha) = \sum_{i\in [m]_{t}, h_{t}(\ve{x}_i) > 0} h_{t}(\ve{x}_i) \defeq J_+$;
  \item [(3)] $\lim_{\alpha \searrow} J_t(\alpha) = \sum_{i\in [m]_{t}, h_{t}(\ve{x}_i) < 0} h_{t}(\ve{x}_i) \defeq J_-$;
  \item [(4)] $\sum_{i\in [m]_{t}, y^*_i = 1} h_{t}(\ve{x}_i) \in [J_-, J_+]$,
  \item [(5)] $\mathrm{Im} ({-\poibayesrisk'})^{-1} = [0,1]$, since if there was an interval of non-zero measure missing then either $\poibayesrisk'$ would not be defined over such an interval (impossible by the differentiability assumption) or it would be constant (impossible given the strict properness condition). The same remarks for a single missing value;
  \end{enumerate}
  which gives the statement of the Lemma. To get rid of infinite values, we remark that this happens only when $\sum_{i\in [m]_{t}, y^*_i = 1} h_{t}(\ve{x}_i) = \sum_{i\in [m]_{t}, h_{t}(\ve{x}_i) < 0} h_{t}(\ve{x}_i) $ ($-h_t$ makes perfect classification over $\mathcal{S}_t$) or $\sum_{i\in [m]_{t}, y^*_i = 1} h_{t}(\ve{x}_i) = \sum_{i\in [m]_{t}, h_{t}(\ve{x}_i) > 0} h_{t}(\ve{x}_i) $ ($h_t$ makes perfect classification over $\mathcal{S}_t$), both of which are not possible. 

\subsection{Proof of Theorem \ref{th-boost-pbls}}\label{proof-th-boost-pbls}

We proceed in several steps. The first shows a general guarantee on the decrease of the surrogate risk.
\begin{lemma}\label{lem-BOO1}
    Let $D_F$ denote the Bregman divergence with (convex) generator $F$. The difference between two successive surrogate risks in \topdowngen~satisfies:
    \begin{eqnarray}
      \popsur(H_{t}, \mathcal{S})-\popsur(H_{t-1}, \mathcal{S}) & = & - p_{t} \cdot \expect_{i\sim [m]_{t}}\left[\left\{
                                                                      \begin{array}{ccl}
                                                                        D_{-\poibayesrisk}(w_{t+1,i}\|w_{t,i}) & \mbox{ if } & y_i = 0\\
                                                                        D_{-\poibayesrisk}(1 - w_{t+1,i}\|1 - w_{t,i}) & \mbox{ if } & y_i = 1 
                                                                      \end{array}
                                                                                                                                       \right.\right],
    \end{eqnarray}
    where $[m]_{t} \subseteq [m]$ is the subset of indices of examples "fed" to the weak learner in $\mathcal{S}_t$ and $p_t \defeq \mathrm{Card}([m]_{t})/m$.
\end{lemma}
\begin{proof}
  We observe
\begin{eqnarray}
  \lefteqn{\popsur(H_{t}, \mathcal{S})-\popsur(H_{t-1}, \mathcal{S})}\nonumber\\
  & = & p_{t} \cdot \expect_{i\sim [m]_{t}}\left[\philoss(-H_{t}(\ve{x}_i)) - \philoss(-H_{t-1}(\ve{x}_i)) - y_i (H_{t}-H_{t-1})(\ve{x}_i)\right].
\end{eqnarray}
For any example $(\ve{x},y)$, if $y=0$, by the definition of Bregman divergences and their dual symmetry property,
\begin{eqnarray*}
  \lefteqn{\philoss(-H_{t}(\ve{x})) - \philoss(-H_{t-1}(\ve{x}))}\nonumber\\
  & = & (-\poibayesrisk)^\star(H_{t}(\ve{x})) - (-\poibayesrisk)^\star(H_{t-1}(\ve{x})) \\
  & = & - \left[(-\poibayesrisk)^\star(H_{t-1}(\ve{x})) - (-\poibayesrisk)^\star(H_{t}(\ve{x})) - (H_{t-1}-H_{t})(\ve{x}) \cdot ({-\poibayesrisk'})^{-1}(H_{t}(\ve{x})) \right]\\
  & & - (H_{t-1}-H_{t})(\ve{x}) \cdot ({-\poibayesrisk'})^{-1}(H_{t}(\ve{x}))\nonumber\\
  & = & -D_{(-\poibayesrisk)^\star}(H_{t-1}(\ve{x})\|H_{t}(\ve{x})) - (H_{t-1}-H_{t})(\ve{x}) \cdot ({-\poibayesrisk'})^{-1}(H_{t}(\ve{x}))\nonumber\\
  & = & - D_{-\poibayesrisk}(({-\poibayesrisk'})^{-1}(H_{t}(\ve{x}))\|({-\poibayesrisk'})^{-1}(H_{t-1}(\ve{x}))) - (H_{t-1}-H_{t})(\ve{x}) \cdot ({-\poibayesrisk'})^{-1}(H_{t}(\ve{x}))\nonumber\\
  & = & - D_{-\poibayesrisk}(w((\ve{x},y),H_{t})\|w((\ve{x},y),H_{t-1})) + \alpha_{t} \cdot w((\ve{x},y),H_{t}) h_{t}(\ve{x})\nonumber\\
  & = & - D_{-\poibayesrisk}(w((\ve{x},y),H_{t})\|w((\ve{x},y),H_{t-1})) - \alpha_{t} \cdot w((\ve{x},y),H_{t}) \cdot y^* h_{t}(\ve{x}),
  \end{eqnarray*}
  If $y=1$ (we do not replace $y$ by 1 to mark its locations),
  \begin{eqnarray*}
  \lefteqn{\philoss(-H_{t}(\ve{x})) - \philoss(-H_{t-1}(\ve{x}))- y (H_{t}-H_{t-1})(\ve{x})}\nonumber\\
    & = & (-\poibayesrisk)^\star(H_{t}(\ve{x})) - (-\poibayesrisk)^\star(H_{t-1}(\ve{x})) - y (H_{t}-H_{t-1})(\ve{x})\nonumber\\
    & = & - D_{-\poibayesrisk}(({-\poibayesrisk'})^{-1}(H_{t}(\ve{x}))\|({-\poibayesrisk'})^{-1}(H_{t-1}(\ve{x}))) - (H_{t-1}-H_{t})(\ve{x}) \cdot ({-\poibayesrisk'})^{-1}(H_{t}(\ve{x}))\nonumber\\
    & & - y (H_{t}-H_{t-1})(\ve{x})\nonumber\\
    & = &   - D_{-\poibayesrisk}(({-\poibayesrisk'})^{-1}(H_{t}(\ve{x}))\|({-\poibayesrisk'})^{-1}(H_{t-1}(\ve{x}))) + (({-\poibayesrisk'})^{-1}(H_{t}(\ve{x}))-y) \cdot \alpha_{t} h_{t}(\ve{x})\nonumber\\
    & = &   - D_{-\poibayesrisk}(({-\poibayesrisk'})^{-1}(H_{t}(\ve{x}))\|({-\poibayesrisk'})^{-1}(H_{t-1}(\ve{x}))) - (y-({-\poibayesrisk'})^{-1}(H_{t}(\ve{x}))) \cdot \alpha_{t} \cdot y^* h_{t}(\ve{x})\nonumber\\
    & = &   - D_{-\poibayesrisk}(y - w((\ve{x},y),H_{t})\|y - w((\ve{x},y),H_{t-1})) - \alpha_{t} \cdot w((\ve{x},y),H_{t}) \cdot y^* h_{t}(\ve{x}),
  \end{eqnarray*}
  and thus, we get for \topdowngen~the relationship between successive surrogate risks,
  \begin{eqnarray}
  \lefteqn{\popsur(H_{t}, \mathcal{S})-\popsur(H_{t-1}, \mathcal{S})}\nonumber\\
  & = & - p_{t} \cdot \expect_{i\sim [m]_{t}}\left[\left\{
                                                                      \begin{array}{ccl}
                                                                        D_{-\poibayesrisk}(w_{t+1,i}\|w_{t,i}) & \mbox{ if } & y_i = 0\\
                                                                        D_{-\poibayesrisk}(1 - w_{t+1,i}\|1 - w_{t,i}) & \mbox{ if } & y_i = 1 
                                                                      \end{array}
                                                                                                                                       \right.\right] - \alpha_{t} \cdot \sum_{i\in [m]_{t}} w_{i,t+1} \cdot y_i^* h_{t}(\ve{x}_i)\nonumber\\
    & = & - p_{t} \cdot \expect_{i\sim [m]_{t}}\left[\left\{
                                                                      \begin{array}{ccl}
                                                                        D_{-\poibayesrisk}(w_{t+1,i}\|w_{t,i}) & \mbox{ if } & y_i = 0\\
                                                                        D_{-\poibayesrisk}(1 - w_{t+1,i}\|1 - w_{t,i}) & \mbox{ if } & y_i = 1 
                                                                      \end{array}
                                                                                                                                       \right.\right],
  \end{eqnarray}
  by \eqref{defAlpha}.
\end{proof}
The following Theorem established a general boosting-compliant convergence bound, the central piece of our proof.
\begin{theorem}\label{th-boost-1}
  Define the expected and normalized weights at iteration $t$:
  \begin{eqnarray}
\overline{w^*}_{t} \defeq \frac{\sum_{i\in [m]_t} w_{t,i}}{\mathrm{Card}([m]_t)} & ; & w^{\mbox{\tiny{norm}}}_{t,i} \defeq \frac{w_{t,i}}{\sum_{j\in [m]_t} w_{t,j}},
  \end{eqnarray}
  and the following two assumptions (\textbf{LOSS0}, \textbf{WLA}):
  \begin{enumerate}
    \item [\textbf{LOSS0}] The loss chosen $\loss$ is strictly proper, differentiable and satisfies $\inf \{\partialloss{-1}'-\partialloss{1}' \}\geq \kappa$ for some $\kappa > 0$;
    \item [\textbf{WLA}] There exists a constant $\gammawla>0$ such that at each iteration $t\in [T]$, the weak hypothesis $h_t$ returned by $\weak$ satisfies
  \begin{eqnarray}
\left|\sum_{i\in [m]_t} w^{\mbox{\tiny{norm}}}_{t,i} \cdot y^*_i \cdot \frac{h_{t}(\ve{x}_i)}{\max_{i\in [m]_t}|h_t(\ve{x}_i)|}\right| & \geq & \gammawla.\label{defWLA2}
  \end{eqnarray}
\end{enumerate}
Then under \textbf{LOSS0} and \textbf{WLA} the following holds:
\begin{eqnarray}
\forall \Phi \in \mathbb{R}, \left(\sum_{t=1}^T p_t \overline{w^*}^2_{t} \geq \frac{2(\popsur(H_{0}, \mathcal{S}) - \Phi)}{\kappa \gamma^2}\right) & \Rightarrow & \left(\popsur(H_{T}, \mathcal{S}) \leq \Phi\right).\label{eqWithWTmain}
\end{eqnarray}
  \end{theorem}
\begin{proof}
  Assuming second-order differentiability, we have the classical Taylor approximation of Bregman divergences \cite[Appendix II]{nmSL}: for $t \in [T], i\in [m]_t$, $\exists u_{t,i}, v_{t,i} \in [0,1]$ such that:
  \begin{eqnarray*}
    D_{-\poibayesrisk}(w_{t+1,i}\|w_{t,i}) & = & \frac{(-\poibayesrisk)''(u_{t,i}) (w_{t+1,i} - w_{t,i})^2 }{2},\\
    D_{-\poibayesrisk}(1-w_{t+1,i}\|1-w_{t,i}) & = & \frac{(-\poibayesrisk)''(v_{t,i}) (w_{t+1,i} - w_{t,i})^2 }{2},
  \end{eqnarray*}
It follows from \cite[Lemma 12]{snsBP} and assumption \textbf{LOSS0} that $\loss$ being strictly proper, we have $(-\poibayesrisk)'' = \partialloss{-1}'-\partialloss{1}' \geq \inf \{\partialloss{-1}'-\partialloss{1}' \}\geq \kappa$, so we get
\begin{eqnarray*}
p_{t} \cdot \expect_{i\sim [m]_{t}}\left[\left\{
                                                                      \begin{array}{ccl}
                                                                        D_{-\poibayesrisk}(w_{t+1,i}\|w_{t,i}) & \mbox{ if } & y_i = 0\\
                                                                        D_{-\poibayesrisk}(1 - w_{t+1,i}\|1 - w_{t,i}) & \mbox{ if } & y_i = 1 
                                                                      \end{array}
                                                                                                                                       \right.\right] & \geq & \frac{p_{t}\kappa}{2}\cdot \expect_{i\sim [m]_{t}}\left[(w_{t+1,i} - w_{t,i})^2\right],
\end{eqnarray*}
We remark
\begin{eqnarray}
  \left(\sum_{i\in [m]_t} w_{t,i} \cdot y^*_i h_{t}(\ve{x}_i)\right)^2 & = & \left(\sum_{i\in [m]_t} w_{t,i} \cdot y^*_i h_{t}(\ve{x}_i) - \sum_{i\in [m]_t} w_{t+1,i} \cdot y^*_i h_{t}(\ve{x}_i)\right)^2\\
                                                                         & = & \left(\sum_{i\in [m]_t} (w_{t,i} - w_{t+1,i}) \cdot y^*_i h_{t}(\ve{x}_i) \right)^2\\
                                                                         & \leq & \sum_{i\in [m]_t} (w_{t,i} - w_{t+1,i})^2 \cdot \sum_{i\in [m]_t} h^2_{t}(\ve{x}_i) \\
  & \leq & \mathrm{Card}([m]_t)^2 M_t^2 \cdot \expect_{i\sim [m]_t}\left[(w_{t+1,i} - w_{t,i})^2\right],
\end{eqnarray}
where we have used \eqref{defAlpha}, Cauchy-Schwartz, the assumption that the distribution is uniform and let $M_t \defeq \max_{i\in [m]_t}|h(\ve{x}_i)|$. Thus,
\begin{eqnarray}
\lefteqn{p_{t} \cdot \expect_{i\sim [m]_{t}}\left[\left\{
                                                                      \begin{array}{ccl}
                                                                        D_{-\poibayesrisk}(w_{t+1,i}\|w_{t,i}) & \mbox{ if } & y_i = 0\\
                                                                        D_{-\poibayesrisk}(1 - w_{t+1,i}\|1 - w_{t,i}) & \mbox{ if } & y_i = 1 
                                                                      \end{array}
                                                                                                                                       \right.\right]}\nonumber\\
                                                                                                               & \geq & \frac{p_{t}\kappa}{2}\cdot \frac{\left(\sum_{i\in [m]_t} w_{t,i} \cdot y^*_i \cdot \frac{h_{t}(\ve{x}_i)}{M}\right)^2}{\mathrm{Card}([m]_t)^2}\nonumber\\
  & & =  \frac{p_{t}\kappa}{2}\cdot
      \underbrace{\left(\frac{\sum_{i\in [m]_t}
      w_{t,i}}{\mathrm{Card}([m]_t)}\right)^2}_{\defeq \overline{w^*}^2_t}
      \cdot \underbrace{\left(\sum_{i\in [m]_t}
      \frac{w_{t,i}}{\sum_{j\in [m]_t} w_{t,j}} \cdot y^*_i \cdot
      \frac{h_{t}(\ve{x}_i)}{M_t}\right)^2}_{\geq \upgamma^2\mbox{ from \textbf{WLA}}}.
\end{eqnarray}
Assumptions \textbf{LOSS0} and \textbf{WLA} thus imply the guaranteed decrease between two successive risks
\begin{eqnarray}
\popsur(H_{t}, \mathcal{S}) & \leq & \popsur(H_{t-1}, \mathcal{S}) - \frac{p_{t}\kappa \gamma^2 \overline{w^*}^2_{t}}{2},
\end{eqnarray}
and we have, after collapsing summing inequalities for $t=1, 2, ..., T$, the guarantee that as long as the WLA holds,
\begin{eqnarray}
\forall \Phi \in \mathbb{R}, \left(\sum_{t=1}^T p_t \overline{w^*}^2_{t} \geq \frac{2(\popsur(H_{0}, \mathcal{S}) - \Phi)}{\kappa \gamma^2}\right) & \Rightarrow & \left(\popsur(H_{T}, \mathcal{S}) \leq \Phi\right),\label{eqWithWT}
\end{eqnarray}
which is the statement of Theorem \ref{th-boost-1}.
\end{proof}
Because it involves $\overline{w^*}_{t}$, this bound is not fully
readable, but there is a simple way to remove its dependence as
$\overline{w^*}_{t}$ is also linked to the quality of the classifier
$H$: roughly speaking, the smaller it is, the worse is the dependence
in \eqref{eqWithWT} \textit{but} the better is $H$ since weights tends
to decrease as $H$ gives the right class with increased confidence
($|H|$). The trick is thus to find a value of $\overline{w^*}_{t}$ below
which $H$ is "satisfying" (boosting-wise) and then plug this bound in \eqref{eqWithWT}, which then gives a number of iterations after which $H$ becomes satisfying anyway.

We need the following definition.
  \begin{definition}(after \cite{bcsEN})
    Weights at iteration $t$ are called $\zeta$-dense if $\overline{w}_t \geq \zeta$, where
    \begin{eqnarray*}
\overline{w}_t & \defeq & \frac{\sum_{i\in [m]} w_{t,i}}{m}.
      \end{eqnarray*}
    \end{definition}
Notice that the expected weight here, $\overline{w}_t$, spans all the training sample, which is \textbf{not} the case for $\overline{w^*}_{t}$ (which relies on the examples "fed" to the weak learner). 
  We make precise the notion of being "satisfying" when weights are "small".
  \begin{lemma}\label{lemEDGE1}
    For any $t\geq 1, \zeta \in [0,1]$, suppose the weights at iteration $t+1$ are not $\zeta$-dense. Then $\forall \theta \in \mathbb{R}$,
    \begin{eqnarray}
\pr_{i\sim [m]}[y^*_i H_{t}(\ve{x}_i) \leq \theta] & < & \frac{\zeta}{\underline{w}(\theta)},
    \end{eqnarray}
    where we let $\underline{w}(\theta) \defeq \min\{1 - ({-\poibayesrisk'})^{-1}(\theta), ({-\poibayesrisk'})^{-1}(-\theta)\}$.
    \end{lemma}
  \begin{proof}
    We denote $[m]_{+}$ the set of indices whose examples have positive class. Let $z^+_1, z^+_2, ..., z^+_{\mathrm{Card}([m]_{+})}$ some associated reals and $w_+(z) \defeq 1 - ({-\poibayesrisk'})^{-1}(z)$ the positive examples' weight function. Being non-increasing and with range in $[0,1]$, we have $\forall \theta \in \mathbb{R}$
    \begin{eqnarray}
      \expect_{i\sim [m]_{+}}[w_+(z^+_i)] & \geq & \pr_{[m]_{+}} [z^+_i \leq \theta] \cdot w_+(\theta) +   \pr_{[m]_{+}} [z^+_i > \theta]   \cdot \inf w_+ \\
      & & = \pr_{[m]_{+}} [z^+_i \leq \theta] \cdot w_+(\theta),
    \end{eqnarray}
    so using this with $z^+_i  = H_{t}(\ve{x}_i) = y^*_i H_{t}(\ve{x}_i)$ yields $w_+(z^+_i) = w_{t+1,i}$ and $\expect_{i\sim [m]_{+}}[w_{t+1,i}] \geq \pr_{[m]_{+}} [y^*_i H_{t}(\ve{x}_i)\leq \theta]\cdot  w_+(\theta)$.\\
    
    Similarly, let $[m]_{-}$ the set of negative indices for iteration $t$. Let $z^-_1, z^-_2, ..., z^-_{\mathrm{Card}([m]_{-})}$ some associated reals and $w_-(z) \defeq ({-\poibayesrisk'})^{-1}(z)$ the negative examples' weight function. Being non-decreasing and with range in $[0,1]$, we have $\forall \theta \in \mathbb{R}$
    \begin{eqnarray}
      \expect_{i\sim [m]_{-}}[w_-(z^-_i)] & \geq & \pr_{[m]_{-}} [z^-_i \geq -\theta] \cdot w_-(-\theta) -   \pr_{[m]_{-}} [z^-_i < -\theta]   \cdot \inf w_- \\
      & & = \pr_{[m]_{-}} [-z^-_i \leq \theta] \cdot w_-(-\theta),
    \end{eqnarray}
so using this with $z^-_i  = H_{t}(\ve{x}_i) = -y^*_i H_{t}(\ve{x}_i)$ yields $w_+(z^-_i) = w_{t+1,i}$ and $\expect_{i\sim [m]_{-}}[w_{t+1,i}] \geq \pr_{[m]_{-}} [y^*_i H_{t}(\ve{x}_i) \leq \theta] \cdot w_-(-\theta)$.\\
    
Denote $c(i) \in \{+,-\}$ the label of index $i$ in $[m]$ and $p^+, p^-$ the proportion of positive and negative examples in $[m]$. With a slight abuse of notation in indices, we have $p^+ \expect_{i\sim [m]_{+}}[w_+(z^+_i)] + p^- \expect_{i\sim [m]_{-}}[w_-(z^-_i)] = \expect_{i\sim [m]_{}}[w_{c(i)}(z^{c(i)}_i)] = \overline{w}_{t+1}$ where the last identity holds for the choices of the $z_i^\bullet$s made above. We thus have the lower-bound on $\overline{w}_{t+1}$:
    \begin{eqnarray}
      \overline{w}_{t+1} & \geq & p^+ \pr_{[m]_{+}} [y^*_i H_{t}(\ve{x}_i) \leq \theta] \cdot w_+(\theta) + p^- \pr_{[m]_{-}} [y^*_i H_{t}(\ve{x}_i) \leq \theta] \cdot w_-(-\theta)\\
          & \geq & (p^+ \pr_{[m]_{+}} [y^*_i H_{t}(\ve{x}_i) \leq \theta]  + p^- \pr_{[m]_{-}} [y^*_i H_{t}(\ve{x}_i) \leq \theta] ) \cdot \min\{w_+(\theta), w_-(-\theta)\}\\
      & & = \pr_{i\sim [m]}[y^*_i H_{t}(\ve{x}_i) \leq \theta] \cdot \min\{1 - ({-\poibayesrisk'})^{-1}(\theta), ({-\poibayesrisk'})^{-1}(-\theta)\},
    \end{eqnarray}
    so for any $\varepsilon \in [0,1]$,
    \begin{eqnarray}
(\pr_{i\sim [m]}[y^*_i H_{t}(\ve{x}_i) \leq \theta] \geq \varepsilon) & \Rightarrow & (\overline{w}_{t+1} \geq \varepsilon \cdot \min\{1 - ({-\poibayesrisk'})^{-1}(\theta), ({-\poibayesrisk'})^{-1}(-\theta)\}),
      \end{eqnarray}
      so if $\overline{w}_{t+1} < \zeta$, then by contraposition
      \begin{eqnarray}
        \pr_{i\sim [m]}[y^*_i H_{t}(\ve{x}_i) \leq \theta] & < & \frac{\zeta}{\min\{1 - ({-\poibayesrisk'})^{-1}(\theta), ({-\poibayesrisk'})^{-1}(-\theta)\}}\\
        & & = \frac{\zeta}{\underline{w}(\theta)},
      \end{eqnarray}
      as claimed.
    \end{proof}
    Fix from now on
    \begin{eqnarray}
\zeta & \defeq & \varepsilon \cdot \underline{w}(\theta)\label{defZETA}.
    \end{eqnarray}
    We have two cases to conclude on our main result.\\

    \noindent \textbf{Case 1}: sometimes during the induction, the weights for the "next iteration" ($t+1$) fail to be $\zeta$-dense. By Lemma \ref{lemEDGE1}, $\pr_{i\sim [m]}[y^*_i H_{t}(\ve{x}_i) \leq \theta] < \varepsilon$ and we are done.\\

    \noindent \textbf{Case 2}: weights are always $\zeta$-dense:
    \begin{eqnarray*}
\overline{w}^2_t & \geq & \varepsilon^2 \cdot \underline{w}(\theta)^2, \forall t = 1, 2, ...
      \end{eqnarray*}
      Recall the key statement of Theorem \ref{th-boost-1}:
      \begin{eqnarray*}
\forall \Phi \in \mathbb{R}, \left(\sum_{t=1}^T p_t \overline{w^*}^2_{t} \geq \frac{2(\popsur(H_{0}, \mathcal{S}) - \Phi)}{\kappa \gamma^2}\right) & \Rightarrow & \left(\popsur(H_{T}, \mathcal{S}) \leq \Phi\right).
\end{eqnarray*}
Provided we can assume a lowerbound of the form\footnote{Note that this is equivalent to $u_t$ compliance in Definition \ref{defUT-COMP}.}
\begin{eqnarray}
  p_t \overline{w^*}^2_{t} & \geq & u_t \overline{w}^2_t, \forall t = 1, 2, ... \label{condUT-pf}
\end{eqnarray}
where $u_t>0$ is not too small, we see that $\zeta$-denseness thus enforces a decrease of $\popsur(H, \mathcal{S})$ via Theorem \ref{th-boost-1}, and we only need a link between this and $\pr_{i\sim [m]}[y^*_i H_t(\ve{x}_i) \leq \theta]$, reminding
\begin{eqnarray*}
  \popsur(H, \mathcal{S}) \defeq \expect_{i\sim [m]}\left[\philoss(-H(\ve{x}_i)) - y_i H(\ve{x}_i)\right] & , & \philoss(z) \defeq (-\poibayesrisk)^\star(-z).
\end{eqnarray*}
\begin{lemma}\label{lemLINK-1}
Let $\underline{\philoss}(z) \defeq \min\{\philoss(z), \philoss(-z)-z\}$.  For any $t\geq 1$ and any $\theta \in \mathbb{R}$ such that:
    \begin{eqnarray*}
\underline{\philoss}(\theta) & > & \min_{i\in [m]} \underline{\philoss}(y^*_i H_t(\ve{x}_i)),
    \end{eqnarray*}
    we have for any $u\in [0,1]$,
    \begin{eqnarray*}
(\pr_{i\sim [m]}[y^*_i H_t(\ve{x}_i)\leq \theta] > u ) & \Rightarrow & \left(\popsur(H_t, \mathcal{S}_t) \geq u\underline{\philoss}(\theta) + (1-u) \min_{i\in [m]} \underline{\philoss}(y^*_i H_t(\ve{x}_i))\right).
      \end{eqnarray*}
  \end{lemma}
  \begin{proof}
    We reuse some notations from Lemma \ref{lemEDGE1}. We first note
\begin{eqnarray}
  \expect_{i\sim [m]}\left[\philoss(-H_t(\ve{x}_i)) - y_i H_t(\ve{x}_i)\right] & = & p^- \expect_{i\sim [m]_{-}}\left[\philoss(-H_t(\ve{x}_i))\right] \nonumber\\
  & & + p^+ \expect_{i\sim [m]_{+}}\left[\philoss(-H_t(\ve{x}_i)) - H_t(\ve{x}_i)\right].\label{decompEXPECT}
    \end{eqnarray}
    Let us analyse the term for negative examples and have  $z^-_i \leftarrow H(\ve{x}_i)$ for short. Because $\philoss(-z)$ is non-decreasing, for any $\theta \in \mathbb{R}$,
    \begin{eqnarray*}
      \lefteqn{\expect_{[m]_{-}}[\philoss(-z^-_i)]}\\
      & \geq & \pr_{[m]_{-}}[z^-_i \geq -\theta] \cdot \philoss(\theta) + (1-\pr_{[m]_{-}}[z^-_i \geq -\theta]) \cdot \min_{i\in [m]_{-}} \philoss(-z^-_i)\\
                                       & & = \pr_{[m]_{-}}[y^*_i H_t(\ve{x}_i) \leq \theta] \cdot \philoss(\theta) + (1-\pr_{[m]_{-}}[y^*_i H_t(\ve{x}_i) \leq \theta]) \cdot \min_{i\in [m]_{-}} \philoss(y^*_i H_t(\ve{x}_i))\\
      & \geq & \pr_{[m]_{-}}[y^*_i H_t(\ve{x}_i) \leq \theta] \cdot \underline{\philoss}(\theta) + (1-\pr_{[m]_{-}}[y^*_i H_t(\ve{x}_i) \leq \theta]) \cdot \min_{i\in [m]_{}} \underline{\philoss}(y^*_i H_t(\ve{x}_i)).
  \end{eqnarray*}
  Similarly for positive examples, letting $z^+_i \leftarrow H(\ve{x}_i)$ for short, we remark that $\philoss(-z)-z$ is non-increasing and so for any $\theta \in \mathbb{R}$,
  \begin{eqnarray*}
    \lefteqn{\expect_{[m]_{+}}[\philoss(-z^+_i)-z^+_i]}\\
    & \geq & \pr_{[m]_{+}}[z^+_i \leq \theta]  \cdot (\philoss(-\theta)-\theta) + (1-\pr_{[m]_{+}}[z^+_i \leq \theta] )\cdot \min_{i\in [m]_{+}} \philoss(-z^+_i)-z^+_i\\
    & & = \pr_{[m]_{+}}[y^*_i H_t(\ve{x}_i) \leq \theta]  \cdot (\philoss(-\theta)-\theta) \\
    & & + (1-\pr_{[m]_{+}}[y^*_i H_t(\ve{x}_i) \leq \theta] )\cdot \min_{i\in [m]_{+}} \philoss(-y^*_i H_t(\ve{x}_i))-y^*_i H_t(\ve{x}_i)\\
    & \geq & \pr_{[m]_{+}}[y^*_i H_t(\ve{x}_i) \leq \theta]  \cdot \underline{\philoss}(\theta) + (1-\pr_{[m]_{+}}[y^*_i H_t(\ve{x}_i) \leq \theta] )\cdot \min_{i\in [m]_{}} \underline{\philoss}(y^*_i H_t(\ve{x}_i)).
    \end{eqnarray*}
    Hence we get from \eqref{decompEXPECT} that for any $\theta \in \mathbb{R}$,
    \begin{eqnarray*}
      \lefteqn{\expect_{i\sim [m]}\left[\philoss(-H_t(\ve{x}_i)) - y_i H_t(\ve{x}_i)\right]}\\
      & = & p^- \expect_{i\sim [m]_{-}}\left[\philoss(-H_t(\ve{x}_i))\right] + p^+ \expect_{i\sim [m]_{+}}\left[\philoss(-H_t(\ve{x}_i)) - H_t(\ve{x}_i)\right]\\
      & \geq & ( p^-  \pr_{[m]_{-}}[y^*_i H_t(\ve{x}_i)\leq \theta]   + p^+  \pr_{[m]_{+}}[y^*_i H_t(\ve{x}_i) \leq \theta] ) \cdot \underline{\philoss}(\theta)\\
      & & + ((p^-+p^+) - (p^- \pr_{[m]_{-}}[y^*_i H_t(\ve{x}_i)\leq \theta]   + p^+  \pr_{[m]_{+}}[y^*_i H_t(\ve{x}_i) \leq \theta] )) \cdot \min_{i\in [m]_{}} \underline{\philoss}(y^*_i H_t(\ve{x}_i))\\
      & & = \pr_{i\sim [m]_{}}[y^*_i H_t(\ve{x}_i)\leq \theta] \cdot \underline{\philoss}(\theta) + (1 - \pr_{i\sim [m]_{}}[y^*_i H_t(\ve{x}_i)\leq \theta]) \cdot \min_{i\in [m]_{}} \underline{\philoss}(y^*_i H_t(\ve{x}_i)).
    \end{eqnarray*}
    We get that for any $\theta \in \mathbb{R}$ such that:
    \begin{eqnarray*}
\underline{\philoss}(\theta) & > & \min_{i\in [m]_{}} \underline{\philoss}(y^*_i H_t(\ve{x}_i)),
    \end{eqnarray*}
    we have for any $u\in [0,1]$,
    \begin{eqnarray*}
      \lefteqn{(\pr_{i\sim [m]_{}}[y^*_i H_t(\ve{x}_i)\leq \theta] > u )}\\
      &\Rightarrow & \left(\underbrace{\expect_{i\sim [m]}\left[\philoss(-H_t(\ve{x}_i)) - y_i H_t(\ve{x}_i)\right]}_{\defeq \popsur(H_t, \mathcal{S}_t) } \geq u\underline{\philoss}(\theta) + (1-u) \min_{i\in [m]_{}} \underline{\philoss}(y^*_i H_t(\ve{x}_i))\right),
      \end{eqnarray*}
    as claimed.
  \end{proof}
If all weights at iterations $t$ are $\zeta \defeq \varepsilon \cdot \underline{w}(\theta)$-dense for $t=1, 2, ...$, then, letting $u\defeq \varepsilon$ in Lemma \ref{lemLINK-1} and
  \begin{eqnarray*}
    \Phi & \defeq & \varepsilon \underline{\philoss}(\theta) + (1-
                      \varepsilon) \philoss_*
  \end{eqnarray*}
  in Theorem \ref{th-boost-1}, for some $\philoss_*$ to be made precise, then a sufficient condition to get $\popsur(H_{T}, \mathcal{S}) \leq \Phi$ is $\sum_{t=1}^T u_t \overline{w}^2_t  \geq 2(\popsur(H_{0}, \mathcal{S}) - \varepsilon \underline{\philoss}(\theta) - (1-\varepsilon) \philoss_*)/(\kappa \gamma^2)$ (using \eqref{condUT-pf}), and integrating the
$\zeta$-denseness of weights, this condition becomes the sufficient condition:
  \begin{eqnarray}
\sum_{t=1}^T u_t & \geq & \frac{2(\popsur(H_{0}, \mathcal{S}) - \varepsilon \underline{\philoss}(\theta) - (1-\varepsilon) \philoss_*)}{\kappa \varepsilon^2 \underline{w}(\theta)^2 \gamma^2}.\label{eqSUMP-suite}
  \end{eqnarray}
  So, if we pick $\philoss_* \defeq \min_{i\in [m]} \underline{\philoss}(y^*_i H_T(\ve{x}_i))$, then from Lemma \ref{lemLINK-1} we get
\begin{eqnarray*}
\pr_{i\sim [m]}[y^*_i H_T(\ve{x}_i) \leq \theta] & < & \varepsilon,
\end{eqnarray*}
which is what we want. We wrap up in two last steps. We first simplify the RHS of \eqref{eqSUMP-suite} by replacing it by a more readable sufficient condition: if the loss' partial losses satisfy
\begin{eqnarray}
\partialloss{-1}(0), \partialloss{1}(1) & \geq & C \label{condPARTIAL}
  \end{eqnarray}
  for some $C \in \mathbb{R}$ (such as if the loss is fair: $C=0$), then we remark that for any $H\in \mathbb{R}$ and $y\in \{0,1\}$,
  \begin{eqnarray}
    \philoss(-H) - y H & = & \sup_{u\in [0,1]} \{(u-y)H + \underline{L}(u)\} \nonumber\\
                        & = & \sup_{u\in [0,1]} \{(u-y)H + u\partialloss{1}(u) + (1-u)\partialloss{-1}(u)\}\nonumber\\
    & \geq & y\partialloss{1}(y) + (1-y)\partialloss{-1}(y).\label{binfUPPHI}
  \end{eqnarray}
The integral representation of proper losses \cite[Theorem 1]{rwCB} \cite[Appendix Section 9]{nmSL},
\begin{eqnarray*}
\partialloss{1}(u) = \int_{u}^1 (1-t) w(t) \mathrm{d}t & , & \partialloss{-1}(u) = \int_{0}^u t w(t) \mathrm{d}t ,
\end{eqnarray*}
where $(0,1) \rightarrow \mathbb{R}_+$, shows that $\partialloss{1}$ is non-increasing and $\partialloss{-1}$ is non-decreasing, so $\inf \partialloss{-1} \geq C \neq \pm\infty$ and $\inf \partialloss{1} \geq C \neq \pm\infty$,
so \eqref{binfUPPHI} yields
\begin{eqnarray}
  \philoss(-H) - y H & \geq & \left\{
                               \begin{array}{rcl}
                                 \partialloss{1}(1) \geq C & \mbox{ if } & y = 1,\\
                                 \partialloss{-1}(0) \geq C & \mbox{ if } & y = 0
                                 \end{array}
                               \right. ,
\end{eqnarray}
so $\min_{i\in [m]} {\philoss}(y^*_i H_T(\ve{x}_i)) \geq C$ and $\underline{\philoss}(\theta) \geq C$, which allows us to replace \eqref{eqSUMP-suite} by the sufficient condition:
\begin{eqnarray}
\sum_{t=1}^T u_t & \geq & \frac{2(\popsur(H_{0}, \mathcal{S}) - C)}{\kappa \varepsilon^2 \underline{w}(\theta)^2 \gamma^2}\label{eqSUMP2}.
\end{eqnarray}
In our second step to wrap-up, if we have $\sum_{t=1}^T u_t \geq U(T)$ (for some $U$ strictly increasing and thus invertible), then under the three conditions:
\begin{itemize}
\item \textbf{LOSS0} and \eqref{condPARTIAL} on the loss,
\item \textbf{WLA} on the weak learner,
  \item \eqref{condUT-pf} on Step 2.1 of \topdowngen~($u_t$ is chosen such that the choice of $\mathcal{X}_t$ is $u_t$ compliant),
  \end{itemize}
we are guaranteed that anytime we have
\begin{eqnarray}
T & \geq & U^{-1}\left(\frac{2\left(\popsur(H_{0}, \mathcal{S}) - C\right)}{\kappa \cdot \varepsilon^2 \underline{w}(\theta)^2 \gamma^2}\right),\label{eqWITHPHIXpf}
\end{eqnarray}
we are guaranteed
\begin{eqnarray*}
\pr_{i\sim [m]}[y^*_i H_T(\ve{x}_i) \leq \theta] & < & \varepsilon,
\end{eqnarray*}
which is the statement of the Theorem.

 \subsection{Proof of Lemma \ref{lem-split-gives-wla}}\label{proof-lem-split-gives-wla}

\begin{figure}[t]
\begin{center}
\includegraphics[trim=30bp 480bp 500bp 20bp,clip,width=0.45\linewidth]{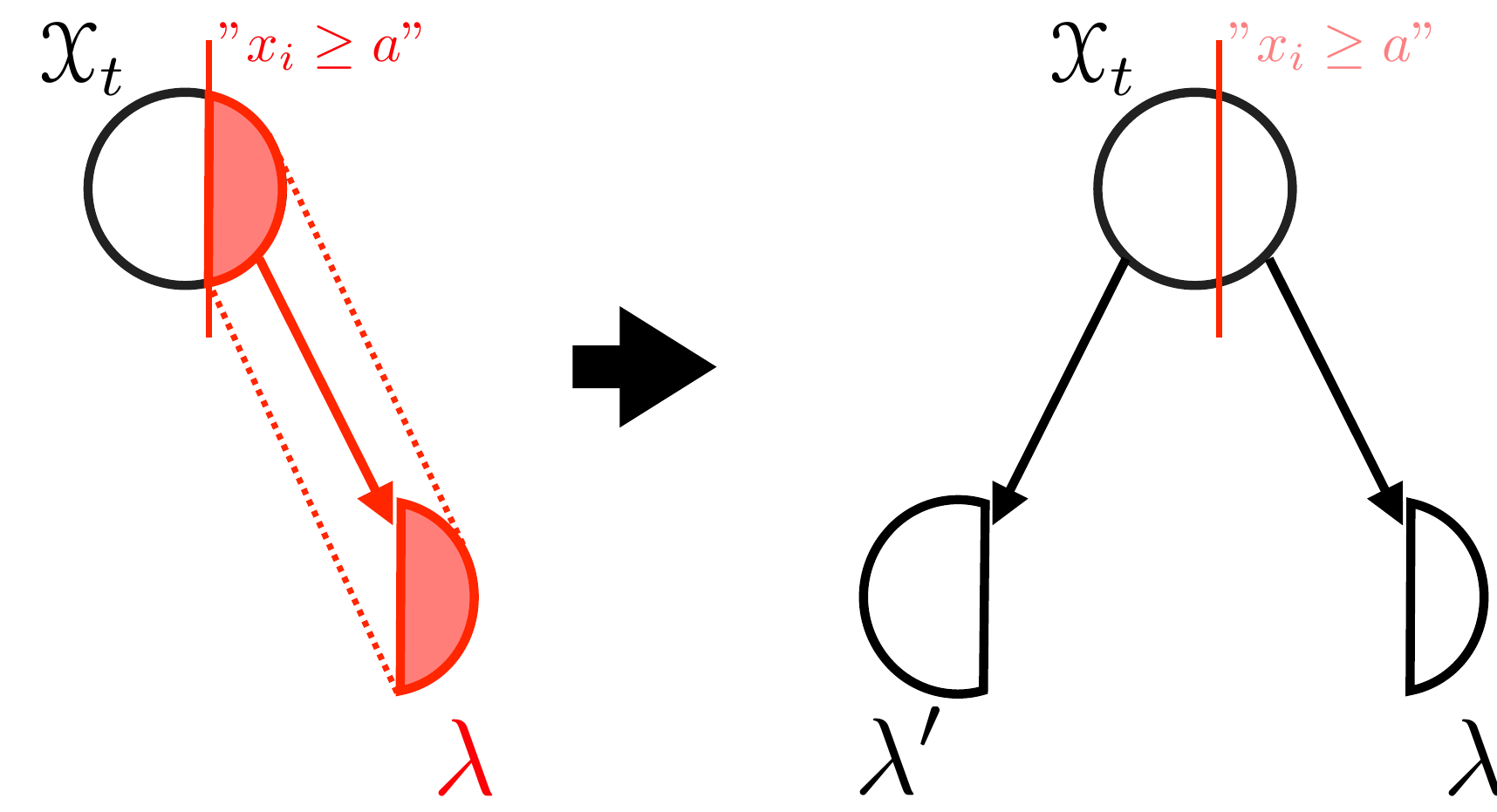} 
\end{center}
\caption{When \topdowngen~learns a decision trees, $\mathcal{X}_t$ in Step 2.1 is the domain corresponding to a leaf $\leaf$. The weak learner gives the split and fits the prediction of one leaf only to guarantee the WLA. The WLA at this split guarantees that the other split also complies with the WLA (the red parts are chosen by the weak learner, see text).}
  \label{f-split}
\end{figure}

Denote $W^+, W^-$ the total sum of (unnormalized) boosting weights in $\mathcal{S}_t$ before the call for splitting the node. Denote $W^+_{\mathrm{r}}, W^-_{\mathrm{r}}$ the corresponding weights at the new leaf $\leaf$ and $W^+_{\mathrm{l}}, W^-_{\mathrm{l}}$ the corresponding weights at the next leaf $\leaf'$, completing the split (See Figure \ref{f-split}). 
The proof of the Lemma relies on the following observations:
\begin{itemize}
\item [(1)] $W^+ = W^-$ (before split, the current leaf is balanced); also, $W^+ = W^+_{\mathrm{l}} + W^+_{\mathrm{r}}$ and $W^- = W^-_{\mathrm{l}} + W^-_{\mathrm{r}}$ (every example in $\mathcal{X}_t$ goes to exactly one new leaf);
  \item [(2)] the weak learner predicts $\mathrm{1}_{x_i \geq a} \cdot h_t \in \{-1,0,1\}$ wlog (call it the "prediction at the right node of the split at $\mathcal{X}_t$");
  \end{itemize}
We then derive from the quantity in absolute value of the \textbf{WLA} \eqref{defWLA3}:
  \begin{eqnarray}
    \sum_{i\in [m]_t} \frac{w_{t,i}}{\sum_{j\in [m]_t} w_{t,j}} \cdot y^*_i \cdot \frac{h_{t}(\ve{x}_i)}{\max_{j\in [m]_t}|h_t(\ve{x}_j)|}  & = &  \sum_{i\in [m]_t: x_i \geq a} \frac{w_{t,i}}{\sum_{j\in [m]_t} w_{t,j}} \cdot y^*_i \cdot h_t\nonumber\\
                                                                                                                                           & = &  \left(\frac{W^+_{\mathrm{r}}}{W^++W^-}-\frac{W^-_{\mathrm{r}}}{W^++W^-}\right) h_t \nonumber \\ 
                                                                                                                                           & = & \left(\frac{W^+ - W^+_{\mathrm{l}}}{W^++W^-}-\frac{W^- - W^-_{\mathrm{l}}}{W^++W^-}\right) h_t\nonumber\\
    & = & \left(\frac{W^+-W^-}{W^++W^-}+ \frac{W^-_{\mathrm{l}}}{W^++W^-} - \frac{W^+_{\mathrm{l}}}{W^++W^-}\right) h_t\nonumber\\
    & = & \left(\frac{W^-_{\mathrm{l}}}{W^++W^-} - \frac{W^+_{\mathrm{l}}}{W^++W^-}\right) h_t\nonumber\\
                                                                                                                                            & = & \sum_{i\in [m]_t: x_i < a} \frac{w_{t,i}}{\sum_{j\in [m]_t} w_{t,j}} \cdot y^*_i \cdot (-h_t)\nonumber\\
    & = & \sum_{i\in [m]_t} \frac{w_{t,i}}{\sum_{j\in [m]_t} w_{t,j}} \cdot y^*_i \cdot \frac{h'_{t}(\ve{x}_i)}{\max_{j\in [m]_t}|h'_t(\ve{x}_j)|} ,
  \end{eqnarray}
  with
  \begin{eqnarray}
h'_t(\ve{x}) & \defeq & \mathrm{1}_{x_i < a} \cdot (-h_t),
  \end{eqnarray}
  which is both (i) a function computing a prediction for the left node of the split at $\mathcal{X}_t$ and (ii) satisfying the \textbf{WLA} since $\mathrm{1}_{x_i \geq a} \cdot h_t$ does satisfy the \textbf{WLA}.
  
\subsection{Proof of Lemma \ref{lem-J}}\label{proof-lem-J}

        Let $\leaf$ denote a leaf of the decision tree. We have
        \begin{eqnarray}
          J(\leaf) & = & m_\leaf \cdot \left(\frac{\sum_{i\sim \leaf} y_i - y^*_i\cdot ({-\poibayesrisk'})^{-1}(H_\leaf)}{m_\leaf}\right)^2\nonumber\\
          & = &  m_\leaf \cdot \left(\frac{m^+_\leaf - (m^+_\leaf - m^-_\leaf) \cdot \frac{m^+_\leaf}{m_\leaf}}{m_\leaf}\right)^2\label{eqHsimp}\\
                   & = &  \frac{(m^+_\leaf)^2}{m^3_\leaf} \cdot \left(m_\leaf - m^+_\leaf + m^-_\leaf\right)^2\nonumber\\
                   & = & 4 \cdot m_\leaf \cdot \left(\frac{m^+_\leaf}{m_\leaf}\right)^2\cdot \left(\frac{m^-_\leaf}{m_\leaf}\right)^2\nonumber\\
                  & = & 4 m \cdot \frac{m_\leaf}{m} \cdot \left(\frac{m^+_\leaf}{m_\leaf}\right)^2\cdot \left(\frac{m^-_\leaf}{m_\leaf}\right)^2\nonumber\\
                   & \propto & p_\leaf \cdot (2 p^+_\leaf (1-p^+_\leaf))^2\nonumber,
        \end{eqnarray}
        as claimed. In \eqref{eqHsimp}, we have made use of the expression in \eqref{eqvalH}.

\subsection{Proof of Lemma \ref{lem-BOO2}}\label{proof-lem-BOO2}

We proceed in three steps. Suppose a new leaf $\leaf$ has been put, with prediction $h_\leaf$, by the weak learner (this is in fact "half a split" as usually described for \cdt s). We compute the leveraging coefficient $\alpha_\leaf$ in Step 2.3 of \topdowngen. Denote $\parent(\node)$ the parent node of $\node$ in $H$.
  \begin{eqnarray}
H_{\parent(\leaf)} & \defeq & \sum_{\node \in \mathrm{path}(\leaf) \backslash \{\leaf_t\}} \alpha_\node h_\node
  \end{eqnarray}
  is the prediction computed from the root of the tree up to the parent of $\leaf$. Given a constant prediction $h_\leaf$ at leaf $\leaf$, We wish to find $\alpha_\leaf$ so that \eqref{defAlpha} holds. We reuse notations from Lemma \ref{lem-J} and its proof.  We note that \eqref{defAlpha} is equivalent to
  \begin{eqnarray}
m_\leaf^+ - (m_\leaf - m_\leaf^+) \cdot ({-\poibayesrisk'})^{-1}(\alpha_\leaf h_\leaf + H_{\parent(\leaf)} ) - m_\leaf^+ \cdot ({-\poibayesrisk'})^{-1}(\alpha_\leaf h_\leaf + H_{\parent(\leaf)} ) & = & 0,
  \end{eqnarray}
  which gives, since $p^+_\leaf = m_\leaf^+ / m_\leaf$,
  \begin{eqnarray}
\alpha_\leaf & = & \frac{1}{h_\leaf} \cdot \left( ({-\poibayesrisk'}) \left(p^+_\leaf\right) - H_{\parent(\leaf)} \right),
  \end{eqnarray}
  Our second step computes the final decision tree prediction at the new leaf $\leaf$, which is trivially:
  \begin{eqnarray}
    H_{\leaf} & \defeq & H_{\parent(\leaf)} + \alpha_\leaf h_\leaf\nonumber\\
    & = & ({-\poibayesrisk'}) \left(p^+_\leaf\right).\label{eqvalH}
  \end{eqnarray}
  Plugging this prediction in the \spd, it simplifies as
  \begin{eqnarray}
\popsur(H, \mathcal{S}) & = & \expect_{\leaf \sim \leafset(h)}\left[(-\poibayesrisk)^\star(-\poibayesrisk'(p^+_\leaf))  + p^+_\leaf \poibayesrisk'(p^+_\leaf) \right]\label{eqLoss1} \\
                                               & = & \expect_{\leaf \sim \leafset(h)}\left[-\poibayesrisk'(p^+_\leaf) \cdot {({-\poibayesrisk}'})^{-1} (-\poibayesrisk'(p^+_\leaf)) + \poibayesrisk \circ {({-\poibayesrisk}')}^{-1} (-\poibayesrisk'(p^+_\leaf)) + p^+_\leaf \poibayesrisk'(p^+_\leaf) \right]\nonumber\\
  & = & \expect_{\leaf \sim \leafset(h)}\left[-p^+_\leaf \poibayesrisk'(p^+_\leaf) + \poibayesrisk (p^+_\leaf) + p^+_\leaf \poibayesrisk'(p^+_\leaf) \right]\nonumber\\
  & = & \expect_{\leaf \sim \leafset(h)}\left[\poibayesrisk (p^+_\leaf) \right],\label{eqLoss2}
\end{eqnarray}
as claimed.

  \subsection{Proof of Lemma \ref{lem-NoNoise}}\label{proof-lem-NoNoise}

If the loss is symmetric, then $p^* = 1/2$. In this case, the sign of $\tilde{p}_\leaf^+ - 1/2$ is the same as the sign of $p_\leaf^+ - 1/2$. Assuming we know the noise rate in advance (like in \cite{ksBI}) and do not have generalisation issues (like with the dataset of \cite{lsRC}), the sign of $H_{\leaf}$ with and without noise are the same and thus \topdowngen~is not affected by noise when inducing decision trees.

If the loss is not symmetric, the picture changes. To have a sign flip between no noise and noise, we need either:
\begin{eqnarray}
p_\leaf^+> p^* > \etanoise + (1-2\etanoise) p_\leaf^+, \label{ineqR}\\
p_\leaf^+< p^* < \etanoise + (1-2\etanoise) p_\leaf^+.\label{ineqL}
\end{eqnarray}
Note that looking at the extremes, we see that \eqref{ineqR} implies $p_\leaf^+ > 1/2$ while \eqref{ineqL} implies $p_\leaf^+ < 1/2$. We have two cases:
\begin{enumerate}
\item [] \textbf{Case 1}: we reach the point where
  \begin{eqnarray*}
p^+_\leaf & \leq & \min \left\{\frac{p^*-\etanoise}{1-2\etanoise}, \frac{1}{2}\right\}.
  \end{eqnarray*}
  We note that the $1/2$ upperbound prevents \eqref{ineqR}, while the other one can be reformulated as $\etanoise + (1-2\etanoise) p_\leaf^+ \leq p^*$, preventing \eqref{ineqL}.
  \item [] \textbf{Case 2}: we reach the point where
  \begin{eqnarray*}
1-p^+_\leaf & \leq & \min \left\{\frac{(1-p^*)-\etanoise}{1-2\etanoise}, \frac{1}{2}\right\}.
  \end{eqnarray*}
  We note that the $1/2$ upperbound prevents \eqref{ineqL}, while the other one can be reformulated as $p^* \leq \etanoise + (1-2\etanoise)p^+_\leaf$, preventing \eqref{ineqR}.
  \end{enumerate}
This ends the proof of Lemma \ref{lem-NoNoise}

\section{Supplementary material on experiments} \label{sec-sup-exp}

  \newcommand{\picwidthsi}{0.28}
\begin{table*}
  \centering
  \begin{tabular}{rccc}\hline\hline
    & \cls \hspace{-0.3cm} & \hspace{-0.3cm} \cdt \hspace{-0.3cm} & \hspace{-0.3cm} 1-\cnn \\
    \rotatebox{90}{{\footnotesize \texttt{Matusita loss}}} & \hspace{-0.3cm} \includegraphics[trim=0bp 0bp 0bp 0bp,clip,width=\picwidthsi\textwidth]{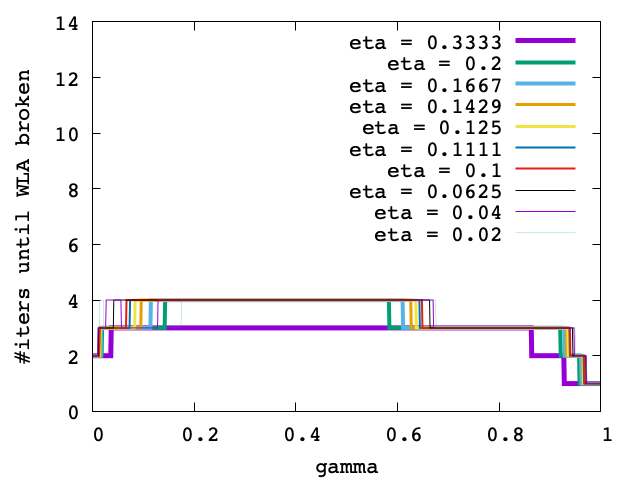} \hspace{-0.3cm} & \hspace{-0.3cm} \includegraphics[trim=0bp 0bp 0bp 0bp,clip,width=\picwidthsi\textwidth]{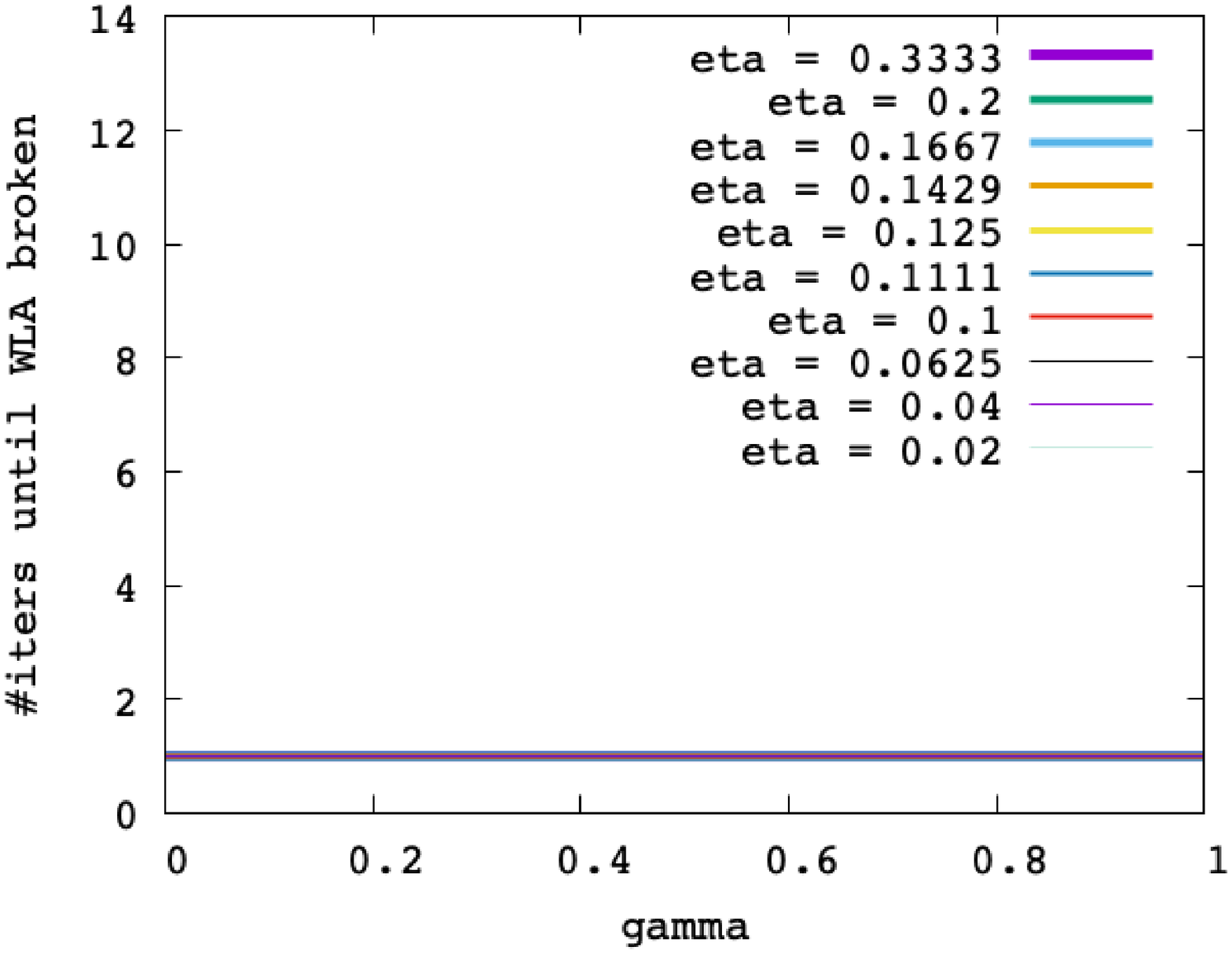} \hspace{-0.3cm} & \hspace{-0.3cm} \includegraphics[trim=0bp 0bp 0bp 0bp,clip,width=\picwidthsi\textwidth]{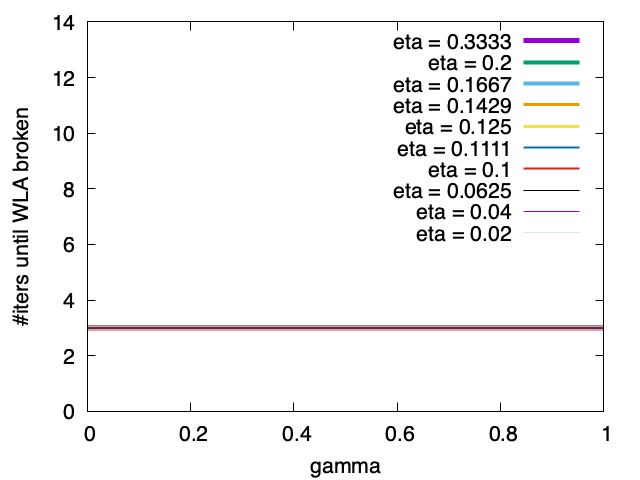} \hspace{-0.3cm} \\
    \rotatebox{90}{{\footnotesize \texttt{Log loss}}} & \hspace{-0.3cm} \includegraphics[trim=0bp 0bp 0bp 0bp,clip,width=\picwidthsi\textwidth]{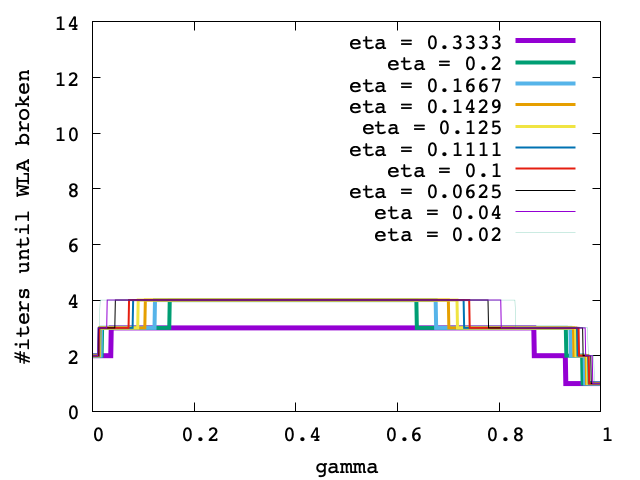} \hspace{-0.3cm} & \hspace{-0.3cm} \includegraphics[trim=0bp 0bp 0bp 0bp,clip,width=\picwidthsi\textwidth]{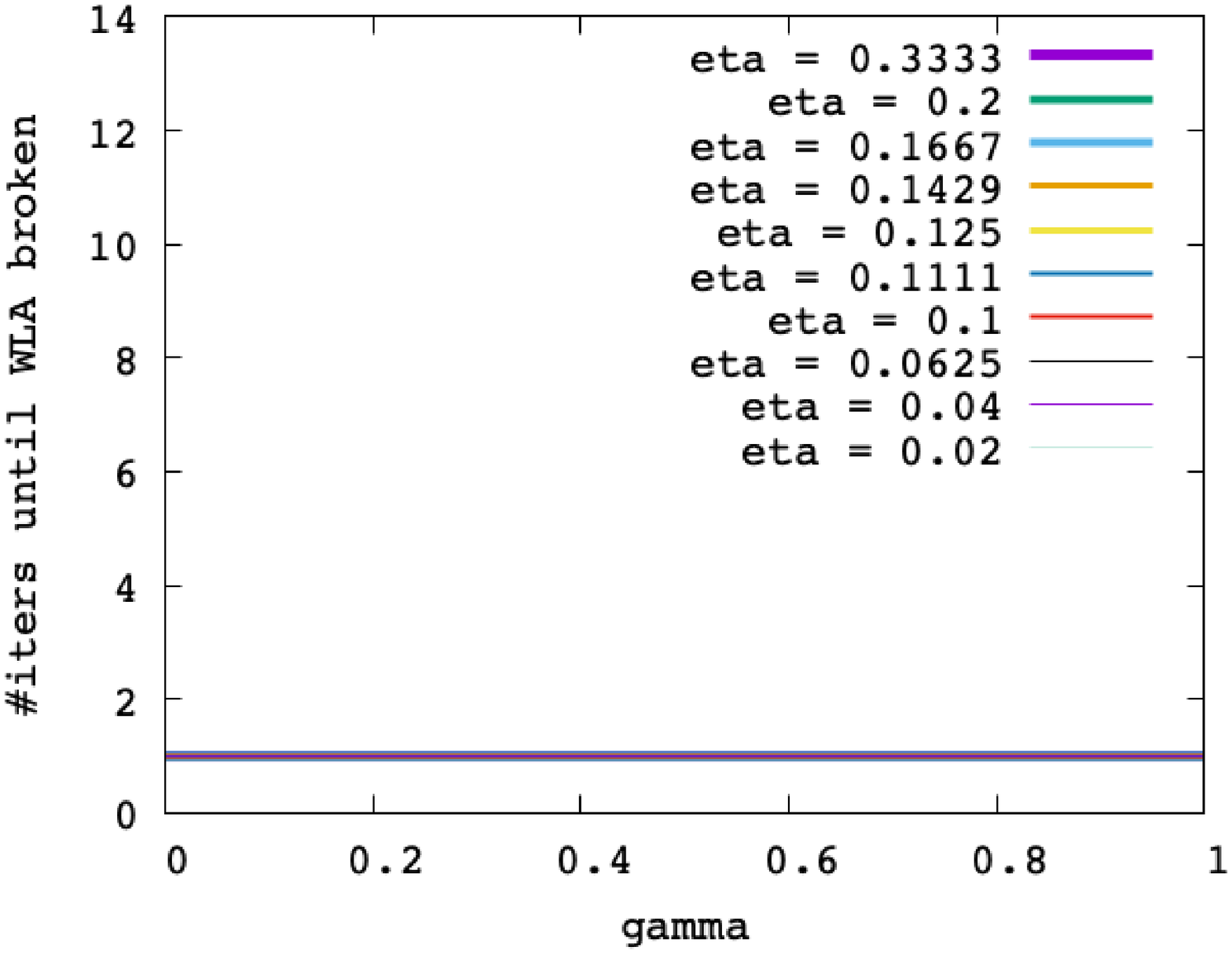} \hspace{-0.3cm} & \hspace{-0.3cm} \includegraphics[trim=0bp 0bp 0bp 0bp,clip,width=\picwidthsi\textwidth]{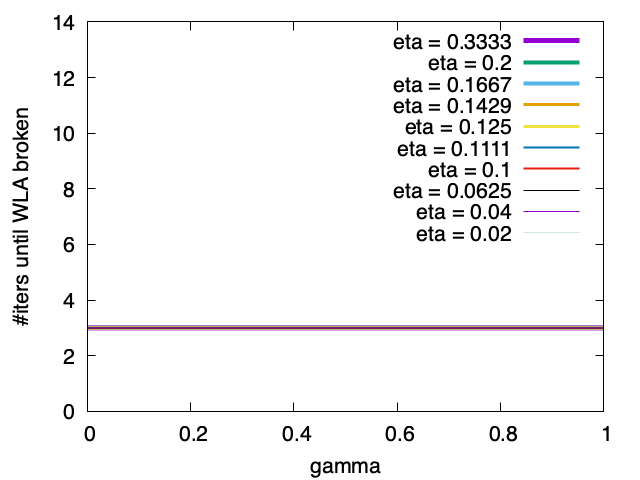} \hspace{-0.3cm} \\
    \rotatebox{90}{{\footnotesize \texttt{Square loss}}} & \hspace{-0.3cm} \includegraphics[trim=0bp 0bp 0bp 0bp,clip,width=\picwidthsi\textwidth]{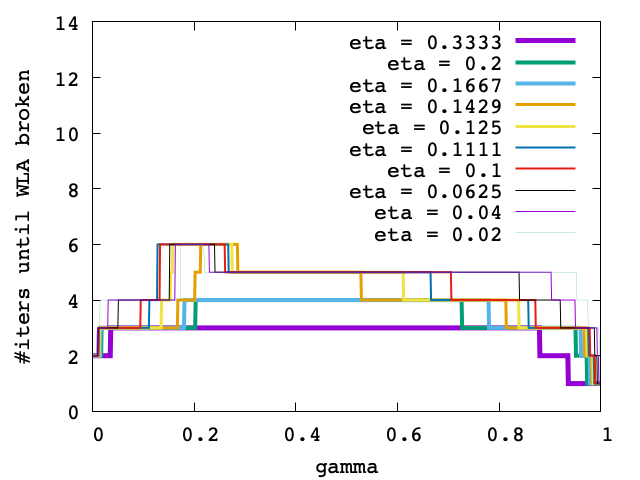} \hspace{-0.3cm} & \hspace{-0.3cm} \includegraphics[trim=0bp 0bp 0bp 0bp,clip,width=\picwidthsi\textwidth]{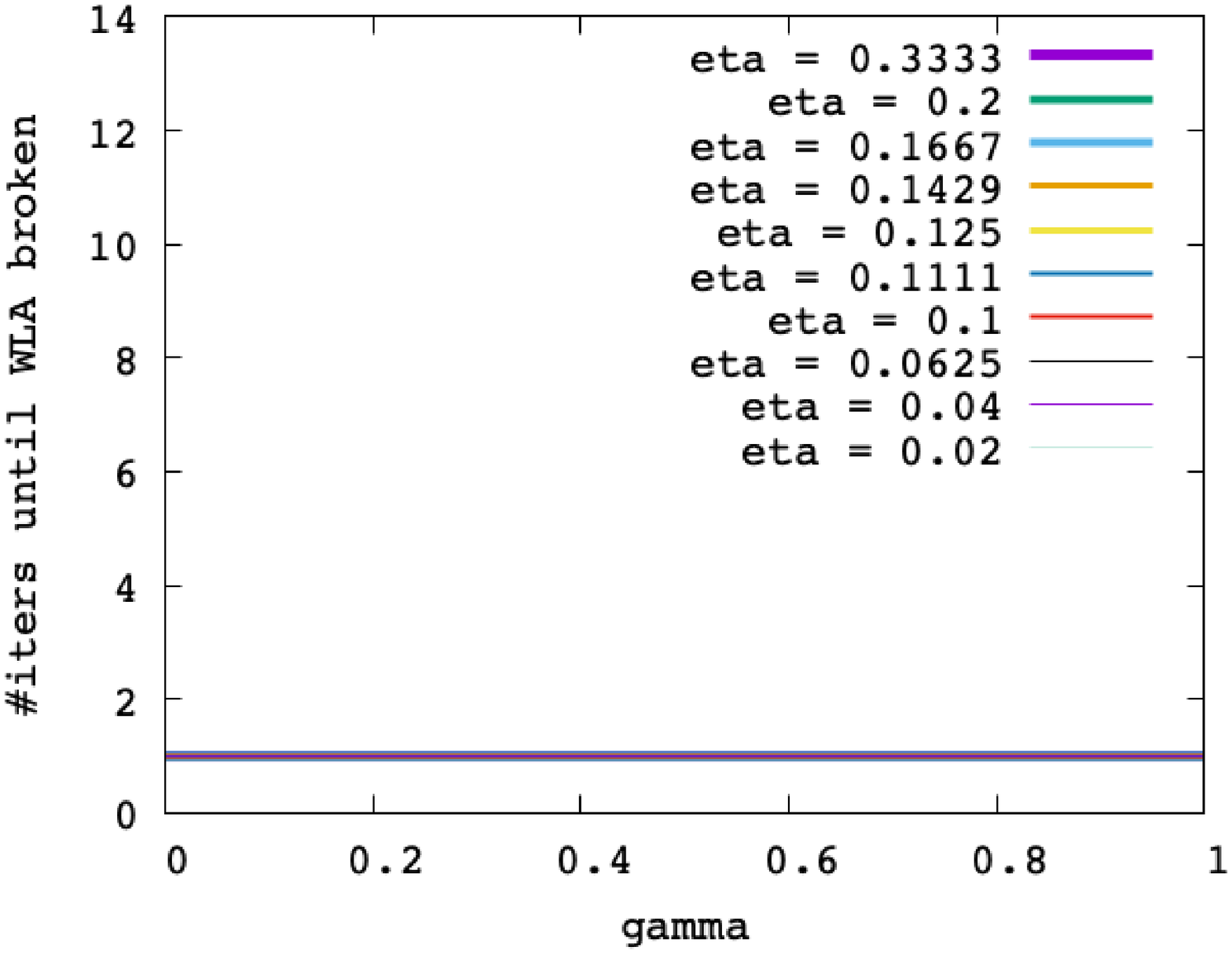} \hspace{-0.3cm} & \hspace{-0.3cm} \includegraphics[trim=0bp 0bp 0bp 0bp,clip,width=\picwidthsi\textwidth]{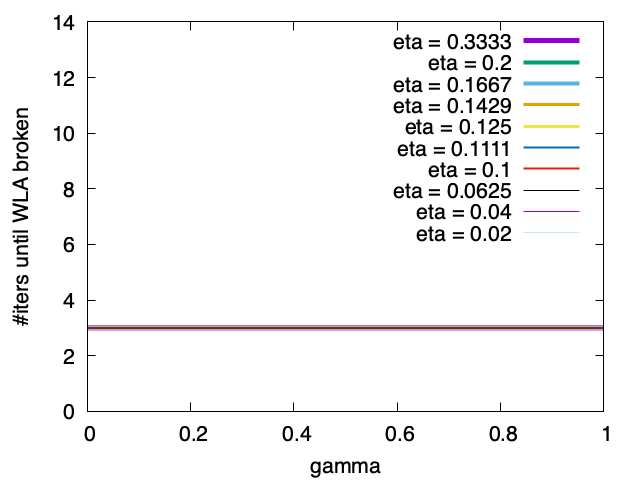} \hspace{-0.3cm} \\
    \rotatebox{90}{{\footnotesize \texttt{Asymmetric loss 1}}} & \hspace{-0.3cm} \includegraphics[trim=0bp 0bp 0bp 0bp,clip,width=\picwidthsi\textwidth]{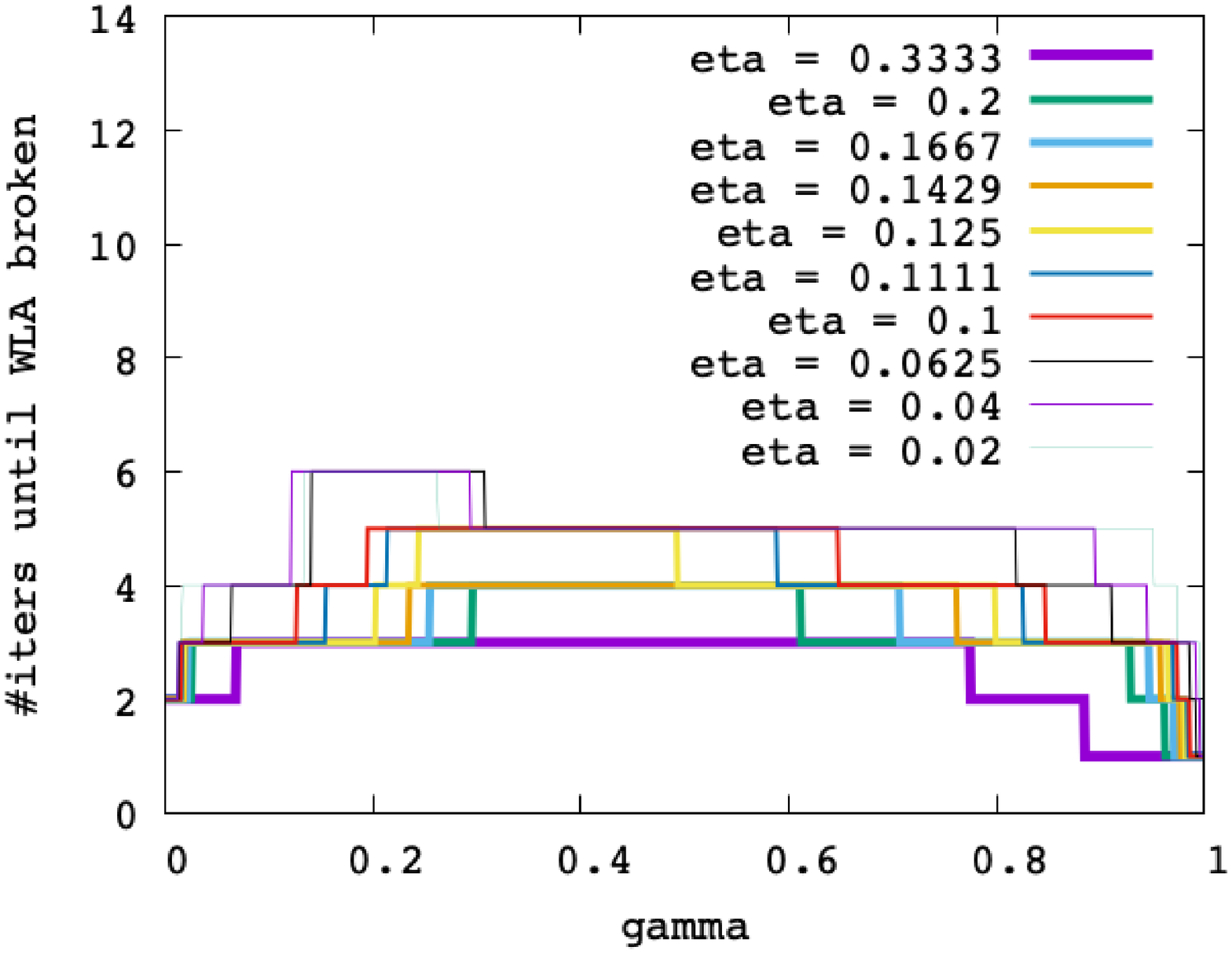} \hspace{-0.3cm} & \hspace{-0.3cm} \includegraphics[trim=0bp 0bp 0bp 0bp,clip,width=\picwidthsi\textwidth]{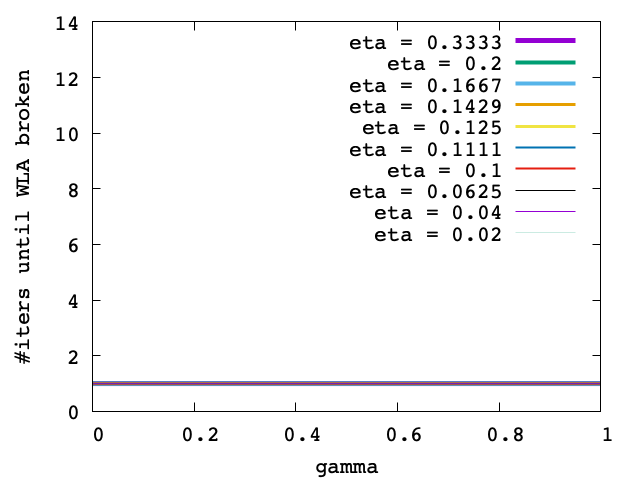} \hspace{-0.3cm} & \hspace{-0.3cm} \includegraphics[trim=0bp 0bp 0bp 0bp,clip,width=\picwidthsi\textwidth]{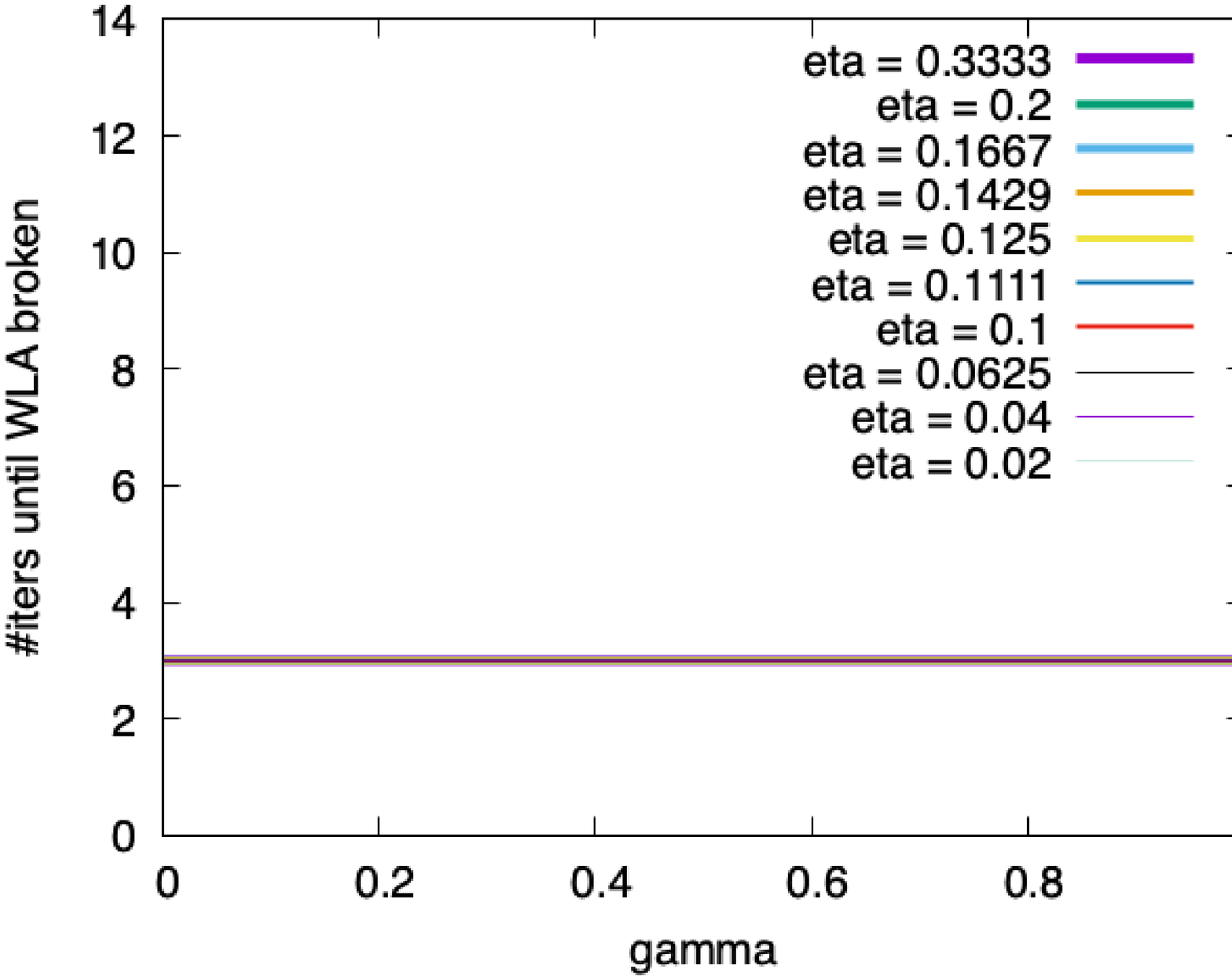} \hspace{-0.3cm} \\ \hline\hline
\end{tabular}
\caption{Number of calls to the Weak Learner in \topdowngen~until the
  Weak Learning Assumption does not hold anymore, for a minimal $\gammawla = 0.001$. Predictably, on Long and Servedio's data, it takes only a single root prediction for a \cdt~to achieve Bayes optimal prediction after computing its leveraging coefficient, and it takes 3 leveraging iterations = number of distinct observations for the \cnn~classifier to achieve Bayes optimal prediction.}
    \label{tab:iter-not-WLA}
  \end{table*}

\begin{table*}
  \centering
  \begin{tabular}{rccc}\hline\hline
    & \cls \hspace{-0.3cm} & \hspace{-0.3cm} \cdt  & \hspace{-0.3cm} 1-\cnn \\
    \rotatebox{90}{{\footnotesize \texttt{Matusita loss}}} & \hspace{-0.3cm} \includegraphics[trim=0bp 0bp 0bp 0bp,clip,width=\picwidthsi\textwidth]{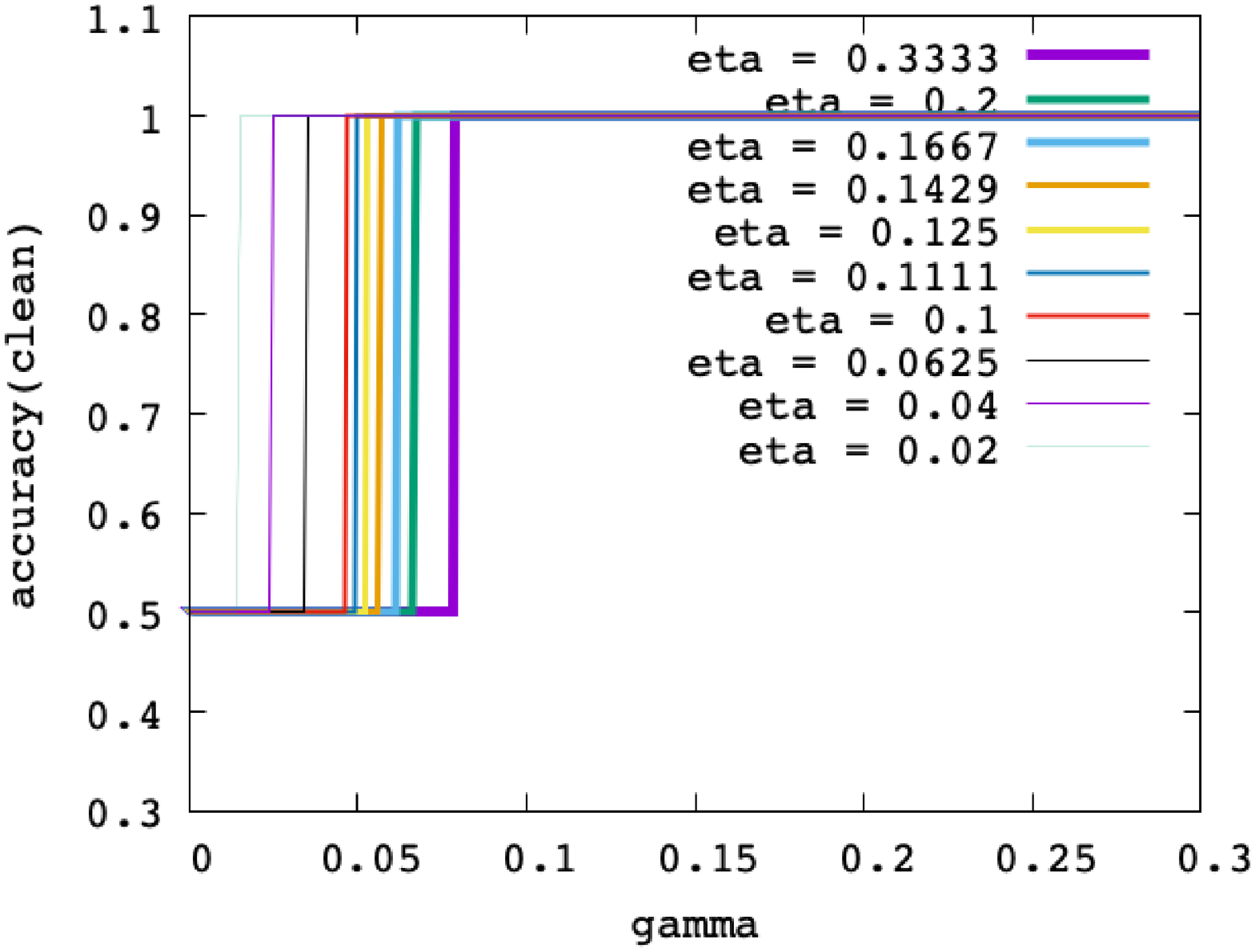} \hspace{-0.3cm} & \hspace{-0.3cm} \includegraphics[trim=0bp 0bp 0bp 0bp,clip,width=\picwidthsi\textwidth]{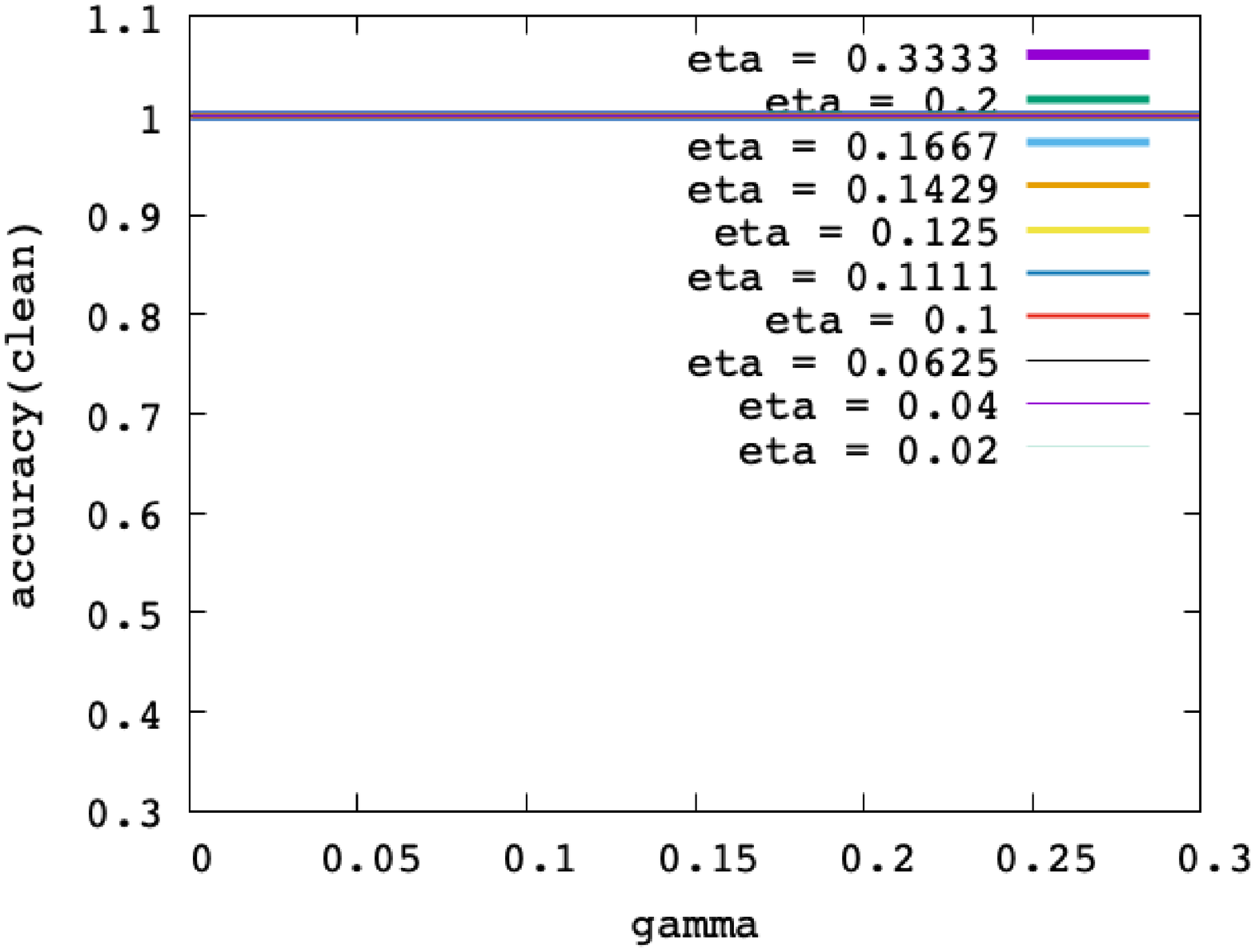} \hspace{-0.3cm} & \hspace{-0.3cm} \includegraphics[trim=0bp 0bp 0bp 0bp,clip,width=\picwidthsi\textwidth]{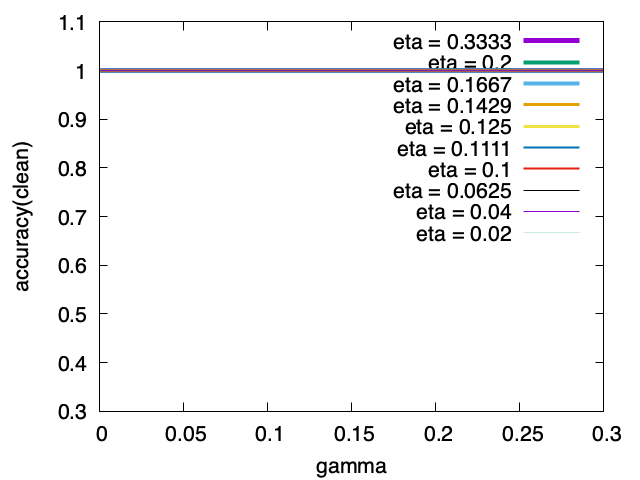} \hspace{-0.3cm} \\
    \rotatebox{90}{{\footnotesize \texttt{Log loss}}} & \hspace{-0.3cm} \includegraphics[trim=0bp 0bp 0bp 0bp,clip,width=\picwidthsi\textwidth]{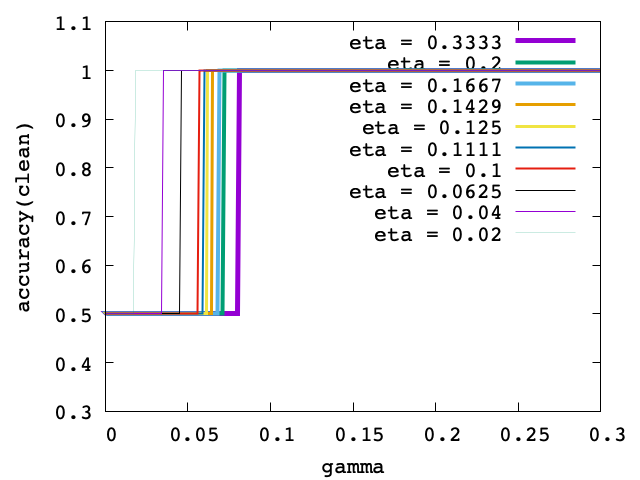} \hspace{-0.3cm} & \hspace{-0.3cm} \includegraphics[trim=0bp 0bp 0bp 0bp,clip,width=\picwidthsi\textwidth]{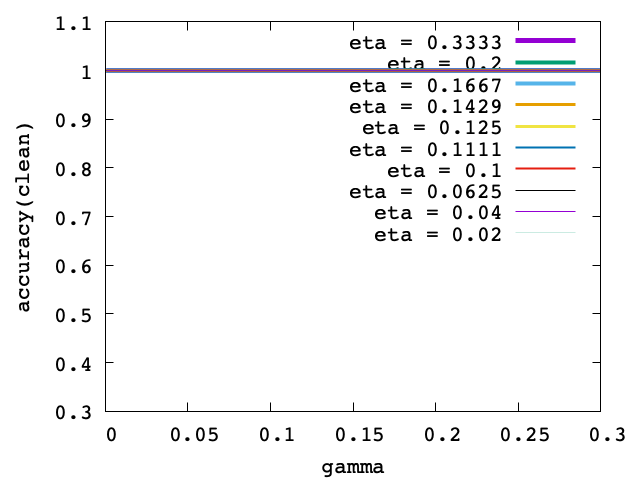} \hspace{-0.3cm}  & \hspace{-0.3cm} \includegraphics[trim=0bp 0bp 0bp 0bp,clip,width=\picwidthsi\textwidth]{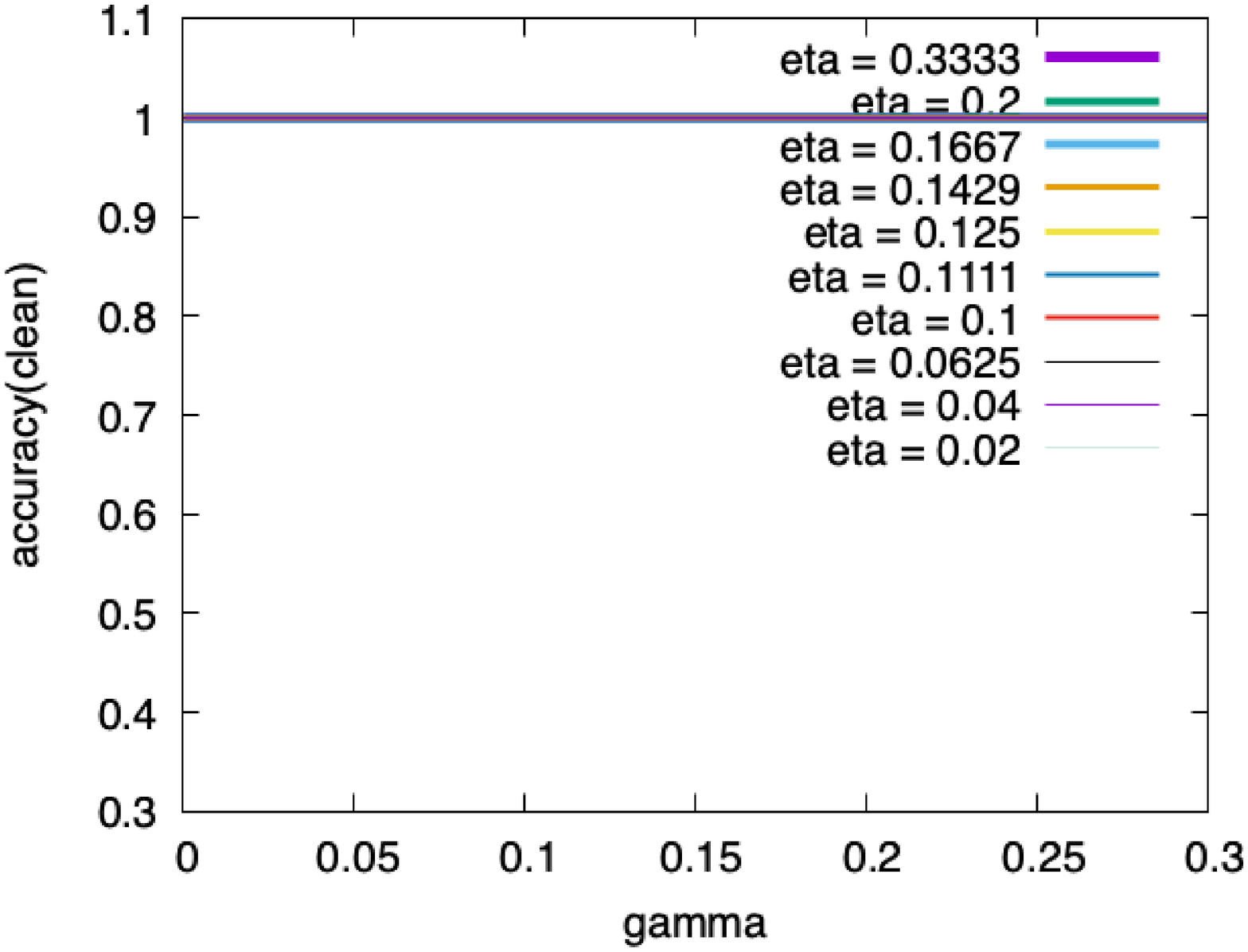} \hspace{-0.3cm} \\
    \rotatebox{90}{{\footnotesize \texttt{Square loss}}} & \hspace{-0.3cm} \includegraphics[trim=0bp 0bp 0bp 0bp,clip,width=\picwidthsi\textwidth]{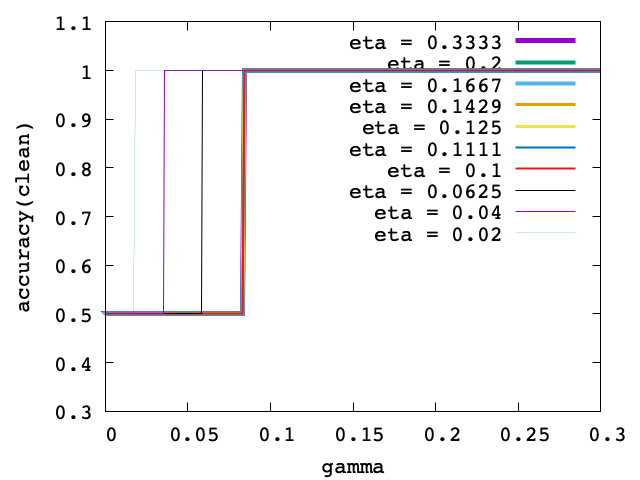} \hspace{-0.3cm} & \hspace{-0.3cm} \includegraphics[trim=0bp 0bp 0bp 0bp,clip,width=\picwidthsi\textwidth]{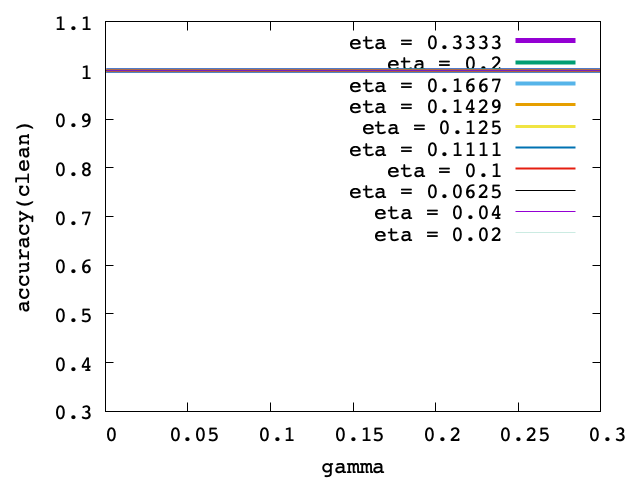} \hspace{-0.3cm}  & \hspace{-0.3cm} \includegraphics[trim=0bp 0bp 0bp 0bp,clip,width=\picwidthsi\textwidth]{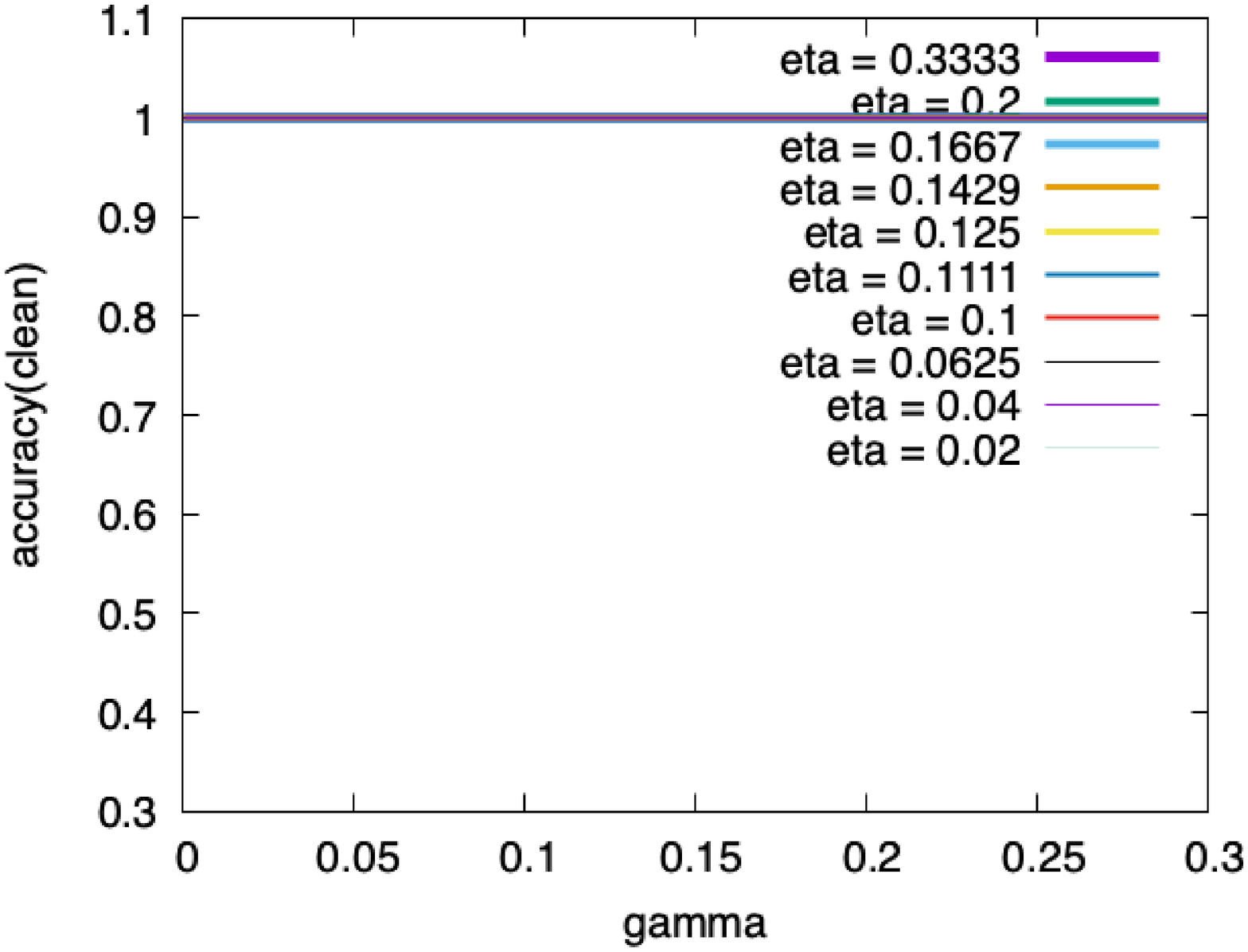} \hspace{-0.3cm} \\
    \rotatebox{90}{{\footnotesize \texttt{Asymmetric loss 1}}} & \hspace{-0.3cm} \includegraphics[trim=0bp 0bp 0bp 0bp,clip,width=\picwidthsi\textwidth]{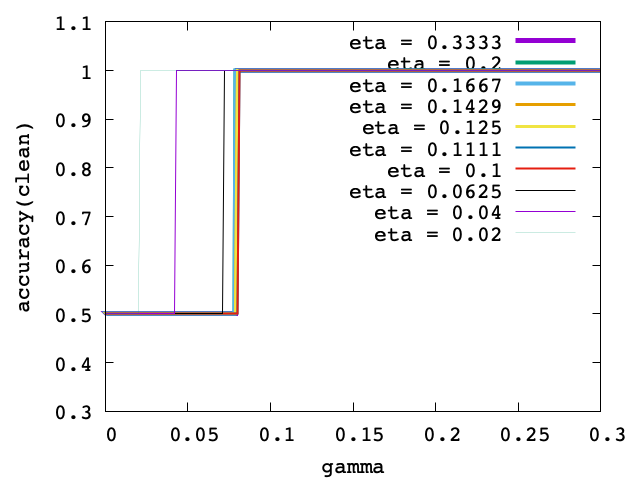} \hspace{-0.3cm} & \hspace{-0.3cm} \includegraphics[trim=0bp 0bp 0bp 0bp,clip,width=\picwidthsi\textwidth]{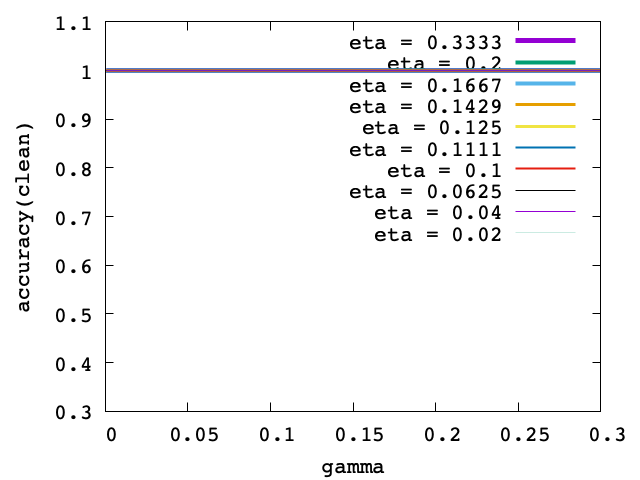} & \hspace{-0.3cm} \includegraphics[trim=0bp 0bp 0bp 0bp,clip,width=\picwidthsi\textwidth]{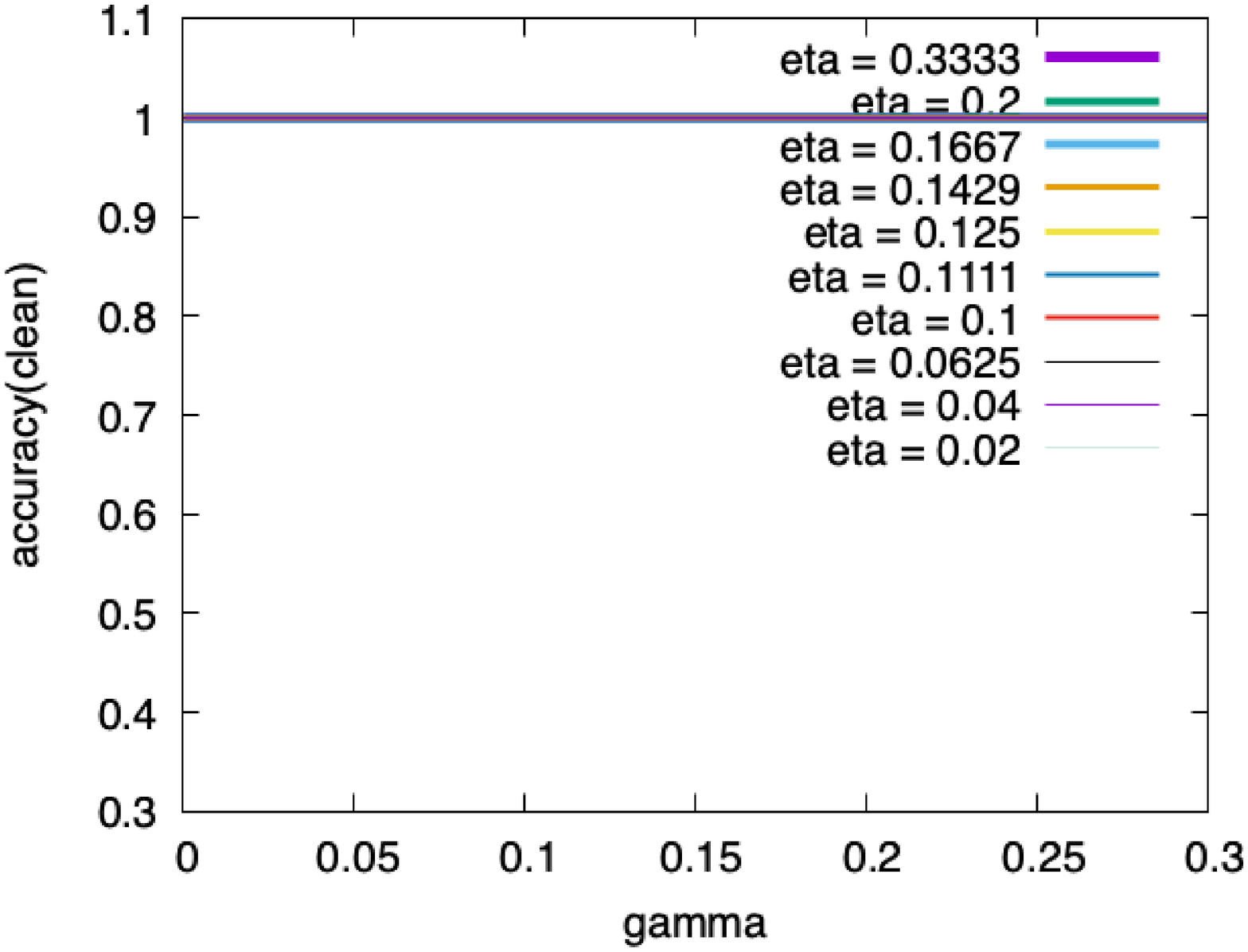} \hspace{-0.3cm} \\ \hline\hline
\end{tabular}
\caption{Accuracy of models obtained on
  $\mathcal{S}_{\mbox{\tiny{clean}}}$. As predicted by theory,
  \cdt~and \cnn~consistently get maximal accuracy while it falls for \cls~to that of the unbiased coin below a threshold value for $\gamma$. Notice that all losses display a similar pattern of phase transition.}
    \label{tab:accuracy}
  \end{table*}

  \begin{table*}
  \centering
  \begin{tabular}{rccc}\hline\hline
    & \cls \hspace{-0.3cm} & \hspace{-0.3cm} \cdt  & \hspace{-0.3cm} 1-\cnn \\
    \rotatebox{90}{{\footnotesize \texttt{Matusita loss}}} & \hspace{-0.3cm} \includegraphics[trim=0bp 0bp 0bp 0bp,clip,width=\picwidthsi\textwidth]{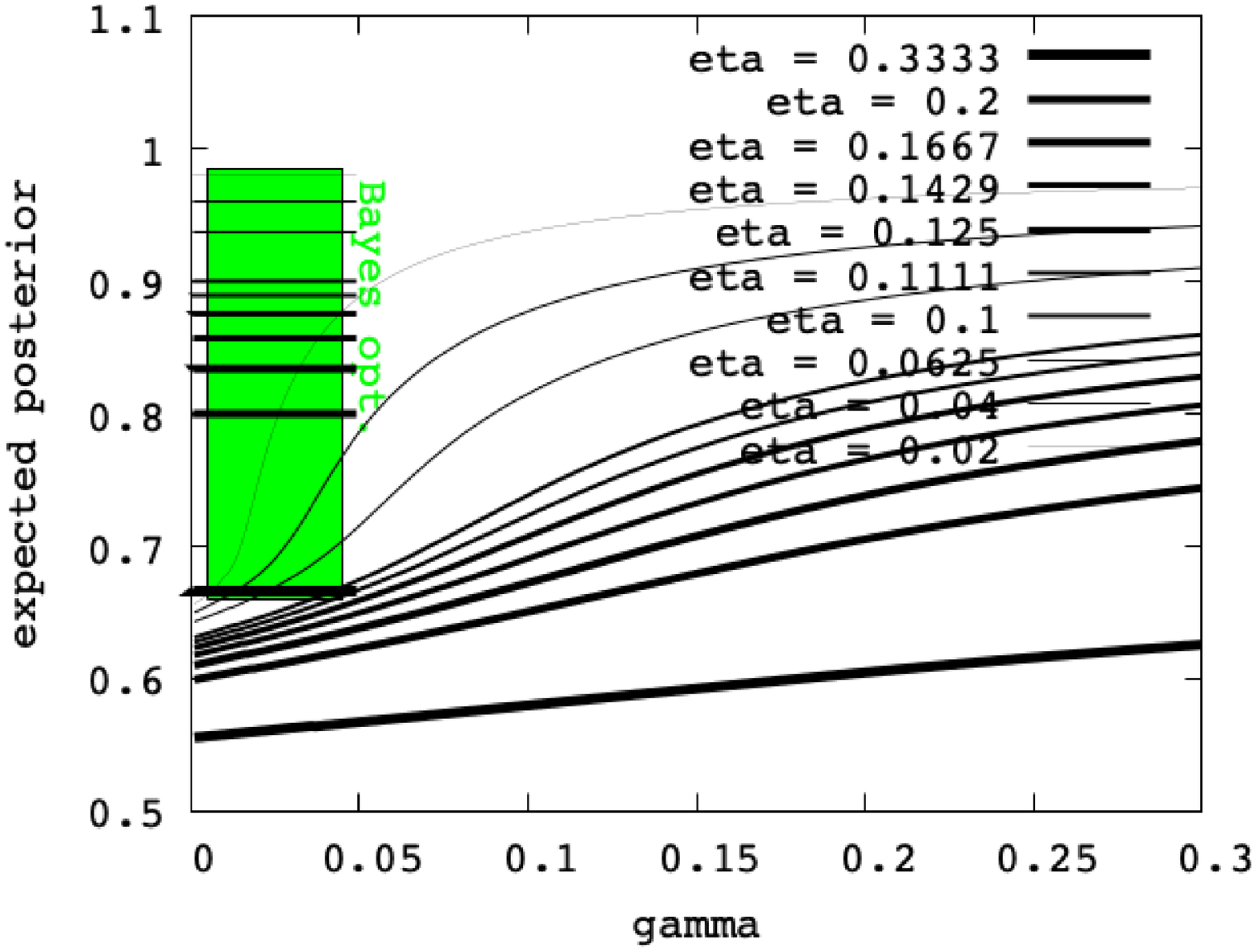} \hspace{-0.3cm} & \hspace{-0.3cm} \includegraphics[trim=0bp 0bp 0bp 0bp,clip,width=\picwidthsi\textwidth]{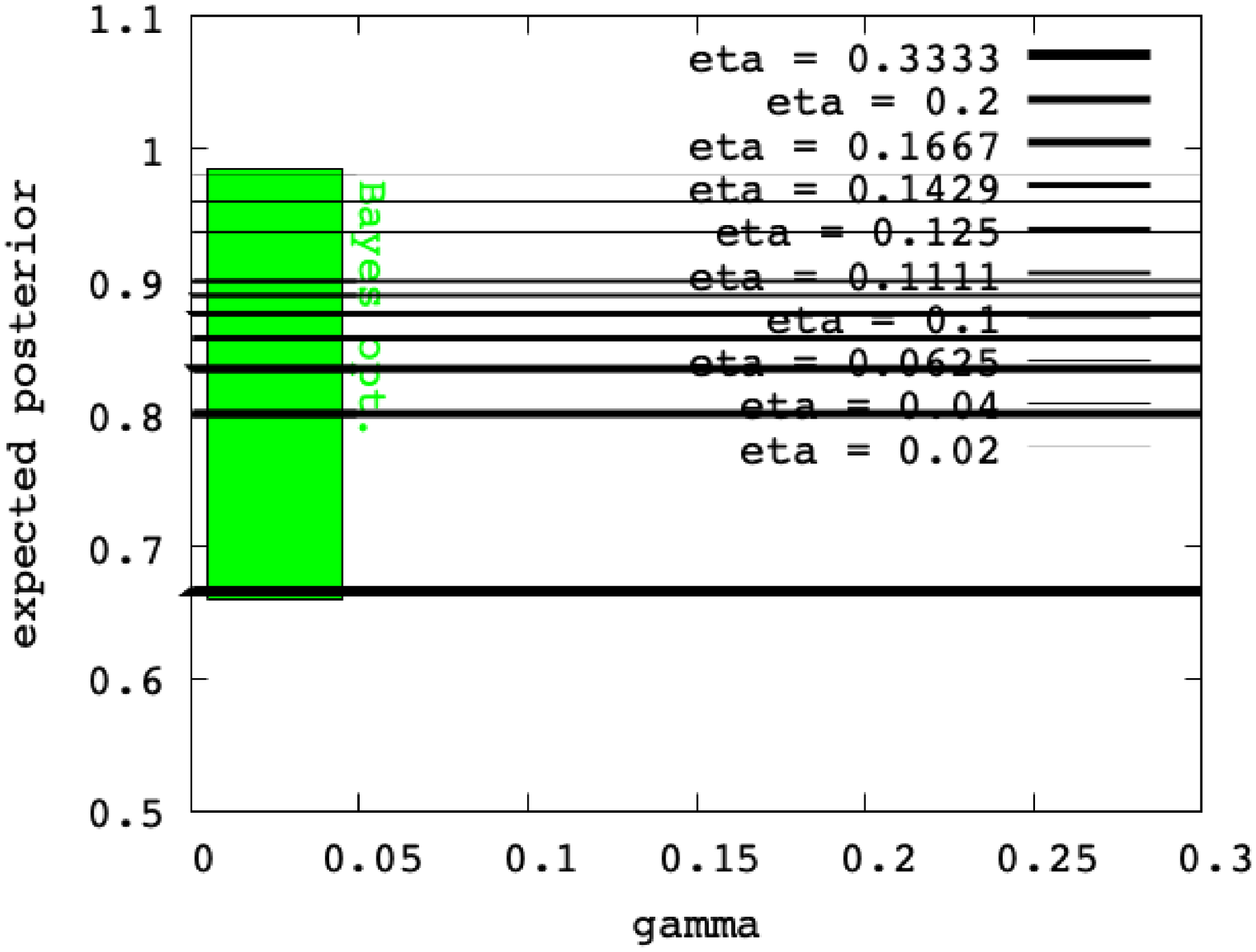} \hspace{-0.3cm} & \hspace{-0.3cm} \includegraphics[trim=0bp 0bp 0bp 0bp,clip,width=\picwidthsi\textwidth]{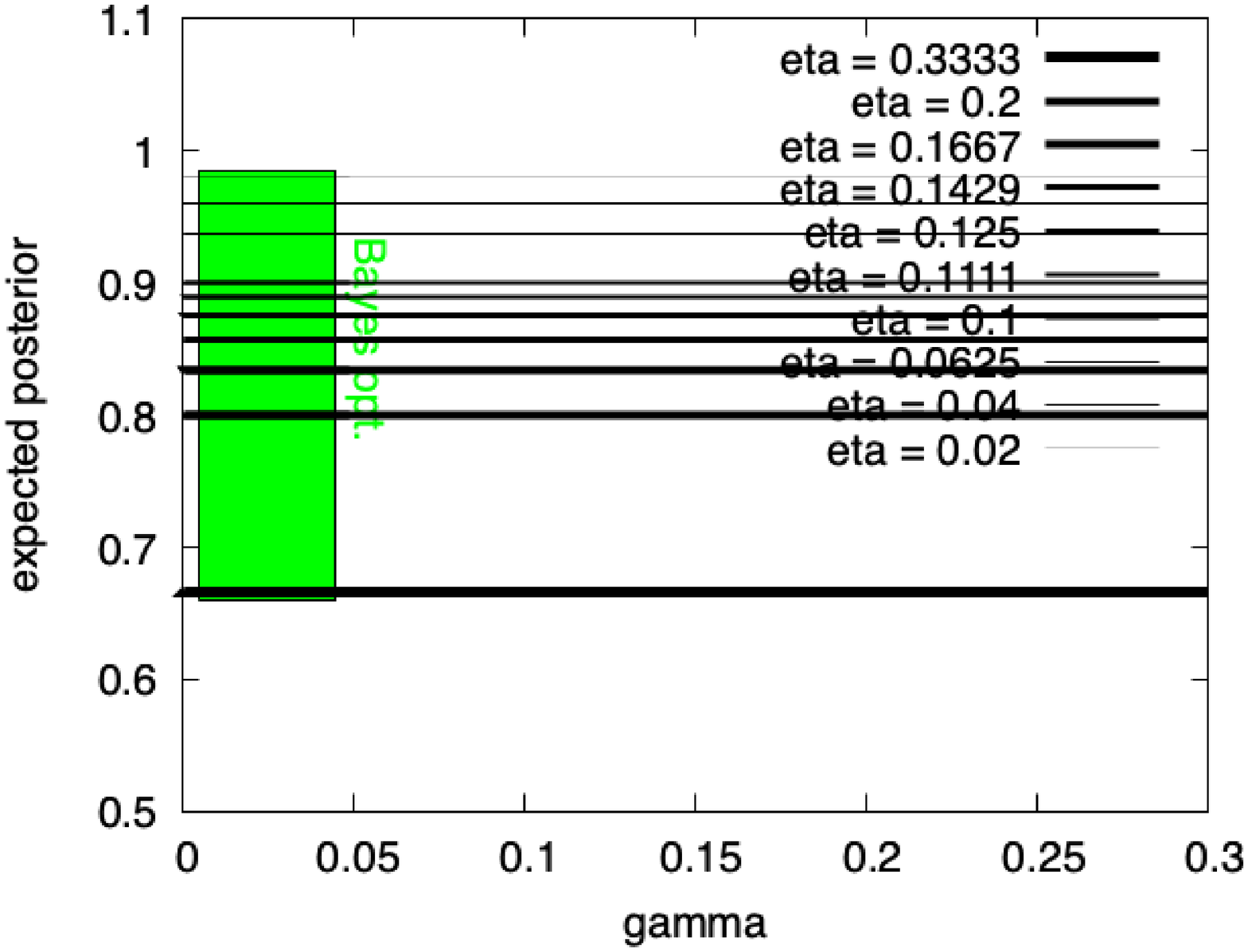} \hspace{-0.3cm} \\
    \rotatebox{90}{{\footnotesize \texttt{Log loss}}} & \hspace{-0.3cm} \includegraphics[trim=0bp 0bp 0bp 0bp,clip,width=\picwidthsi\textwidth]{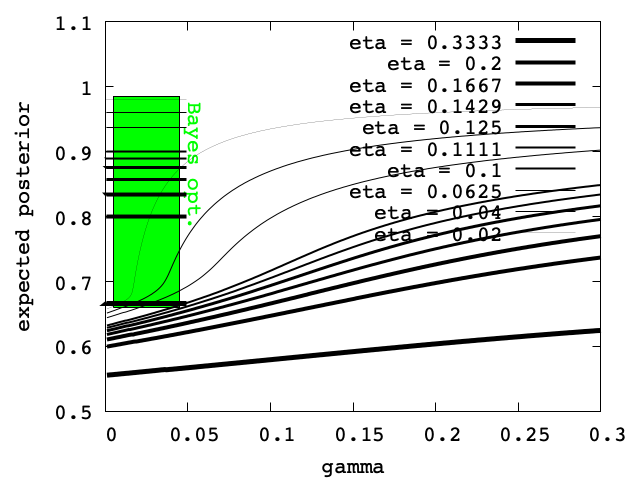} \hspace{-0.3cm} & \hspace{-0.3cm} \includegraphics[trim=0bp 0bp 0bp 0bp,clip,width=\picwidthsi\textwidth]{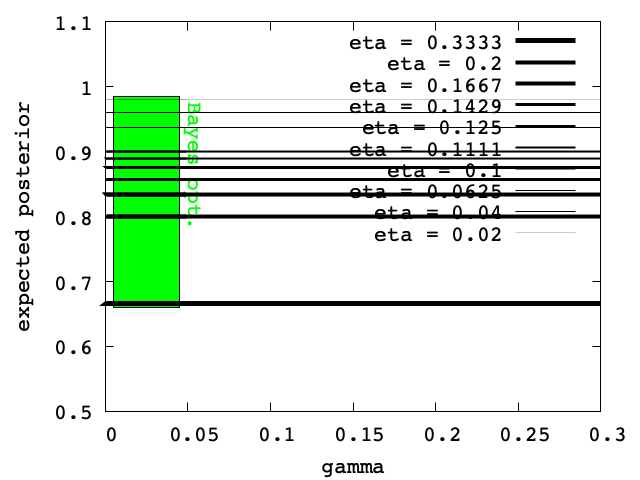} \hspace{-0.3cm} & \hspace{-0.3cm} \includegraphics[trim=0bp 0bp 0bp 0bp,clip,width=\picwidthsi\textwidth]{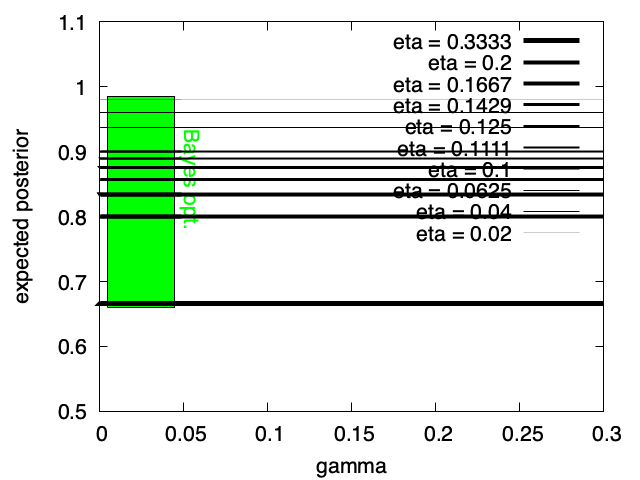} \hspace{-0.3cm} \\
    \rotatebox{90}{{\footnotesize \texttt{Square loss}}} & \hspace{-0.3cm} \includegraphics[trim=0bp 0bp 0bp 0bp,clip,width=\picwidthsi\textwidth]{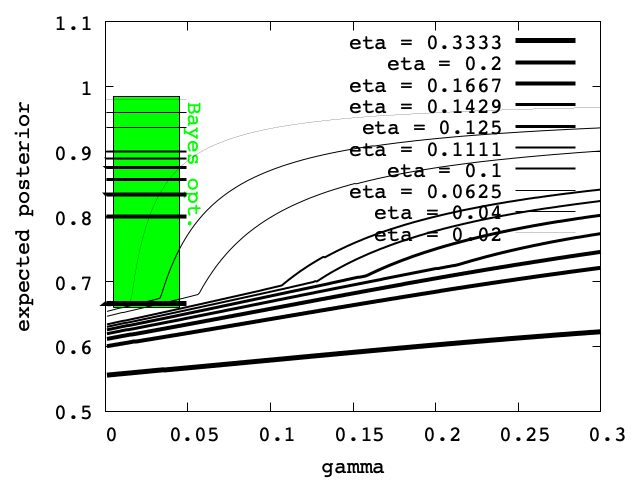} \hspace{-0.3cm} & \hspace{-0.3cm} \includegraphics[trim=0bp 0bp 0bp 0bp,clip,width=\picwidthsi\textwidth]{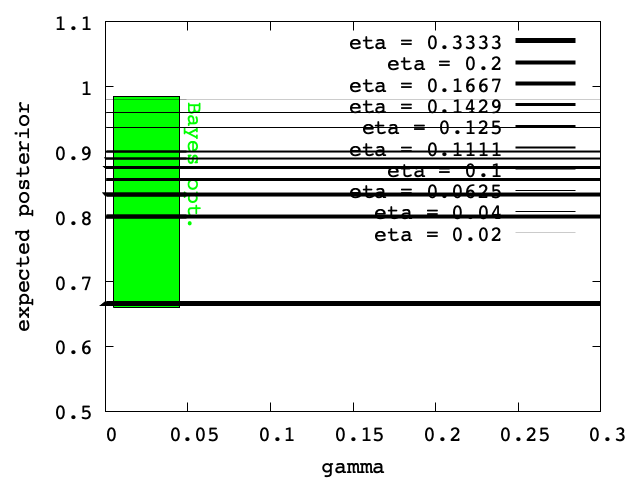} \hspace{-0.3cm}  & \hspace{-0.3cm} \includegraphics[trim=0bp 0bp 0bp 0bp,clip,width=\picwidthsi\textwidth]{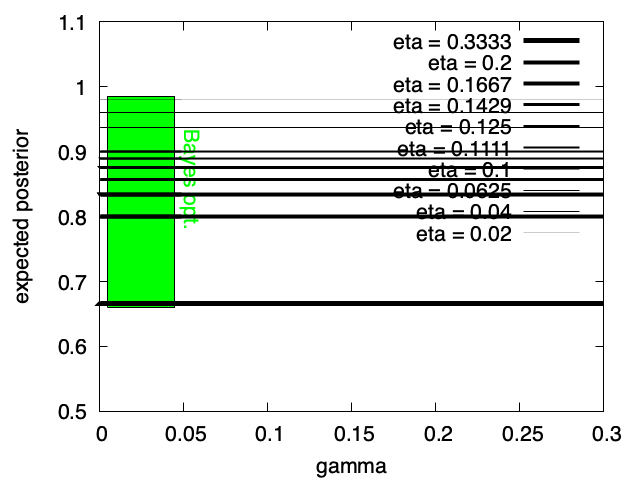} \hspace{-0.3cm} \\
    \rotatebox{90}{{\footnotesize \texttt{Asymmetric loss 1}}} & \hspace{-0.3cm} \includegraphics[trim=0bp 0bp 0bp 0bp,clip,width=\picwidthsi\textwidth]{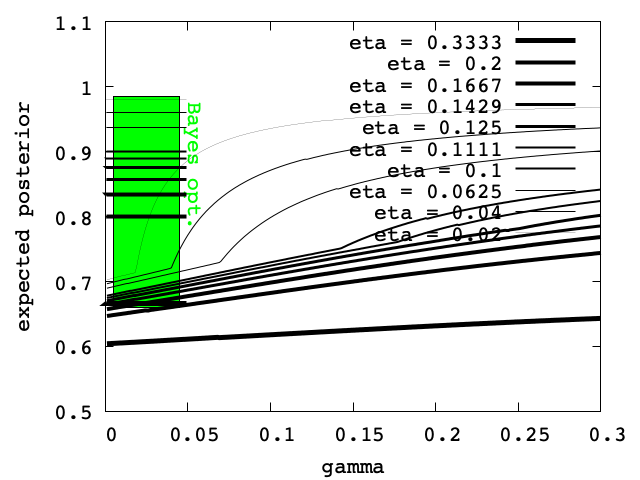} \hspace{-0.3cm} & \hspace{-0.3cm} \includegraphics[trim=0bp 0bp 0bp 0bp,clip,width=\picwidthsi\textwidth]{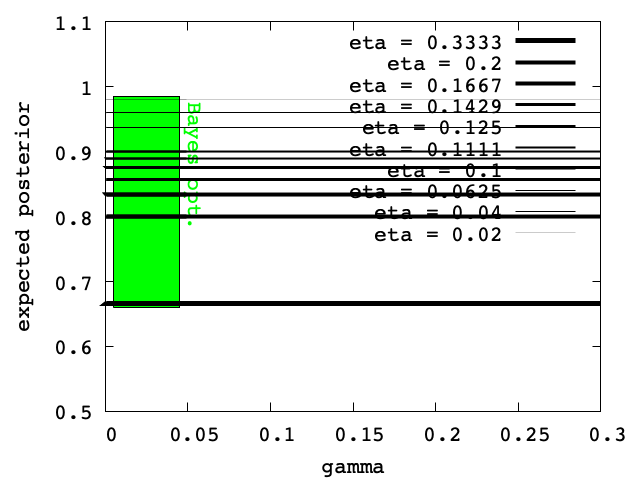} \hspace{-0.3cm} & \hspace{-0.3cm} \includegraphics[trim=0bp 0bp 0bp 0bp,clip,width=\picwidthsi\textwidth]{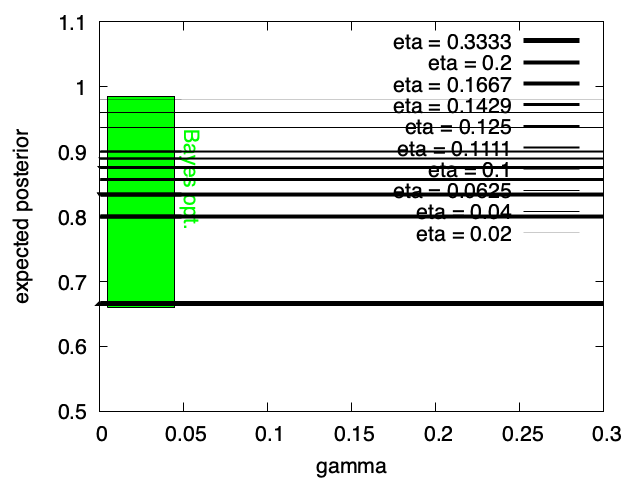} \hspace{-0.3cm} \\ \hline\hline
\end{tabular}
\caption{Expected posterior of models obtained from
  $\mathcal{S}_{\mbox{\tiny{noisy}}}$ (Bayes' optimal value is shown with a line segment on the left of each plot, in a green square). As predicted by theory,
  \cdt~consistently get Bayes'optimal prediction with just a single node \cdt~(Table \ref{tab:iter-not-WLA}), and \cnn~gets Bayes'optimal prediction in 3 iterations, which displays that the rate in \eqref{boost-rate-NN} is pessimistic for its dependence in $m$. \cls, on the other hand, get a substantial worsening of their estimation in the $\gamma$ window corresponding to the phase transition in accuracy (Table \ref{tab:accuracy}).}
    \label{tab:postest}
  \end{table*}

\begin{table*}
  \centering
  {\small
  \begin{tabular}{cc|c}\hline\hline
    \multirow{5}{*}{\rotatebox{90}{{\footnotesize \texttt{Matusita loss}}}} & $\partialloss{1}(u)$ & $\sqrt{\frac{1-u}{u}}$\\
                                                                            & $\partialloss{-1}(u)$ & $=\partialloss{1}(1-u)$\\
                                                                            & $\poibayesrisk(u)$ & $2\sqrt{u(1-u)}$\\
                                                                            & $ ({-\poibayesrisk'})^{-1}(z)$ & $\frac{1}{2}\cdot\left(1+\frac{z}{\sqrt{1+z^2}}\right)$\\
    & $\philoss(z)$ & $\frac{-x+\sqrt{1+x^2}}{2}$\\ \hline
   \multirow{5}{*}{\rotatebox{90}{{\footnotesize \texttt{Log loss}}}} & $\partialloss{1}(u)$ & $-\log u$ \\
                                                                            & $\partialloss{-1}(u)$ & $=\partialloss{1}(1-u)$\\
                                                                            & $\poibayesrisk(u)$ & $-u\log u - (1-u) \log (1-u)$ \\
                                                                            & $ ({-\poibayesrisk'})^{-1}(z)$ & $\frac{1}{1+\exp(-z)}$\\
                                                                            & $\philoss(z)$ & $\log(1+\exp(-z))$ \\ \hline
       \multirow{5}{*}{\rotatebox{90}{{\footnotesize \texttt{Square loss}}}} & $\partialloss{1}(u)$ & $(1-u)^2$\\
                                                                            & $\partialloss{-1}(u)$ & $=\partialloss{1}(1-u)$ \\
                                                                            & $\poibayesrisk(u)$ & $u(1-u)$ \\
                                                                            & $ ({-\poibayesrisk'})^{-1}(z)$ & $\left\{ \begin{array}{ccl}
                                                                                                                          0 & \mbox{ if } & z < -1\\
                                                                                                                          \frac{1+z}{2} & \mbox{ if } & z \in [-1,1]\\
                                                                                                                          1 & \mbox{ if } & z > 1
                                                                                                                          \end{array}\right.$\\
    & $\philoss(z)$ & $\left\{ \begin{array}{ccl}
                                                                                                                          -z & \mbox{ if } & z < -1\\
                                                                                                                          \frac{(1-z)^2}{4} & \mbox{ if } & z \in [-1,1]\\
                                                                                                                          0 & \mbox{ if } & z > 1
                                                                                                                          \end{array}\right.$\\ \hline
       \multirow{6}{*}{\rotatebox{90}{{\footnotesize \texttt{Asymmetric loss 1}}}} & $\partialloss{1}(u)$ & $\log(5u^2-8u+4) + \arctan\left(\frac{1}{2}\right) - \arctan\left(\frac{5u-4}{2}\right)$ \\
                                                                            & $\partialloss{-1}(u)$ & $\log\left(\frac{5u^2-8u+4}{4}\right) + 4\arctan\left(2\right) - 4\arctan\left(\frac{4-5u}{2}\right)$\\
                                                                            & $\poibayesrisk(u)$ & $\log(5u^2-8u+4) + A u + 4\arctan\left(2\right) - \log(4) + (4-5u) \arctan\left(\frac{5u-4}{2}\right)$ \\
                                                                            & $ ({-\poibayesrisk'})^{-1}(z)$ & $\left\{ \begin{array}{ccl}
                                                                                                                          0 & \mbox{ if } & z < -B\\
                                                                                                                          \frac{2}{5}\cdot \left(2-\tan\left(-\frac{z+A}{5}\right)\right) & \mbox{ if } & z \in [-B,C]\\
                                                                                                                          1 & \mbox{ if } & z > C
                                                                                                                          \end{array}\right.$ \\
    & $\philoss(z)$ & $\left\{ \begin{array}{ccl}
                                                                                                                          -z & \mbox{ if } & z < -B\\
                                                                                                                          2\log\left(\frac{\cos\left(\frac{A-B}{5}\right)}{\cos\left(\frac{A-z}{5}\right)}\right) + 4\cdot\frac{B-z}{5} & \mbox{ if } & z \in [-B,C]\\
                                                                                                                          0 & \mbox{ if } & z > C
                             \end{array}\right.$\\ 
                                                                            &  & $A \defeq \log(4) -4\arctan\left(2\right)+\arctan\left(\frac{1}{2}\right)$; $B\defeq \frac{\pi}{2}+\log(4)$; $C\defeq 2\pi - \log(4)$\\ \hline\hline
\end{tabular}
}
\caption{Definitions of strictly proper losses used in the Experiments (Section \ref{sec-toy-exp}, main file).}
    \label{tab:all-losses}
  \end{table*}

\begin{figure*}
  \centering
\includegraphics[trim=0bp 0bp 0bp 0bp,clip,width=0.8\textwidth]{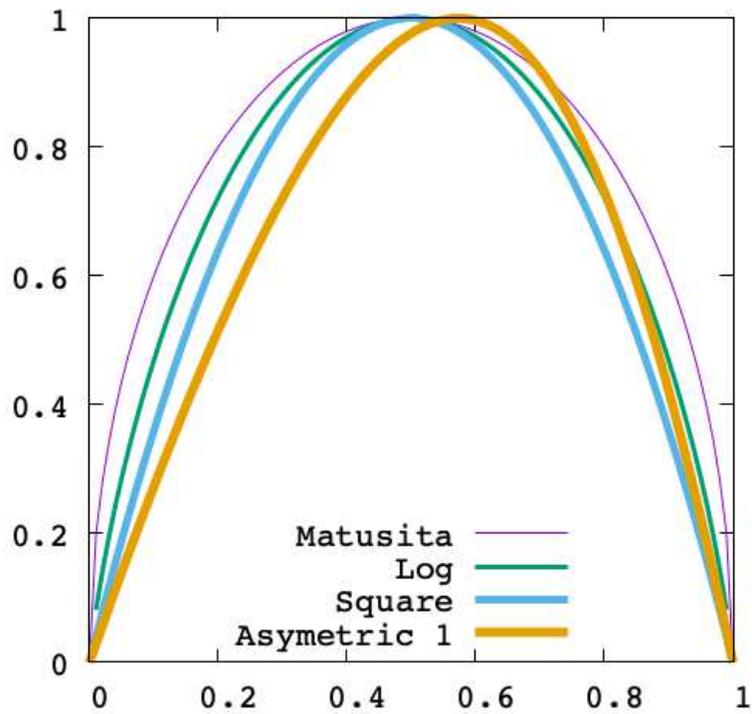} 
\caption{Pointwise Bayes risk ($\poibayesrisk(u)$, normalized so that $\max \poibayesrisk = 1$) for all losses in Table \ref{tab:all-losses}. Remark that all, except for \texttt{Asymetric loss 1}, are symmetric with respect to $u=1/2$.}
    \label{fig:all-cbr}
  \end{figure*}